\documentclass[mnsc,nonblindrev]{informs4} 

\RequirePackage{tgtermes}
\RequirePackage{newtxtext}
\RequirePackage{newtxmath}
\RequirePackage{bm}
\RequirePackage{endnotes}

\usepackage{tcolorbox}

\usepackage{xcolor}
\usepackage[framemethod=TikZ]{mdframed}
\usepackage{caption} 
\newmdenv[
    backgroundcolor=gray!10, 
    linecolor=black, 
    linewidth=1pt, 
    roundcorner=5pt, 
    font=\ttfamily\small, 
    skipabove=10pt, 
    skipbelow=10pt, 
    innerleftmargin=8pt, 
    innerrightmargin=8pt, 
    innertopmargin=6pt, 
    innerbottommargin=6pt, 
]{promptenv}

\usepackage{algorithm2e}

\OneAndAHalfSpacedXI

\usepackage{natbib}
 \bibpunct[, ]{(}{)}{,}{a}{}{,}%
 \def\bibfont{\small}%
\newcommand{\E}{{\rm E}}

\TheoremsNumberedThrough     
\EquationsNumberedBySection 

\usepackage[bookmarks=false, hypertexnames=false]{hyperref}
\hypersetup{
   colorlinks=true,
   linkcolor=blue, 
   citecolor= blue, 
   filecolor= blue, 
   urlcolor=blue, 
	pdfpagemode=FullScreen,
   }

\usepackage{epigraph}

\begin{document}


\RUNAUTHOR{Li, Fan, and Hong}

\RUNTITLE{Efficient Budget Allocation for Large-Scale LLM-Enabled Virtual Screening}

\TITLE{Efficient Budget Allocation for  Large-Scale  LLM-Enabled Virtual Screening}


\ARTICLEAUTHORS{%
    \AUTHOR{Zaile Li}
    \AFF{Technology and Operations Management Area, INSEAD, Fontainebleau, France \\\EMAIL{zaile.li@insead.edu}}
    \AUTHOR{Weiwei Fan}
    \AFF{Advanced Institute of Business and School of Economics and Management, Tongji University, Shanghai, China \\\EMAIL{wfan@tongji.edu.cn}} 
\AUTHOR{L. Jeff Hong}
    \AFF{Department of Industrial and Systems Engineering, University of Minnesota, Minneapolis, Minnesota\\\EMAIL{lhong@umn.edu}}
} 

\ABSTRACT{%
Screening tasks that aim to identify a small subset of top alternatives from a large pool are common in business decision-making processes. These tasks often require substantial human effort to evaluate each alternative’s performance, making them time-consuming and costly. Motivated by recent advances in large language models (LLMs), particularly their ability to generate outputs that align well with human evaluations, we consider an LLM-as-human-evaluator approach for conducting screening virtually, thereby reducing the cost burden. To achieve scalability and cost-effectiveness in virtual screening, we identify that the stochastic nature of LLM outputs and their cost structure necessitate efficient budget allocation across all alternatives.
To address this, we propose using a top-$m$ greedy evaluation mechanism, a simple yet effective approach that keeps evaluating the current top-$m$ alternatives, and design the explore-first top-$m$ greedy (EFG-$m$) algorithm. We prove that EFG-$m$ is both sample-optimal and consistent in large-scale virtual screening. Surprisingly, we also uncover a bonus ranking effect, where the algorithm naturally induces an indifference-based ranking within the selected subset. To further enhance practicality, we design a suite of algorithm variants to improve screening performance and computational efficiency.  Numerical experiments validate our results and demonstrate the effectiveness of our algorithms. Lastly, we conduct a case study on LLM-based virtual screening. The study shows that while LLMs alone may not provide meaningful screening and ranking results when directly queried, integrating them with our sample-optimal algorithms unlocks their potential for cost-effective, large-scale virtual screening.
}%





\maketitle

\epigraph{``... artificial intelligence, the science of making machines do things that would require intelligence if done by men."}{-- \cite{minsky1969semantic}}

\section{Introduction}\label{sec: introduction}

In many decision making processes, screening tasks play an important role. The goal of screening is to select a small subset of high-quality alternatives from a vast pool—whether in hiring, product design and investment project selection. By narrowing down alternatives, screening reduces decision making complexity, allowing decision makers to focus on the most promising options. After the initial screening, the shortlisted alternatives often undergo more rigorous evaluation, such as large-scale experiments or in-depth qualitative reviews, so that final decisions can be made. These screening processes are prevalent across business and management scenarios, but they often require substantial human effort, making them costly, time-consuming, and resource-intensive. For instance, evaluating job applicants requires expert screening of resumes and initial interviews, market research surveys rely on direct feedback from consumers, and investment decisions demand in-depth assessments by professionals. In each case, human evaluation is indispensable.
As computing power and artificial intelligence continue to advance—driving automation in industries from manufacturing and healthcare to financial services—it is natural to ask: \emph{Can we leverage machines to assist in screening tasks to reduce human workload while maintaining screening quality?}



Machine-assisted screening is not an entirely new concept. A practical example comes from drug discovery, where researchers must identify promising drug candidates from a vast library of chemical compounds for further refinement and testing. 
To address this challenge, \textit{virtual screening}, or in silico screening, has emerged as a common practice. Instead of relying on human-led laboratory experiments, virtual screening uses computational simulation tools—such as molecular dynamics models—to evaluate each compound's potential effectiveness.
This \textit{simulator-as-evaluator} approach significantly simplifies the screening process. Researchers can use simulation-based evaluations of all compounds to quickly rank them and select the top candidates for further experimental validation. A key advantage of virtual screening is its scalability. Advancements in computing power, driven by Moore’s Law, have made it convenient and cost-effective to run simulations in parallel on high-performance computing systems. This allows researchers to screen a considerably large number of alternatives efficiently. 

The cost-effectiveness and scalability of virtual screening in drug discovery motivate a similar simulator-as-evaluator approach for business screening tasks. However, applying this approach is not straightforward. Virtual screening works well in drug discovery because molecular interactions follow well-defined physical laws, making simulation modeling relatively simple. In contrast, many business tasks—especially those involving human judgment—lack such precisely defined evaluation criteria. Consider a marketing research scenario where Lenovo surveys customers to estimate their willingness to pay (WTP) for a new laptop design. Unlike molecular interactions, WTP evaluations are shaped by individual preferences, experiences, and perceptions. Building an explicit simulator to replicate human evaluation in such contexts is not only challenging but may also be infeasible.

Luckily, recent advancements in large language models (LLMs) may offer a promising remedy for the challenge of ``simulating" human evaluators. In a nutshell, LLMs are a class of models that process text-based descriptions (called prompts) as inputs and generate text outputs corresponding to the provided input (many modern LLMs can also handle multimodal inputs and outputs, such as images). Trained on vast amounts of human-generated data, LLMs are increasingly perceived to exhibit a certain level of intelligence, which has been leveraged in applications such as code generation and writing assistance. A particularly notable development relevant to our focus emerges in the marketing area, where \cite{brand2023using} and \cite{li2024frontiers} find that LLMs are capable of generating realistic survey responses that closely align with responses from real customers. For example, when provided (prompted) with a survey problem that includes a product description and a question about the maximum WTP, LLMs can generate reasonable WTP estimates that are consistent with human responses and real market prices. We illustrate this process in Figure \ref{fig:example_LLM}. Such findings are fascinating, as they demonstrate the potential for LLMs to act as virtual human evaluators; see Section \ref{subsec: review} for a review of related studies reporting similar results. In the WTP example, LLMs evaluate the ``performance" of a product in terms of WTP reasonably well. This concept is directly aligned with the core principle of virtual screening: a machine evaluates each alternative without requiring human intervention. 

\begin{figure}[htbp]
    \centering
    \includegraphics[width=0.7\linewidth]{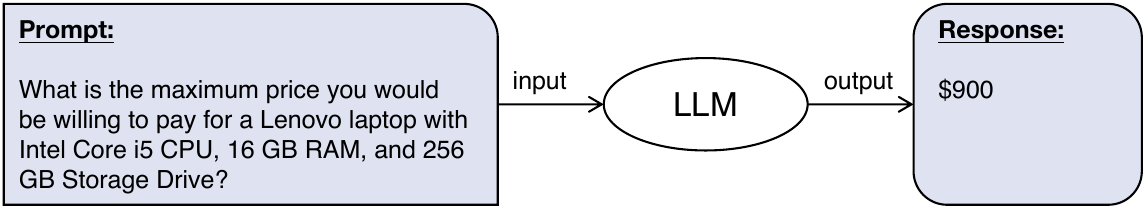}
    \caption{Evaluation process using LLMs to estimate willingness to pay for a laptop design}
    \label{fig:example_LLM}
\end{figure}

The ability of LLMs to act as virtual human evaluators may pave the way for a virtual approach to business screening tasks. Just as virtual screening has transformed drug discovery, LLM-based virtual screening could enable decision makers to rapidly evaluate and screen alternatives without the need for costly and time-consuming human assessments. With the fast-growing LLM community and continuous advancements in new models,  we believe that this avenue holds significant potential. Therefore, in this paper, we focus on the \textit{LLM-as-human-evaluator} approach for virtual screening, explore the key challenges to its efficient application, propose solutions to address these challenges, and finally demonstrate its economic value.

A key feature of LLMs is that their output is inherently random. This stochastic nature enables them to reflect human-like variability and heterogeneity, but it also complicates screening tasks. In the WTP example, for the same product and identical prompt, querying the LLM multiple times may yield different WTP estimates. In general, a prompt to an LLM does not determine a single fixed response but rather a distribution of possible responses, from which the model samples a response each time it is queried. This randomness arises from the stochastic nature of LLMs' model structure and is controlled by a hyperparameter called \emph{temperature}, which typically takes a value between 0 and 2. By default, the temperature is often set to 1, allowing the model to sample one response unbiasedly from the distribution it believes to be most plausible. 
Interestingly, this inherent randomness in LLMs may enhance their alignment with human decision making. Since LLMs are trained on vast amounts of data from diverse sources, their output does not reflect the knowledge or opinion of a single individual but instead captures the collective patterns of human knowledge and judgment. Consequently, this randomness may help LLMs capture human heterogeneity and subjective preferences. However, the same randomness also introduces new challenges for screening tasks.
For each alternative (and its associated prompt), we may need to query the LLM multiple times to obtain multiple observations of its true performance (e.g., WTP or evaluation score), then compute a mean estimate from the observations and make screening decisions based on the mean estimates of all alternatives. When the number of alternatives in a screening task is relatively large, the total number of required queries can become enormous. This naturally raises concerns about the cost of LLM-based virtual screening.


Using LLMs may come at a financial cost, which necessitates efficient budget allocation when solving screening problems. Running LLMs efficiently requires professional-grade GPUs and expertise in GPU deployment. For ordinary users, a more accessible approach is to use commercial LLM APIs, such as OpenAI’s GPT-4o, through which users can send prompts and receive responses without hosting the model themselves. These services are reliable and fast, as API providers own massive amounts of GPUs. However, they are not free. API-based LLM services charge based on prompt length, which is measured by the number of tokens. One token roughly corresponds to four characters or 0.75 words in English text. For example, according to OpenAI’s pricing, GPT-4o charges \$5 per million tokens. This pricing structure means that in LLM-based screening tasks, each observation of an alternative incurs a monetary cost. 
This cost structure of LLM-based applications naturally raises a fundamental decision problem in virtual screening. Consider a common scenario where the screening task is subject to a financial budget, which determines the total number of allowable LLM queries. As discussed earlier, multiple observations are typically required for each alternative to estimate its true mean performance and make an informed screening decision.  Then, in the budgeted setting, a key question arises: How should the total query budget be allocated across alternatives to maximize the quality of the final screening outcome? The efficiency of this budget allocation directly determines the cost effectiveness of LLM-based virtual screening.  

Resolving the budget allocation problem efficiently requires careful formulation. The first step is to define a suitable evaluation criterion for the efficiency of budget allocation. A natural choice is the probability of correctly selecting the true top-\(m\) alternatives  (e.g.,  $m=5$ or 10) in terms of their mean performance. 
Given this criterion, the budget allocation problem can be restated as: how should we allocate the total query budget to maximize the probability of correct screening? Furthermore, notice that this formulation reflects a strict correctness requirement, where the selected subset must exactly match the true top-\(m\) alternatives. In scenarios with a tight budget constraint, it may be desirable to relax this requirement. For screening problems, as the number of alternatives increases, some alternatives may become nearly indistinguishable from the true top-\(m\). Achieving exact correctness thus becomes increasingly difficult and potentially unnecessary. In this context, a more feasible goal is to aim for ``good screening", where the selected \(m\) alternatives are close in performance to the true top-\(m\). This relaxed objective may require significantly fewer queries while still maintaining high decision quality.

Readers familiar with the simulation optimization literature may recognize that the top-\(m\) screening formulation aligns closely with the optimal subset selection (OSS) problem \citep{dudewicz1975allocation, chen2008efficient}. Like virtual screening, the goal of OSS is to identify the top alternatives from a set using stochastic simulation. In the OSS literature, a number of budget allocation algorithms have been developed, and these algorithms could, in principle, be applied to support virtual screening.  However, a fundamental issue in applying these algorithms lies in their limited scalability (see Section~\ref{subsec: review} for a review). From a theoretical perspective, a key factor that determines the performance of any budget allocation strategy in large-scale problems is how the total number of observations required for a meaningful screening grows with the number of alternatives. In the context of LLM-based virtual screening, this growth rate directly dictates how the monetary cost scales with the size of the alternative set. The lower the growth rate, the more cost-efficient the algorithm becomes. Ideally, the most economical algorithm would achieve the lowest possible growth rate in budget scaling.  This raises a theoretical problem: what is the minimal achievable growth rate? Intuitively, the total budget must grow at least linearly, since each alternative should be evaluated at least once. If this intuition holds, the next challenge is to design a budget allocation strategy that achieves this rate—thereby enabling cost-effective virtual screening at scale.

Furthermore, under the top-\(m\) screening formulation, an additional and interesting question may arise: can we obtain a reliable ranking among the selected alternatives? Ranking information is often desirable in real-world screening applications where further decision making is constrained by non-financial resources—such as time, experimental capacity, or managerial attention—making prioritization and sequencing critical. This is especially true in domains like drug development and recruitment, where decisions are made on how to proceed with the shortlisted options.  Taking ranking into account introduces a new dimension to the screening problem. It invites a deeper exploration of what constitutes a good ranking, how such a criterion can be formally defined, and how it may further complicate the budget allocation.

\subsection{Our Contributions}
We systematically address the above problems and challenges surrounding efficient budget allocation in virtual screening. Specifically, we begin by defining the good screening event and the corresponding screening quality criterion: the probability of good screening, denoted by \(\text{PGS}_m\). This probability serves as the primary performance measure for budget allocation algorithms throughout the paper.  We then establish a fundamental theoretical result: the minimal growth rate of the evaluation budget (and thus the monetary budget) required to achieve a non-zero \(\text{PGS}_m\) is linear in the number of alternatives.  Next, to enable efficient and scalable budget allocation, we take inspiration from \cite{li2023surprising} and design a simple yet effective top-\(m\) greedy allocation mechanism, which concentrates evaluation efforts on the current top-\(m\) alternatives. Building on this mechanism, we propose the explore-first top-\(m\) greedy (EFG-\(m\)) algorithm. For this algorithm, we derive a lower bound for its PGS$_m$ and then prove that this EFG-\(m\) algorithm is both sample-optimal—achieving the minimal linear budget growth rate—and consistent for large-scale virtual screening problems.


Further, we discover that EFG-$m$ provides not only sample-optimal screening but also effective ranking of the selected top-$m$ alternatives—almost {for free}. To measure the effectiveness of this ranking, we borrow the concept of semiordering from economics \citep{luce1956semiorders} and introduce the notion of indifference-based ranking. Building on this concept, we expand our assessment of the EFG-$m$ algorithm by introducing an additional metric: the probability of good screening and ranking, denoted by PGSR$_m$. Interestingly, our numerical experiments show that as the number of alternatives increases, PGSR$_m$ tends to align with PGS$_m$. Inspired by this, we prove that once a good screening is achieved, the additional evaluation budget required to ensure a good ranking is (asymptotically) negligible. This leads to a strengthened result: EFG-$m$ is also sample optimal and consistent in terms of PGSR$_m$. These findings highlight the algorithm's efficacy in simultaneously addressing both screening and ranking challenges in virtual screening.

We also develop a suite of EFG-\(m\) variants to its enhance practical performance. First, we extend the top-\(m\) greedy phase to a top-\(M\) greedy phase where \(M > m\). This simple modification leads to the EFG-\(M\) algorithm, which can significantly outperform EFG-\(m\) by broadening the sampling focus beyond the top-\(m\) set. Next, we further improve the exploration phase by incorporating a seeding step as in \cite{hong2022solving} and obtain the EFG-\(M^+\) algorithm. Finally, to leverage the concurrent capability of LLM API and services—i.e., their ability to respond to multiple queries simultaneously—we propose an asynchronous parallel version of the top-\(M\) greedy phase and design the EFG-\(M^{++}\) algorithm. This variant maintains similar performance as EFG-\(M^+\) while offering significantly improved computational efficiency.

Lastly, we conduct a comprehensive numerical study to validate our theoretical results and evaluate the performance of our proposed algorithms. We also present an illustrative case study on LLM-based virtual screening. We formulate a LLM-based laptop design screening problem that seeks the top designs with highest willingness to pay (WTP). Using a locally deployed LLM, we first verify the reasonableness of the model’s WTP estimates and then apply our algorithm to solve the screening problem. The results show that, when coupled with our sample-optimal algorithms, the LLM-as-human-evaluator approach can effectively support virtual screening.  The cost analysis reveals that to solve a problem with 3,240 alternatives, a total budget of \$62.4 for OpenAI's GPT-4o mini API (or \$28.1 for DeepSeek's DeepSeek V3 API) may be sufficient, which demonstrates the cost-effectiveness of our approach.
%
The readers may be curious about whether the LLMs can directly perform screening if given the descriptions of all alternatives. To explore this, we test an alternative approach that directly prompts the LLM to name the top-$m$ designs. Interestingly, we find that LLMs often fail to provide meaningful screening results even for small-scale problems, let alone achieving sample optimality. We believe this aligns with the ``hallucination'' phenomenon, a well-documented issue with LLM-generated outputs \citep{huang2025survey}. While the model may perform better when asked to identify the single best alternative, its performance deteriorates rapidly as the number of alternatives increases. These findings suggest that LLMs may not be capable of making informed screening decisions yet and underscore the importance of the LLM-as-human evaluator approach in enabling reliable and scalable virtual screening.

\subsection{Literature Review}
\label{subsec: review}

The use of LLMs to substitute for human effort and support cost-effective decision-making has attracted growing attention across fields such as marketing, accounting, and management science. As mentioned earlier, \cite{brand2023using} and \cite{li2024frontiers} explore and validate the potential of LLMs to simulate human customers in marketing research. In a similar direction, \cite{goli2024frontiers} examine how preferences can be elicited from LLM-generated responses in marketing surveys. This stream of work motivates the LLM-as-human-evaluator approach for virtual screening proposed in this paper. Furthermore, \cite{de2025chatgpt} review and validate the potential of LLMs to act as human researchers in accounting tasks, while \cite{ye2025lola} explore their role in evaluating content appeal for news recommendation systems. From a decision making perspective, \cite{wasserkrug2024combining} offer an overview of combining LLMs with management science tools to support smarter, data-driven decisions. Our work contributes to this emerging research agenda by developing an efficient budget allocation framework for leveraging LLMs in virtual screening.

The budget allocation problem in virtual screening aligns closely with the optimal subset selection (OSS) problem \citep{chen2008efficient, zhang2021asymptotically} in the simulation optimization literature. OSS generalizes the classical ranking and selection (R\&S) problem \citep{bechhofer1954single, hong2021review} by shifting the objective from selecting the single best alternative to identifying the top-\(m\) alternatives. 
The OSS problem is also studied in the best arm identification literature; see, e.g., \citet{bubeck2013multiple}.  
Despite the conceptual alignment between OSS and virtual screening, existing OSS algorithms may fall short in addressing several key aspects specific to virtual screening. First, the ranking aspect of screening remains underexplored \citep{bechhofer1968sequential}. 
More importantly, current OSS methods do not account for the scalability requirement that is central to virtual screening. While the notion of sample optimality—ensuring that the required evaluation budget grows minimally with the number of alternatives—has been a central focus in recent work on large-scale R\&S \citep{jamieson2013finding, zhong2022knockout, hong2022solving}, it has not yet been studied in the OSS setting.  Furthermore, computational considerations, such as parallel and asynchronous execution, which are especially relevant in large-scale problems, remain largely unaddressed.  Despite these limitations, to highlight the advantages of our approach, we compare our proposed algorithms with several prevalent OSS algorithms in Section~\ref{subsec: comparison}. We also discuss additional algorithms not included in the main comparison in Section~\ref{subsec: other_algorithms}.
Readers may refer to these sections for more details.

The remainder of this paper is organized as follows.  In Section \ref{sec: problem}, we formulate the budget allocation problem for virtual screening. In Section \ref{sec: top-m-boundary-crossing}, we propose the EFG-$m$ algorithm and analyze the algorithm's effectiveness. In Section \ref{sec: sample-optimality}, we establish the sample optimality and consistency of the EFG-$m$ algorithm.  In Section \ref{sec: ranking},  we define the indifference-based ranking for virtual screening and show that the EFG-$m$ algorithm may automatically obtain the ranking within the selected subset.  In Section \ref{sec: enhancements}, we introduce improvement strategies for the EFG-$m$ algorithm.  In Section \ref{sec: numerical}, we present our numerical results and the LLM-based cash study.  Lastly, we conclude in Section \ref{sec: conclusion} and present auxiliary materials in the e-companion.

\section{Problem Statement}
\label{sec: problem}
\subsection{Large Language Models and Virtual Screening}

Consider a screening problem with \( k \) candidate alternatives, denoted by the index set \( \mathcal{K} = \{1, 2, \dots, k\} \). Each alternative \( i \in \mathcal{K} \) has an unknown constant performance level \( \mu_i \), which reflects its value or quality from the perspective of the decision maker. The goal of screening is to identify the top-\(m\) ($m>1$) alternatives from \( \mathcal{K} \) with the highest performance. 
Motivated by recent advances in large language models (LLMs), we consider using LLMs as virtual human evaluators to assess the performance of alternatives and thereby conduct screening \textit{virtually} with the goal of reducing the cost and time associated with human evaluation. Specifically, for each alternative, we query the LLM with a text prompt that describes the alternative and ask the LLM to evaluate it. We then extract a numerical performance estimate from the model's response. For simplicity, we focus on scenarios where each alternative can be described in natural language, such that it can be directly prompted to the LLM (note that many modern LLMs are also capable of handling multimodal inputs and outputs, and our results may also hold in these more general settings). As discussed earlier, LLM outputs are inherently stochastic. For each alternative \( i \), we model the LLM’s output as a random variable \( X_i \), of which each observation corresponds to a sampled performance estimate in response to a query about that alternative.  Throughout the paper, we adopt the following key assumption:
\begin{assumption}
\label{assu: hypothesis}
For each alternative \( i \) with performance level \( \mu_i \), given a properly designed prompt, the LLM's evaluation output \( X_i \) satisfies 
\[
\E[X_i] = \mu_i \quad \text{and} \quad 0 <  \mathrm{Var} [X_i] = \sigma_i^2 < \infty.
\]
\end{assumption}  
This assumption posits that an LLM can act as a noisy but unbiased evaluator of the true performance \( \mu_i \) for each alternative, when queried with an appropriately constructed prompt. The condition of finite, non-zero variance ensures that the LLM’s output is neither deterministic nor excessively volatile.
Notice that unlike traditional simulation models, where the data-generating process is explicitly defined, LLM-generated responses depend on training data, model architecture, and prompting strategies—making it difficult to theoretically validate the assumption. Nevertheless, with the rapid advancements in LLM capability, we believe this assumption has practical merit. Empirical evidence supporting this assumption is emerging; for example, \cite{brand2023using} and \cite{li2024frontiers} demonstrate that LLM-generated evaluations in marketing research may closely match human responses.

We make an additional assumption that, for the purpose of screening, a proper prompt is available for each alternative. 
In practice, proper prompts can be manually crafted through a trial-and-error process or even generated by the LLM itself.  
In addition, we note that the distribution of \( X_i \) depends on the temperature parameter used in the LLM, which governs the randomness of the model’s output. Throughout this paper, we adopt the default setting of letting temperature be 1, which corresponds to unbiased sampling from the model’s learned knowledge. Furthermore, under Assumption \ref{assu: hypothesis}, obtaining reliable performance estimates for each alternative requires multiple evaluations from the LLM. We assume that, for each alternative, the outputs from repeated prompts are independent and identically distributed (i.i.d.). This condition can be easily achieved by prompting the LLM in stateless mode without providing access to previous responses. 


\subsection{Budget Allocation for Virtual Screening}
\label{subsec: problem}

LLM-based evaluations are not free—each query may incur a monetary cost. This naturally leads to a key decision problem in virtual screening:   Given a limited budget, how should queries be allocated across alternatives to maximize screening quality?  Let \( M \) denote the total financial budget available for the screening task. For each alternative \( i \), let \( h_i \) represent the cost of obtaining one evaluation (observation) from the LLM. In general, \( h_i \) may depend on the prompt length. In this paper, we consider a structured prompting scenario, where all alternatives are described under a fixed template. As a result, the cost per query is approximately uniform across alternatives, i.e., \( h_i \approx h \) for all \( i \in \mathcal{K} \). Then, the total number of observations—also referred to as the total evaluation budget—is given by $B = \left\lfloor \frac{M}{h} \right\rfloor$. Consequently, the budget allocation problem is to allocate this evaluation budget \( B \) across the \( k \) alternatives.  Without loss of generality, we assume that the alternatives are ordered in descending order by their  performance, i.e., \( \mu_i \geq \mu_{i+1} \) for all \( i = 1, \dots, k-1 \). The goal is to screen out all inferior alternatives to select the top-\(m\) alternatives, which form the optimal subset \( \mathcal{S}^* = \{1, 2, \dots, m\} \). Uniqueness of the optimal subset would require \( \mu_m \neq \mu_{m+1} \). 

A budget allocation algorithm collects i.i.d. observations for each alternative from the LLM, estimates their mean performances, and selects a subset \( \mathcal{S} \subseteq \mathcal{K} \) of size \( m \) when the total evaluation budget is exhausted. A natural measure of an algorithm's effectiveness is the probability of correctly selecting \( \mathcal{S}^*\), defined as:  
\begin{eqnarray*}\label{def: PCS}  
    \mbox{PCS}_m=\Pr \left\{ \mathcal{S} = \mathcal{S}^* \right\},  
\end{eqnarray*}  
where the subscript \( m \) in PCS\(_m\) highlights its dependency on the size of the optimal subset \( \mathcal{S}^* \).  While the  PCS\(_m\)  is a natural measure, it can be too strict. As the number of alternatives \( k \) grows, it may become increasingly likely that some inferior alternatives in \( \mathcal{K} \setminus \mathcal{S}^* \) have mean performances very close to those in the optimal subset \( \mathcal{S}^* \). In such cases, exact screening may be both difficult and unnecessary. The decision maker may be satisfied with a slightly relaxed yet more practical objective—selecting a set of \( m \) approximately top-\( m \) or ``good" alternatives. Inspired by \cite{kalyanakrishnan2012pac}, we define an alternative \( i \) as good if its mean satisfies \( \mu_i \geq \mu_m - \delta \) where $\delta>0$ is the indifference-zone parameter denoting the smallest mean difference that is practically meaningful, and we denote the set of all good alternatives as \( \mathcal{G} \). Then, we say that a good screening (GS) is achieved if the selected subset \( \mathcal{S} \) of size \( m \) lies entirely within \( \mathcal{G} \). The corresponding performance measure, the probability of good screening (PGS\(_m\)), is defined as:
\begin{eqnarray*}
    \mbox{PGS}_m  = \Pr \left\{\mathcal{S}\subseteq \mathcal{G} \right\}.
\end{eqnarray*}
Notably, PGS\(_m\) generalizes PCS\(_m\): when \( \delta = 0 \), \( \mathcal{G} = \mathcal{S}^* \), and thus \(\text{PGS}_m = \text{PCS}_m\). Therefore, throughout the remainder of this paper, we adopt PGS\(_m\) as our primary screening quality criterion.

We note that both PCS\(_m\) and PGS\(_m\) focus solely on which alternatives are selected and do not evaluate the ranking among the selected alternatives. As discussed in Section~\ref{sec: introduction}, ranking information can also be very valuable in downstream decision making. We revisit this issue in detail in Section~\ref{sec: ranking}. Besides, to facilitate the design and analysis of budget allocation algorithms, throughout the paper we assume that the LLM evaluation output for each alternative is normally distributed, as stated in the following assumption:
\begin{assumption}
\label{assu: normal}
For each alternative \( i \in \mathcal{K}\), $
X_i \sim \text{Normal}(\mu_i, \sigma_i^2)$.
\end{assumption}
Such normality assumption is standard in the simulation optimization or best arm identification literature (see, e.g., \citealt{chen2011stochastic, hong2021review}). Its popularity stems from the Central Limit Theorem, which shows that the sample mean of a non-normal random variable can be approximately normal. While LLM outputs may not always conform perfectly to normality, the assumption offers a pragmatic and analytically convenient approximation. Numerical experiments in Section~\ref{sec: numerical} show that our proposed algorithms remain robust to deviations from this assumption.

\subsection{Sample Optimality and Consistency}
\label{subsec: regime}
A defining feature of virtual screening is its scalability: it should be versatile enough to handle problems involving a large number of alternatives in a cost-effective manner. To analyze and evaluate the performance of budget allocation algorithms in large-scale settings, we adopt an asymptotic regime in which the number of alternatives \( k \rightarrow \infty \)  \citep{jamieson2013finding, hong2022solving, pei2022parallel}. 
To make this asymptotic analysis meaningful, we impose two regularity conditions on the problem parameters—specifically, the size of the optimal subset $m$ and the number of good alternatives $|\mathcal{G}|$. We maintain $m$ as a bounded constant as $k$ increases. This reflects a common scenario in which the number of alternatives to be selected is determined by practical constraints (e.g., time, budget, or evaluation capacity), rather than by the size of the candidate pool.  We also assume that the number of good alternatives $|\mathcal{G}|$ remains bounded as $k$ increases. This captures the realistic notion that high-quality options are relatively scarce in large pools of candidates.



Under the asymptotic regime of $k\to \infty$, a budget allocation algorithm is measured by the growth rate of the total evaluation budget $B$ required for ensuring an asymptotically non-zero PGS$_m$. Algorithms that achieve the minimal growth rate of $B$ are considered ideal for addressing large-scale screening problems and are termed \textit{sample-optimal}. We prove in \ref{subsec: proof_best_order} that the minimal growth rate of $B$ should be at least $\mathcal{O}(k)$. 
Therefore, we can formally define the sample optimality in terms of PGS$_m$ as follows.
\begin{definition}
    \label{def: rate_optimality_PCS}
    A budget allocation algorithm is sample-optimal if there exists a constant $c>0$ such that
    \begin{equation}\label{eqn:ro}
        \liminf_{k\to\infty}\, {\rm PGS}_m >0\ {\rm for}\ B=ck.
    \end{equation}
\end{definition}

Beyond sample optimality, another important property of a budget allocation algorithm is its consistency \citep{li2023surprising}. While sample optimality focuses on minimizing the growth rate of the required budget, consistency ensures that the algorithm can eventually succeed—i.e., achieve a PGS\(_m\) close to 1—when the total evaluation budget is sufficiently large relative to the problem scale \( k \) (i.e., increases at a slightly faster rate than \( \mathcal{O}(k) \)). The formal definition of consistency is provided as follows.


\begin{definition}
    \label{def: large_scale_consistency} 
A sample-optimal budget allocation algorithm is consistent if for any $\alpha \in (0, 1) $, there exists a constant $c>0$ such that
    \begin{equation}\liminf\limits_{k \to \infty}\, {\rm PGS}_m \geq  1-\alpha, \mbox{ for } B=ck.
    \end{equation}
\end{definition}

While several budget allocation algorithms have been proposed in the simulation optimization and bandit learning literature and may be applied to virtual screening, their applicability remains limited. As reviewed in Section~\ref{subsec: review}, these algorithms often fall short in addressing several key aspects specific to virtual screening, particularly in large-scale settings. These limitations underscore the need for new budget allocation algorithms that can achieve both sample optimality and consistency for large-scale virtual screening.

\section{Algorithm Design and Performance Analysis}
\label{sec: top-m-boundary-crossing}
In this section, we design a new budget allocation algorithm to achieve sample optimality and consistency for large-scale virtual screening. We then analyze the performance of the proposed algorithm. This analysis provides the foundation for establishing the algorithm's theoretical properties in the next section.

\subsection{Budget Allocation via Top-${m}$ Greedy Selection}
\label{subsec: topm_algorithm}

A pragmatic approach of algorithm development is to adapt ideas in the literature used to solve similar but simpler problems like ranking and selection. 
In this paper, we focus on an interesting greedy allocation framework \citep{bayati2020unreasonable, li2023surprising}. Specifically, we build on the explore-first greedy (EFG) algorithm proposed by \cite{li2023surprising}, which features a straightforward two-phase design: In the exploration phase, a fixed number \( n_0 \) of observations are collected from each alternative. In the greedy phase, the remaining budget is allocated sequentially to the current best alternative with the highest sample mean at each round until the total budget is exhausted. 


To adapt the two-phase EFG algorithm for our budget allocation setting, we introduce only a simple modification to the greedy phase: instead of focusing solely on the single current best alternative at each round, we evaluate the current top-\(m\) alternatives. Specifically, At each round of the greedy phase, one observation is collected for each of the current top-\(m\) alternatives.
Suppose that at round $t=0,1, 2,\dots,(B-n_0k)/m-1$, each alternative $i\in\mathcal{K}$ has received $n_i(t)$ observations with $\sum_{i\in\mathcal{K}}n_i(t)=n_0 k + mt$ and its current sample mean is denoted by $\bar{X}_i(n_i(t))$. 
To determine the alternatives to evaluate in the next round, we utilize the operator 
$\stackrel{1, \dots,  m}{ \argmax}$, which identifies the indices of the top-$m$ elements in descending order.  Then, at the subsequent round $t+1$, the EFG-$m$ algorithm collects one observation for each alternative in the set defined by
\begin{equation}\label{eqn: sampling}
S(t)=\stackrel{1, \dots,  m}{ \argmax}_{i \in \mathcal{K}} \bar{X}_i(n_i(t)).
\end{equation}
The detailed EFG-$m$ algorithm is presented as Algorithm \ref{algorithm: topm}. 

    \RestyleAlgo{ruled}
    \LinesNumbered
    \SetAlgorithmName{Algorithm}{Algorithm}{Algorithm}
    \SetAlgoCaptionLayout{centerline}
    \begin{algorithm}[hbtp]
    \caption{\textbf{Explore-First Top-${m}$ Greedy (EFG-$m$) Algorithm}}
        \label{algorithm: topm}
    \KwIn{$m$, $k$ alternatives and their prompts, the total evaluation budget $B=ck$, the evaluation budget allocated to the exploration phase $n_0k (<B)$, and the LLM}
    \For(\tcp*[f]{{Exploration Phase}}){$i=1$ \KwTo $k$ }{
        prompt the LLM $n_0$ times for alternative $i$  to get $n_0$ observations $x_{i1},\ldots,x_{in_0}$\;
    set $\bar X_i(n_0)=\frac{1}{n_0}\sum_{j=1}^{n_0} x_{ij}$ and let $n_i=n_0$\;
    }
    \While(\tcp*[f]{{Top-m Greedy Phase}}){$\sum_{i=1}^k n_i < B$ }{
    let $\mathcal{S} = \,  \stackrel{1, \dots, m}{ \argmax}_{i \in \{1, \ldots, k\}} \bar X_i(n_i)$\;
      \For{$j \in\mathcal{S}$}{
      prompt the LLM for alternative $j$ to get one observation $x_{j}$; 
      
      update $\bar X_{j}(n_j+1) = \frac{1}{n_j+1}\left[n_j\bar X_{j}(n_j) + x_j\right]$ and let $n_j = n_j+1$\;
      }
    }
    \KwOut{$\stackrel{1, \dots, m}{ \argmax}_{i \in \{1, \ldots, k\}} \bar X_i(n_i)$}
    \end{algorithm}


The EFG-\(m\) algorithm offers two main advantages. First, it is simple to implement. The exploration phase adopts a simple equal-allocation scheme, and identifying the top-\(m\) alternatives in the greedy phase—as shown in Equation~\eqref{eqn: sampling}—is computationally straightforward and fast  (see Section~\ref{subsec: other_algorithms} for the implementation guidance). This simplicity is particularly noteworthy, as it makes the algorithm easy for practitioners to understand and adopt. Second, the straightforward structure of EFG-\(m\) lends itself to flexibility, opening up opportunities to further improve its efficiency by adjusting the budget allocation in both the exploration and greedy phases. We will explore this flexibility in Section~\ref{sec: enhancements}.
Despite the structural simplicity of the EFG-$m$ algorithm, the analysis of its PCS$_m$ and PGS$_m$ is challenging due to its adaptive selection dynamics in the greedy phase. Inspired by the boundary-crossing framework proposed
by \cite{li2023surprising}, we are equipped with a new perspective to dissect the complicated selection process of EFG-$m$ and develop useful lower bounds for the PGS$_m$ in the rest of this section. These bounds are critical for demonstrating the sample optimality and consistency of the EFG-$m$ algorithm. To facilitate understanding, we first examine PCS$_m$ to introduce the core ideas and then extend the analysis to PGS$_m$.

\subsection{Boundary-Crossing Framework for the PCS$_{\mathbf{m}}$}
\label{subsec: boundary_crossing_PCS}
To analyze the PCS$_m$, the key task is to investigate when a given evaluation budget $B$ is sufficient to ensure a correct screening. This task is generally difficult due to the complicated selection process of EFG-$m$, where the sample sizes of all alternatives are inter-connected with all sample means, as outlined in Equation \eqref{eqn: sampling}. To facilitate the analysis, \cite{li2023surprising} introduced a boundary-crossing framework for the standard EFG algorithm (when $m=1$). It uncovers an important insight about the EFG's top-1 greedy selection process: the inferior alternatives continue to receive observations until their running averages fall below that of the best alternative. This implies, while precisely determining the sample size for each inferior alternative is difficult, its sample size can be upper bounded by the boundary-crossing time of its running average from a boundary determined by the best alternative. 
With this framework, we are able to shift our focus from the complex selection dynamics to the more tractable task of analyzing the boundary-crossing times of inferior alternatives, thus making the analysis of PCS$_m$ feasible.

However, the boundary-crossing framework developed for the EFG may not readily apply to EFG-$m$ due to the inherent differences in their selection mechanisms. Unlike the EFG, the EFG-$m$ allows for multiple alternatives to be evaluated at each round, increasing the opportunities for all alternatives to receive observations. Specifically, for any inferior alternative $i\in\mathcal{K}\setminus\mathcal{S}^*$, it may still receive observations even after it is statistically dominated by the best alternative. Therefore, the boundary-crossing time of each inferior alternative is not solely determined by the best alternative, adding complexity to its analysis. 



In what follows, we first consider the boundary-crossing time of each inferior alternative $i\in\mathcal{K}\setminus\mathcal{S}^*$. Under the top-$m$ greedy selection mechanism, once the sample mean of alternative $i$ drops below those of all the top-$m$ alternatives in $\mathcal{S}^*$, it will not be sampled at this round. Moreover, if at this point each of the top-$m$ alternatives has reached the minimum of its running average, then alternative $i$ will be excluded from evaluation in subsequent rounds. This exclusion is due to its consistent dominance by the top-$m$ alternatives. To formalize this, we let $\bar{X}_j^* = \min_{n \in [n_0, \infty)} \bar{X}_j(n)$ represent the minimum running average for $n \in [n_0, \infty)$ for each alternative $j \in \mathcal{S}^*$.
Then, for any inferior alternative $i\in\mathcal{K}\setminus\mathcal{S}^*$, it will never be evaluated again and receive any new observations from round $t$ if
\[
\bar{X}_i(n_i(t))\leq \bar{X}_j^*, \mbox{ for all } j\in\mathcal{S}^*.
\]
This implies that $\bar{X}^{**}=\min_{j\in\mathcal{S}^*}\bar{X}_j^*$ serves as a natural boundary with respect to the boundary-crossing time of alternative $i$ and we summarize the result as follows.
\begin{observation}\label{obs: boundary} Regardless of the total evaluation budget, for each inferior alternative $i\in\mathcal{K}\setminus\mathcal{S}^*$, its sample size can be upper bounded by the boundary-crossing time
\begin{eqnarray}\label{eqn: bound_TS}
    N_i(\bar{X}^{**};n_0):= \inf\{n \in [n_0, \infty): \bar X_i(n)\leq \bar{X}^{**} \}.
\end{eqnarray} 
\end{observation}
\begin{remark}
\label{rem: independence}
    \textcolor{black}{The sample mean process $\{\bar{X}_i(n):n=1,2,\ldots\}$ of each alternative $i$ is defined over $n=1$ to $\infty$. This definition does not require the alternative to be evaluated infinitely many times in practice, but it is used to support our theoretical analysis. Building upon it, the corresponding minimal running average \( \bar{X}_j^* \) are mutually independent across \( j \in \mathcal{S}^* \), and the boundary-crossing times \( N_i(\cdot; n_0) \)  are fully determined by the sample path itself, independent of the evaluation budget and its allocation. These quantities, along with their nice properties, form the foundation of our boundary-crossing analysis. }
\end{remark}
\begin{remark}
\label{rem: infinite}
    Notice that \( N_i(\bar{X}^{**};n_0) \) may be infinite. Consider a sample path where \( \bar{X}^{**} \) is significantly low—lower than \( \mu_i \) for some inferior alternative \( i \). In this case, it is possible that \( \bar{X}_i(n) > \bar{X}^{**} \) for all \( n \in [n_0, \infty) \). If this happens, \( N_i(\bar{X}^{**};n_0) \) will be infinite in the sample path, and the algorithm will never select the correct subset regardless of the total evaluation budget. However, this will not affect our analysis.
\end{remark}

The $\bar{X}^{**}$, determined by the entire set of top-$m$ alternatives (i.e., $\mathcal{S}^*$), sets a common boundary for all inferior alternatives and determines their boundary-crossing times. It also plays a crucial role in ensuring a correct screening. To elaborate, suppose that all top-$m$ alternatives reach their minimum running averages, thus reaching the boundary $\bar{X}^{**}$. If the remaining budget is sufficient to lower the sample means of all inferior alternative below the boundary $\bar{X}^{**}$, then the true top-$m$ alternatives remain as the sample top $m$ until the evaluation budget is exhausted, guaranteeing a correct screening. We summarize this observation as follows.
\begin{observation}\label{obs: timelines}
A correct screening is ensured if the evaluation budget in the greedy phase, i.e., $B-n_0k$, is sufficient to cover the following two kinds of key timelines:
\begin{enumerate}
    \item $T_{\mathcal{S}^{*}}^{j}\,(j\in\mathcal{S}^*)$, the number of greedy selection rounds on each alternatives $j$ such that its running average achieves the minimum $\bar{X}_j^{*}$;
    \item $T_{\mathcal{K}\setminus\mathcal{S}^*}$, the number of greedy selection rounds on the inferior alternatives in $\mathcal{K}\setminus\mathcal{S}^*$ such that all their running averages first drop below the boundary $\bar{X}^{**}=\min_{j\in\mathcal{S}^*} \bar{X}_j^{*}$.
\end{enumerate} 
\end{observation}

From the observation above, the PCS$_m$ of the EFG-$m$ algorithm can be bounded below:
\begin{eqnarray}\label{eqn: PCS}
     \mbox{PCS}_m \geq \Pr\left\{ B-n_0k\geq m T_{\mathcal{K}\setminus\mathcal{S}^*}+m\sum_{j\in\mathcal{S}^*}T_{\mathcal{S}^{*}}^j\right\}.
\end{eqnarray}
Given the maximal sample size of each inferior alternative $i\in \mathcal{K}\setminus\mathcal{S}^*$, as specified in Observation \ref{obs: boundary}, we have that $T_{\mathcal{K}\setminus\mathcal{S}^*}$ must satisfy 
\begin{equation}\label{eqn: bound_K-S}
    T_{\mathcal{K}\setminus\mathcal{S}^*} \leq\ \sum_{i\in\mathcal{K}\setminus\mathcal{S}^*} \left\{N_i (X^{**}; n_0)-n_0\right\},
\end{equation}
where ``$-n_0$" is included to subtract the observations count for each alternative during the exploration phase. The upper bound is constructed based on the fact that at each round in $T_{\mathcal{K}\setminus\mathcal{S}^*}$, at least one inferior alternative is sampled. As the EFG-$m$ allows for simultaneously evaluation of multiple alternatives at each round, this assumption may render the upper bound conservative. Nevertheless, we illustrate that this bound is achievable through a simplified scenario where the observations of the top-$m$ alternatives have no random noise. In this scenario, we set $n_0=1$ and $m=2$, and visualize a specific selection process of the EFG-$m$ in Figure \ref{fig: top_2_sampling}.

\begin{figure}[htbp]
    \begin{center}
\includegraphics[width=0.9\textwidth]{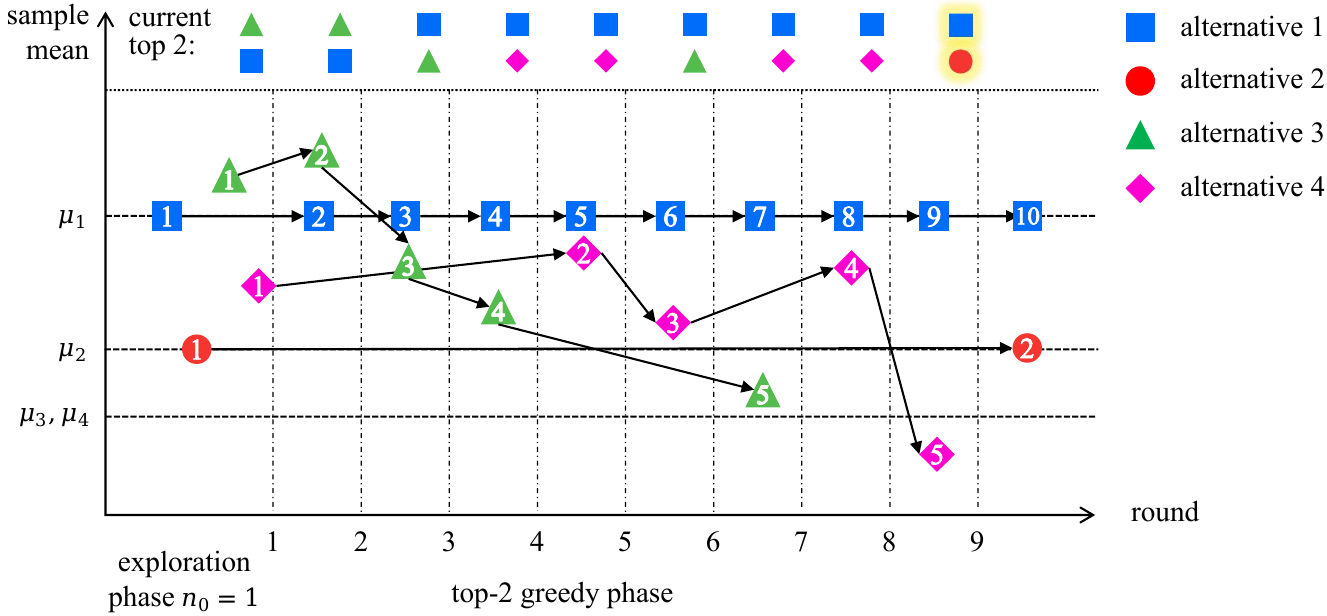}
    \end{center}
    \caption{Top-2 greedy selection process of an example problem with $4$ alternatives where $\bar X_1(n) = \mu_1$ and $\bar X_2(n) = \mu_2, \forall n \geq 1$. Numbers in the markers represent the total sample sizes.}
    \label{fig: top_2_sampling}
\end{figure}

Additionally, by the virtue of definition, it is straightforward that $T_{\mathcal{S}^{*}}^{j}=\argmin_{n \in [n_0, \infty)} \bar{X}_j(n)$ for any $j\in\mathcal{S}^*$. By plugging Equation \eqref{eqn: bound_K-S} into \eqref{eqn: PCS}, we can obtain a lower bound for the desired PCS$_m$:
\begin{eqnarray}
    \label{eq: general_PCS2lbound}
\notag\mbox{PCS}_m &\geq &  \Pr \left\{ B-n_0k \geq m  \left[\sum_{i\in\mathcal{K}\setminus\mathcal{S}^*} \left\{N_i (X^{**}; n_0)-n_0\right\} + \sum_{j\in\mathcal{S}^*} \left\{\argmin_{n \in [n_0, \infty)} \bar{X}_j(n)-n_0\right\}\right]\right\}\\
& = & \Pr \left\{ B-n_0k \geq m \left[\sum_{i\in\mathcal{K}\setminus\mathcal{S}^*} N_i (X^{**}; n_0) + \sum_{j\in\mathcal{S}^*} \argmin_{n \in [n_0, \infty)} \bar{X}_j(n)-n_0k\right]\right\}.
\end{eqnarray}
Equation \eqref{eq: general_PCS2lbound} showcases the power of a boundary-crossing perspective for analyzing the performance of the EFG-$m$ algorithm. Although it seems complicated, it enables us to prove the sample optimality and consistency of the EFG-$m$ algorithm.

\subsection{From PCS$_\mathbf{m}$ to PGS$_{\mathbf{m}}$}
\label{subsec: boundary_crossing_PGS}
Now we extend the boundary-crossing analysis developed above to study the PGS$_m$, where the number of good alternatives is greater than $m$, i.e., $|\mathcal{G}|>m$. For the PCS$_m$ scenario detailed in Section \ref{subsec: boundary_crossing_PCS}, the boundary, which is a crucial element in the boundary-crossing framework, can be clearly defined by the optimal subset $\mathcal{S}^*$. However, such a boundary is no longer readily available when the target turns into selecting a good subset from a pool of good alternatives $\mathcal{G}$. This  arises from the existence of multiple viable good subsets within $\mathcal{G}$, each potentially suggesting a different choice of the boundary.

To address the challenge, we define the minimum running average of each good alternative $l$ in $\mathcal{G}$ as $\bar{X}_l^*=\min_{n \in [n_0, \infty)}\bar{X}_l(n)$. These values are then ranked in descending order:
\begin{equation}\label{eqn: argmin_ordering}
\bar{X}_{(1)}^* \geq \bar{X}_{(2)}^*\geq\dots\geq \bar{X}_{(g)}^*,
\end{equation}
where $g=|\mathcal{G}|$. For illustrative purposes, consider a simple scenario where the observations of each good alternative have no random noise, such that $\bar{X}_l(n) \equiv \mu_l$ for all $n \in [n_0, \infty)$ and all $l\in\mathcal{G}$. In this ideal scenario, 
the top-$m$ alternatives (specifically alternatives $1,2,\dots,m$) will consistently outperform all the other alternatives and continue to be evaluated until the end of algorithm, whenever the running averages of all inferior alternatives in $\mathcal{K}\setminus\mathcal{G}$ fall below $\mu_{m}$, namely the $m$th largest mean. Thus, $\mu_{m}$ effectively acts as the boundary. Extending this logic to the general case with random noise, we analogously set the boundary as $\bar{X}_{(m)}^*$, which represents the $m$th largest value among the minimum running averages $\bar{X}_{(1)}^*,\bar{X}_{(2)}^*,\dots, \bar{X}_{(g)}^*$. It is important to notice that the boundary $\bar{X}_{(m)}^*$ is determined by the entire set of good alternatives. This frees us from having to select a specific subset of good alternatives to determine the boundary, but rather automatically sets the boundary according to the selection processes of all good alternatives. 

Given the boundary $\bar{X}_{(m)}^*$, the boundary-crossing times of all inferior alternatives can be observed as follows, which is an analogous result to Observation \ref{obs: boundary} for the PCS$_m$ scenario.
\begin{observation}\label{obs: boundary_PGS}
For each inferior alternative $i\in\mathcal{K}\setminus\mathcal{G}$, its sample size can be upper bounded by the boundary-crossing time
\begin{eqnarray}\label{eqn: bound_TS2}
    N_i(\bar{X}^{*}_{(m)};n_0):= \inf\{n \in [n_0, \infty): \bar X_i(n)\leq \bar{X}^{*}_{(m)} \}.
\end{eqnarray} 
\end{observation}

Additionally, the boundary $\bar{X}_{(m)}^*$ basically divides the good alternatives in $\mathcal{G}$ into two disjoint subsets: $\mathcal{G}' = \{(1),(2),\dots,(m)\}$ and $\mathcal{G}''=\{(m+1),(m+2),\dots,(g)\}$. Notice that for each alternative $l$ from $\mathcal{G}''$, whenever it achieves the minimum running average $\bar{X}_l^*$, its sample mean will remain consistently smaller than those of the $m$ alternatives in $\mathcal{G}'$ (by Equation \eqref{eqn: argmin_ordering}) and never be selected again. Then, it is straightforward to obtain the following interesting observation about the good alternatives in $\mathcal{G}''$. 
\begin{observation}\label{obs: boundary_G''}
All the good alternatives from $\mathcal{G}''$ will be ultimately dominated by the $m$ good alternatives in $\mathcal{G}'$. Additionally, for each alternative $l\in\mathcal{G}''$, its sample size can be upper bounded by $\argmin_{n \in [n_0, \infty)}\bar{X}_l(n)$.
\end{observation}

\begin{remark}
\label{rem: bounds}
While the upper bound in Equation \eqref{eqn: bound_TS2} applies to all alternatives, we introduce a new upper bound specifically for each alternative \( l\in\mathcal{G}'' \) in Observation \ref{obs: boundary_G''}. As shown later in Section \ref{subsec: arguments},  this new upper bound is crucial to establish the sample optimality of EFG-\( m \). In contrast, applying the upper bound in Equation \eqref{eqn: bound_TS2} to alternatives in \( \mathcal{G}'' \) would render the approach invalid.
\end{remark}

As we can see from Observations \ref{obs: boundary_PGS} and \ref{obs: boundary_G''}, the good alternatives included in $\mathcal{G}'$ ultimately dominates not only all the inferior alternatives in $\mathcal{K}\setminus\mathcal{G}$ but also the other good alternatives in  $\mathcal{G}''$, as the selection process proceeds. Specifically, when this event occurs within the evaluation budget $B$, the EFG-$m$ selects the subset $\mathcal{G}'$ at terminal round, leading to a good screening. To formally capture this behavior, we can introduce an observation analogous to Observation \ref{obs: timelines}.
\begin{observation}\label{obs: timelines_PGS}
A good screening is ensured if the evaluation budget in the greedy selection phase, i.e., $B-n_0k$, is sufficient to cover the following three kinds of key timelines:
\begin{enumerate}
    \item $T_{\mathcal{G}'}^j\,(j\in\mathcal{G}')$, the number of greedy selection rounds on alternatives $j\in\mathcal{G}'$ such that its minimum  running average $\bar{X}_j^{*}$ is reached;
    \item $T_{\mathcal{K}\setminus\mathcal{G}}$, the number of greedy selection rounds on the inferior alternatives in $\mathcal{K}\setminus\mathcal{G}$ such that all their running averages first drop below $\bar{X}^{*}_{(m)}$;
    \item $T_{\mathcal{G}''}^l\,(l\in\mathcal{G}'')$, the number of greedy selection rounds on good alternative $l\in \mathcal{G}''$ such that its minimum running average $\bar{X}_l^{*}$ is reached.
\end{enumerate} 
\end{observation}

Compared to its analogous Observation \ref{obs: timelines}, Observation \ref{obs: timelines_PGS} involves an additional timeline $T_{\mathcal{G}''}^l$, which is crucial for tracking the selection dynamics of the good alternatives within $\mathcal{G}''$. Then, the PGS$_m$ can be expressed as
\begin{eqnarray}\label{eqn: bound_PGS}
    \mbox{PGS}_m \geq \Pr\left\{ B -   n_0k\geq m\sum_{j\in\mathcal{G}'} T_{\mathcal{G}'}^j+mT_{\mathcal{K}\setminus\mathcal{G}}+m\sum_{l\in\mathcal{G}''}T_{\mathcal{G}''}^l\right\}.
\end{eqnarray}
Similarly to Equation  \eqref{eqn: bound_K-S}, we have
\begin{equation}\label{eqn: bound_K-G}
    T_{\mathcal{K}\setminus\mathcal{G}} \leq\ \sum_{i\in\mathcal{K}\setminus\mathcal{G}} \left\{N_i (\bar{X}_{(m)}^{*}; n_0)-n_0\right\}.
\end{equation}
Meanwhile, Observations \ref{obs: boundary_PGS} and \ref{obs: boundary_G''} show that
\begin{eqnarray}\label{eqn: bound_argmin}
    T_{\mathcal{G}'}^j = \argmin_{n \in [n_0, \infty)} \bar{X}_j(n)-n_0, \mbox{  and  } T_{\mathcal{G}''}^l = \argmin_{n \in [n_0, \infty)} \bar{X}_l(n)-n_0.
\end{eqnarray}
Plugging Equations  \eqref{eqn: bound_K-G} and \eqref{eqn: bound_argmin} into \eqref{eqn: bound_PGS} leads to
\begin{equation}\label{eqn: bound_PGS2}
\begin{aligned}
\mbox{PGS}_m &\geq\ \Pr \left\{ B-n_0k \geq m \left[\sum_{i\in\mathcal{K}\setminus\mathcal{G}} \left\{N_i (X_{(m)}^{*}; n_0)-n_0\right\} + \sum_{j\in\mathcal{G}} \left\{\argmin_{n \in [n_0, \infty)} \bar{X}_j(n)-n_0\right\}\right]\right\}\\
&\geq\ \Pr \left\{ B-n_0k \geq m \left[\sum_{i\in\mathcal{K}\setminus\mathcal{G}} N_i (X_{(m)}^{*}; n_0)+ \sum_{j\in\mathcal{G}} \argmin_{n \in [n_0, \infty)} \bar{X}_j(n)-n_0k\right]\right\}.
\end{aligned}
\end{equation}

The PGS$_m$ lower bound above is an extension of the PCS$_m$ lower bound in Equation \eqref{eq: general_PCS2lbound}. When $\mathcal{S}^*=\mathcal{G}$, these two lower bounds become identical. In addition, it is worth acknowledging that the PGS$_m$ lower bound may not be tight. It is because it is constructed based on a particular scenario that leads to a good screening. Nonetheless, as we will show in the following section, this PGS$_m$ lower bound is sufficient to establish the sample optimality and consistency of the EFG-$m$ algorithm.


\section{Properties of the EFG-${m}$ Algorithm}
\label{sec: sample-optimality}

In this section, we prove the desired properties of the EFG-$m$ algorithm, particularly its sample optimality and consistency in terms of the PGS$_m$. We first introduce a necessary assumption on the LLM's evaluation output $X_i$ for each alternative in the asymptotic regime as $k\to\infty$.

\begin{assumption}
    \label{assu: guassian}
    There exists a positive constant $\bar \sigma$ such that
    $ \max\limits_{i\in\mathcal{K}}\sigma_i^2 \leq \bar \sigma^2$ regardless of how large the problem scale $k$ is.
\end{assumption}
This assumption helps to prevent scenarios where the variances of the LLM's evaluation outputs could become unbounded as  $k$ increases. Importantly, it is worth noting that the value of the variance bound $\bar\sigma^2$ does not need to be known for the implementation and analysis of our EFG-$m$ algorithm.

Meanwhile, we also include here the boundedness assumption on  the number of good alternatives (i.e., $|\mathcal{G}|$), which has been explained in Section \ref{subsec: regime}. It helps to streamline our analysis and understanding of the desired properties of EFG-$m$.  
\begin{assumption}\label{assum: bounded_g}
    There exists a positive constant $\bar{g}$ such that $g=|\mathcal{G}|\leq \bar{g}$ regardless of how large $k$ is.
\end{assumption}

\subsection{Sample Optimality}
\label{subsec: arguments}

To establish the sample optimality as defined in Definition \ref{def: rate_optimality_PCS}, it suffices to verify that the PGS$_m$ lower bound outlined in Equation \eqref{eqn: bound_PGS2} does not approach to zero as $k\to\infty$, given a linearly growing evaluation budget $B=ck$ for some positive constant $c$. For ease of presentation, we rewrite Equation \eqref{eqn: bound_PGS2} as
\begin{equation}\label{eqn: bound_PGS3}
\mbox{PGS}_m \geq \Pr \left\{ \frac{B}{mk}-\frac{n_0}{m} +n_0 \geq  \frac{1}{k} \sum_{j\in\mathcal{G}}\argmin_{n \in [n_0, \infty)}  \bar{X}_j(n)
    + \frac{1}{k} \sum_{i\in\mathcal{K}\setminus\mathcal{G}}  N_i \left(X^*_{(m)}; n_0\right)\right\}. 
\end{equation}

The equation above relates to the properties of the running average process $\{\bar{X}_i(n): n=n_0,n_0+1,\dots\}$ for each alternative $i\in\mathcal{K}$, particularly focusing on its minimum  or boundary-crossing time. 
As such, we first summarize  two useful lemmas from \cite{li2023surprising} concerning their properties.
\begin{lemma}\label{lem: argmin}
    For any alternative $j\in\mathcal{G}$, its running average process $\{\bar{X}_j(n): n=n_0,n_0+1,\dots\}$ reaches its minimum within a finite number of observations almost surely, i.e., $\argmin_{n \in [n_0, \infty)}\bar{X}_j(n)<\infty$ almost surely.
\end{lemma}

\begin{lemma}\label{lem: boundary-crossing}
    For any alternative $i\in \mathcal{K}\setminus\mathcal{G}$ and a fixed boundary $x>\mu_i$, its associated boundary-crossing time $N_i \left(x; n_0\right)$ of the running average process $\{\bar{X}_i(n): n=n_0,n_0+1,\dots\}$ has a finite expectation. Specifically, 
    \[
    \mathrm{E}[N_i(x;n_0)]\leq C\left(\frac{x-\mu_i}{\bar\sigma};n_0\right),
    \]
    where $C(z;n_0)=\mathrm{E}[\inf\{n \in [n_0, \infty): \bar{Z}(n)\leq z\}]$ and  $\{\bar{Z}(n): n =  n_0, n_0+1,\dots\}$ denotes the running average process of a sequence of independent standard normal random variables.
\end{lemma}

Based on the two lemmas, we  establish the sample optimality of the EFG-$m$  by first presenting two key arguments. These arguments not only sketch our proof, but also explain the insights behind. 
 
\begin{argument}
    \label{argu4}
     $\sum\limits_{j\in\mathcal{G}} \argmin_{n \in [n_0, \infty)}  \bar{X}_j(n)<\infty$ almost surely.
    \end{argument}

Argument \ref{argu4} is straightforward by combining Assumption \ref{assum: bounded_g} and Lemma \ref{lem: argmin}. It indicates that all the good alternatives in $\mathcal{G}$ are guaranteed to reach the minimums of their running averages within a finite number of observations almost surely. Consequently, the boundary $\bar{X}_{(m)}^*$ specified in the boundary-crossing framework of the EFG-$m$ algorithm can also be reached within a finite number of observations almost surely.  This implies, given the evaluation budget $B=ck$ as $k\to\infty$, the requirement of reaching the boundary $\bar{X}_{(m)}^*$ does not affect the sample optimality.

    \begin{argument}
        \label{argu5}
        Under the condition $\bar X^*_{(m)} \geq \mu_m - \delta_0$ with  $\delta_0 \in (0, \delta)$,  $\limsup\limits_{k\to\infty}\, \frac{1}{k} \sum\limits_{i\in\mathcal{K}\setminus\mathcal{G}} N_i (\bar X^*_{(m)}; n_0)\leq C\left(\frac{\delta - \delta_0}{\bar \sigma} ; n_0\right)$ almost surely.
        \end{argument}
        
Under the condition $\bar X^*_{(m)} \geq \mu_m - \delta_0$, for each inferior alternative $i\in\mathcal{K}\setminus\mathcal{G}$, because $\mu_i<\mu_m-\delta$,
                    \begin{eqnarray}
                        \label{eq: bound_transform}
                        \begin{aligned}
                        N_i \left(X^*_{(m)}; n_0\right) 
                        &=\ \inf\left\{n \in [n_0, \infty): \bar X_i(n) \leq \bar X^*_{(m)}\right\}\\
                        &\leq\ \inf\left\{n \in [n_0, \infty): \bar X_i(n) \leq\mu_m-\delta_0\right\} \\
                       &\leq\  \inf\left\{n \in [n_0, \infty): \bar X_i(n) \leq\mu_i+\delta-\delta_0\right\}\\
                       & =\ N_i \left(\mu_i+\delta-\delta_0; n_0\right).
                   \end{aligned}
                    \end{eqnarray}
Notice that \( N_i \left(X^*_{(m)}; n_0\right) \) for different \( i \) are generally correlated due to the shared random boundary \( X^*_{(m)} \). However, their corresponding upper bounds, \( N_i \left(\mu_i + \delta - \delta_0; n_0\right) \), are mutually independent because the boundary becomes the constant \( \mu_i + \delta - \delta_0 \) (see Remark \ref{rem: independence} for a reminder). By leveraging this independence, we may use the strong law of large numbers (SLLN) and Lemma \ref{lem: boundary-crossing} to obtain Argument \ref{argu5}.

Argument \ref{argu5} highlights an important feature of the EFG-$m$ algorithm that, whenever the condition $\bar X^*_{(m)}\geq \mu_m - \delta_0$ is satisfied, the maximal evaluation budget allocated to all the inferior alternatives, quantified by the sum of their corresponding boundary-crossing times, is  $\mathcal{O}(k)$.  This feature is crucial to understand the sample optimality of the EFG-$m$ algorithm. It also suggests that the PGS$_m$ might be closely related to the probability of fulfilling the condition $\bar X^*_{(m)}\geq \mu_m - \delta_0$.

\vspace*{0.3cm}
Building on the two arguments above, we can rigorously prove the sample optimality of the EFG-$m$ algorithm, which is presented in the following theorem. The detailed proof is deferred to \ref{subsec: proof_nonnormal}. 

\begin{theorem}
    \label{thm: nonnormalPGS}
    Suppose that Assumptions \ref{assu: guassian} and \ref{assum: bounded_g} hold. If the total evaluation budget $B=(n_0 + n_g) k$ and $n_g > m\left[C\left(\frac{\delta}{\bar \sigma} ; n_0\right)-n_0\right]$, the PGS$_m$ of the EFG-$m$ algorithm satisfies
    \begin{eqnarray}\label{eqn: limitingPGS_bound}
    \liminf\limits_{k \to \infty} {\rm PGS}_m \geq  \Pr\left\{\bar X_{(m)}^* \geq  \mu_m - \delta_0\right\} > 0,
    \end{eqnarray}
    where $\delta_0 \in (0, \delta)$ is a constant such that $m\left[C\left(\frac{\delta - \delta_0}{\bar \sigma}; n_0\right)-n_0\right] = n_g$. Therefore, the EFG-$m$ algorithm is sample optimal.
    \end{theorem}

Notice that the PGS$_m$ lower bound $\Pr\left\{\bar X_{(m)}^* \geq  \mu_m- \delta_0\right\}$  above is not tight. But it aligns well with our intuitive understandings of the PGS$_m$. By Equation \eqref{eqn: argmin_ordering}, $\bar{X}_{(m)}^*$ is the $m$th order statistics among the minimum running averages of $g$ good alternatives in $\mathcal{G}$. The lower bound $\Pr\left\{\bar X_{(m)}^* \geq  \mu_m - \delta_0\right\}$ refers to the probability that at least $m$ of the $g$ good alternatives maintain their running averages above a constant $\delta-\delta_0$, which  generally increases as $g$ increase. This supports our practical observation that  including more good alternatives  often result in a higher PGS$_m$. We remark that the PGS$_m$ lower bound in Theorem \ref{thm: nonnormalPGS} is a worst-case guarantee, which is independent of how the means of the inferior alternatives are distributed. Furthermore, for a fixed $m$, the required greedy budget $m\left[C\left(\frac{\delta}{\bar \sigma} ; n_0\right)-n_0\right]$ decreases exponentially as $n_0$ increases (see Lemma \ref{lem: bct_mean_convergence}). Consequently, when $n_0$ is relatively large, only a small proportion of the total evaluation budget is required for the top-$m$ greedy phase. In our experiments, we allocate 20\% of the total evaluation budget to the top-$m$ greedy phase.

\subsection{Consistency}
\label{subsec: normal_properties}

We further prove the consistency of the EFG-$m$ algorithm.
By the definition of consistency in Definition \ref{def: large_scale_consistency}, it suffices to show that, for any $\alpha\in(0,1)$, these exists a pair of $n_0$ and $n_g$  such that
\begin{eqnarray*}
    \liminf_{k\to\infty}\, \mbox{PGS}_m\geq 1-\alpha.
\end{eqnarray*}
With Theorem \ref{thm: nonnormalPGS}, we may first reformulate the limiting PGS$_m$ above in a more tractable form:
\begin{eqnarray*}
    \liminf_{k\to\infty}\mbox{PGS}_m&\geq &
   \Pr\left\{\bar X_{(m)}^* \geq  \mu_m - \delta_0\right\}\\
    &\geq & \Pr\left\{\bar X_{1}^* \geq  \mu_m - \delta_0,\bar X_{2}^* \geq  \mu_m - \delta_0,\dots,\bar X_{m}^* \geq  \mu_m - \delta_0\right\}\\
    &=&\prod_{j= 1\ldots, m} \Pr\left\{\min_{n \in [n_0, \infty)} \bar X_j(n) \geq  \mu_m - \delta_0\right\}\\
    &\geq &\prod_{j= 1\ldots, m} \Pr\left\{\min_{n \in [n_0, \infty)} \bar X_j(n) \geq  \mu_j - \delta_0\right\} \qquad (\mbox{because }\mu_j\geq \mu_m).
\end{eqnarray*}
 Then our task is modified to find a pair of $n_0$ and $n_g$ such that 
\begin{eqnarray}\label{eqn: n0_find}
    \prod_{j= 1\ldots, m} \Pr\left\{\min_{n \in [n_0, \infty)} \bar X_j(n) \geq  \mu_j - \delta_0\right\}\geq 1-\alpha.
\end{eqnarray}
To proceed, we prepare the following lemma and its proof is included in \ref{subsec: proof_hitting_time_prob}.
\begin{lemma}
\label{lem: hitting_time_prob}
Suppose that Assumption \ref{assu: guassian} holds. Then, we have:
    $$\Pr\left\{\min_{n \in [n_0, \infty)} \bar X_j(n) > \mu_j-\delta_0\right\} \geq 1-2\exp\left(-\frac{n_0\delta_0^2}{2\bar\sigma^2}\right).$$
\end{lemma} 

Using Lemma \ref{lem: hitting_time_prob}, we may readily derive the left-hand-side term in Equation \eqref{eqn: n0_find} as:
\begin{eqnarray}
    \label{eq: lb_normalPACS}
    \notag  \prod_{j= 1\ldots, m} \Pr\left\{\min_{n \in [n_0, \infty)} \bar X_j(n)\geq \mu_j-\delta_0\right\} \geq  \left[1-2\exp\left(-\frac{n_0\delta_0^2}{2\bar \sigma^2}\right)\right]^m\geq 1-2m\exp\left(-\frac{n_0\delta_0^2}{2\bar \sigma^2}\right).
\end{eqnarray}
When $\delta_0$ is fixed to be $\delta/2$, for $\alpha\in(0,1)$, we may choose $n_0$ as $n_0 =  \frac{8 \bar \sigma^2}{\delta^2} \log\left(\frac{2m}{\alpha}\right)$, and subsequently determine $n_g$ using the formula $n_g =m\left[C\left(\frac{\delta - \delta_0}{\bar \sigma};n_0\right)-n_0\right]$ as outlined in Theorem \ref{thm: nonnormalPGS}.  Such choice of $n_0$ and $n_g$ ensures that Equation \eqref{eqn: n0_find} is met, thereby confirming the consistency of the EFG-$m$ algorithm. We formally present this result in the following theorem and its proof is included in \ref{subsec: proof_consistency}.

\begin{theorem}
    \label{thm: consistency}
    Suppose that Assumptions \ref{assu: guassian} and \ref{assum: bounded_g} hold.
    For any $\alpha \in (0, 1)$, if the total evaluation budget $B=(n_0 + n_g) k$ with $n_0 = \frac{8\bar\sigma^2}{\delta^2}\log\frac{2m}{\alpha}$ and $n_g \geq \frac{\alpha}{2} + \frac{4\alpha \bar\sigma^2}{\delta^2}$,
    the PGS$_m$ of the EFG-$m$ algorithm satisfies
       $ \liminf_{k \to \infty} {\rm PGS}_m\geq 1-\alpha$. Thus, the EFG-$m$ algorithm is consistent.
\end{theorem}

Theorem \ref{thm: consistency} not only establishes the consistency of the EFG-\( m \) algorithm but also reveals another important property. In the best arm identification (BAI) literature, worst-case sample complexity is a key measure. It determines the minimum number of observations required to select a good subset of size \( m \) with a given fixed precision $1-\alpha$. The recognized lower bound for worst-case sample complexity in the BAI literature is \( \mathcal{O}\left(\frac{k}{\delta^2} \log\left(\frac{m}{\alpha}\right)\right) \) \citep{kalyanakrishnan2012pac}.  From Theorem \ref{thm: consistency}, it follows directly that the sample complexity of EFG-\( m \), given by \( k(n_0+n_g) \), aligns with this established bound as \( k \to \infty \), suggesting that EFG-\( m \) may achieve the optimal sample complexity in the asymptotic regime $k\to\infty$. Numerical experiments may verify this result for \( \alpha \) and \( \delta \). Notably, for \( m \), the result requires \( k \) to be sufficiently large relative to \( m \). See Section \ref{subsec: numerical_topm_consistency} and Section \ref{subsec: additional_numerical} for detailed numerical results. 

\section{Ranking within the Selected Subset}
\label{sec: ranking}
As previously mentioned in Section \ref{sec: introduction}, for better supporting decision making, it is valuable to not only select a subset of top-performing alternatives but also correctly rank the alternatives within the subset. To address this, we first introduce a new concept of indifference-based ranking for large-scale virtual screening where the means of top alternatives may be very close to each other. We then investigate the effectiveness of the EFG-$m$ algorithm in delivering both good screening and ranking.

\subsection{Indifference-based Ranking}
\label{subsec: PGSR}
When the total evaluation budget $B$ is exhausted, the EFG-$m$ algorithm selects the subset $\mathcal{S}$ based on the terminal sample means of  alternatives in $\mathcal{K}$. Let $\{\pi_1(\mathcal{S}),\pi_2(\mathcal{S}),\dots,\pi_m(\mathcal{S})\}\subset\mathcal{K}$ denote the indices of alternatives included in $\mathcal{S}$. For simplicity, we let $\pi_1(\mathcal{S})<\pi_2(\mathcal{S})<\dots<\pi_m(\mathcal{S})$, indicating $\mu_{\pi_1(\mathcal{S})}\geq \mu_{\pi_2(\mathcal{S})}\geq\dots\geq\mu_{\pi_m(\mathcal{S})}$. To obtain a ranking of  these alternatives, the most straightforward method is by sorting their sample means in descending order. The alternative with a larger sample mean is assigned a higher rank. Clearly, the resulting ranking is correct if  
\begin{eqnarray*}
    \bar{X}_{\pi_1(\mathcal{S})}(n_{\pi_1(\mathcal{S})}(B))>\bar{X}_{\pi_2(\mathcal{S})}(n_{\pi_2(\mathcal{S})}(B))>\dots>\bar{X}_{\pi_m(\mathcal{S})}(n_{\pi_m(\mathcal{S})}(B)),
\end{eqnarray*}
where $n_{\pi_i(\mathcal{S})}(B)$ denotes the sample size  allocated to alternative $\pi_i(\mathcal{S})\in\mathcal{S}$ at the end of algorithm. 

Recall that we introduced an indifference-zone parameter $\delta>0$ in Section \ref{sec: problem}. It refers to the minimal mean difference among alternatives that decision makers feel worthy detecting. Any two alternatives whose mean difference is within $\delta$ are considered indifferent for screening purpose. It follows logically that the ranking between these two alternatives should also be treated as indifferent. To formally extend this concept of indifference to the ranking purpose, we adopt the semiordering principle from \cite{luce1956semiorders} and define the \textit{indifference-based ranking} (IBR) among alternatives. The IBR consists of two binary relations depending on the indifference-zone parameter $\delta$: indifference ($\sim_\delta$) and dominance ($>_\delta$). These two relations are formalized as follows.
\begin{definition} Given an indifference-zone parameter $\delta >0$, for any two alternatives $A$ and $B$ whose means are denoted by $\mu_A$ and $\mu_B$, define
    \label{def_IBR}
    \begin{itemize}
        \item [(1)] $A \sim_\delta B$  if and only if  $|\mu_A - \mu_B| < \delta$;
        \item [(2)] $ A >_\delta B$   if and only if  $\mu_A - \mu_B \geq \delta$.
    \end{itemize}
\end{definition}
Notice that the dominance relation ``$>_\delta$" above is transitive. More specifically, for any three alternatives $A$, $B$ and $C$, if $ A >_\delta B$  (i.e., $\mu_A - \mu_B \geq \delta$) and  $ B >_\delta C$  (i.e., $\mu_B - \mu_C \geq \delta$), then it must be true that $ A >_\delta C$, because $\mu_A-\mu_C= \mu_A-\mu_B+\mu_B-\mu_C\geq 2\delta>\delta$.  This transitive property is crucial to ensure consistency when alternatives are ranked based on their pairwise dominance.


With IBR, we are ready to define a good ranking within the selected subset $\mathcal{S}$.  Compared to the case of correct ranking, we only need to care about the correctness of ranking among alternatives which exhibits dominance relations. By the transitivity property of the dominance relation in Definition  \ref{def_IBR},  we declare that a good ranking of $\mathcal{S}$ is obtained if the following condition
\[
\bar{X}_{\pi_i(S)}(n_{\pi_i(S)}(B))>\bar{X}_{\pi_j(S)}(n_{\pi_j(S)}(B)), \mbox{ if } \pi_i(\mathcal{S})>_\delta \pi_j(\mathcal{S}), \ \forall \pi_i(\mathcal{S}), \pi_j(\mathcal{S})\in\mathcal{S}
\]
is satisfied, or equivalently in a neater form
\begin{eqnarray}\label{eqn: GR}
\bar{X}_{i}(n_{i}(B))>\bar{X}_{j}(n_{j}(B)), \mbox{ if } \mu_i-\mu_j\geq \delta, \ \forall i,j\in\mathcal{S}.
\end{eqnarray}
To measure the effectiveness of the EFG-$m$ algorithm when considering both the screening and ranking, we propose a new metric called the probability of good screening and ranking (PGSR$_m$), which is expressed as
\begin{eqnarray}
    \mbox{PGSR}_m 
    =\Pr\left\{\{\mathcal{S}\subset \mathcal{G}\} \cap \left\{  \bar X_i(n_i(B))  >  \bar X_j(n_j(B)) \mbox{ if } \mu_{i} - \mu_{j} > \delta,  \forall i, j \in \mathcal{S} \right\} \right\}.
\end{eqnarray}
The definition of PGSR$_m$ allows us to measure the performance of a budget allocation algorithm for good screening and ranking when solving large-scale virtual screening problems.

\subsection{PGSR$_{\mathbf{m}}$ of the EFG-${m}$ Algorithm}
\label{subsec: PGSR_properties}


    \label{eq: def_PCSR}

Similar to the PGS$_m$ analysis, we may leverage the boundary-crossing framework to form a lower bound for the PGSR$_m$.  As noticed in Observation \ref{obs: timelines_PGS} or Equation \eqref{eqn: bound_PGS}, after at most 
\begin{eqnarray}\label{eqn: T_GS}
T_{GS}=\sum_{j\in\mathcal{G}'}T_{\mathcal{G}'}^j+\sum_{l\in\mathcal{G}''}T_{\mathcal{G}''}^l+T_{\mathcal{K}\setminus\mathcal{G}}
\end{eqnarray}
rounds of top-$m$  greedy selection, the EFG-$m$ algorithm will concentrate on the set of good alternatives in $\mathcal{G}'$ until the end of algorithm. Consequently, a good screening event is ensured and the subset $$\mathcal{S}=\mathcal{G}'=\{(1),(2),\dots,(m)\}$$ is selected. When taking PGSR$_m$ into consideration, we are required to additionally ensure the good ranking within $\mathcal{G}'$, or equivalently, the condition detailed in Equation \eqref{eqn: GR}. Then, beyond the PGS$_m$ analysis, the key challenge in analyzing the PGSR$_m$ is to investigate how many additional rounds, if any are required to satisfy condition \eqref{eqn: GR}.

Suppose that the terminal sample mean of each alternative $j\in\mathcal{G}'$ lies within $\delta/2$ of its true mean, namely,
\begin{eqnarray}\label{eqn: samplemean_delta}
    |\bar{X}_j(n_j(B))-\mu_j|<\delta/2.
\end{eqnarray}
Then, condition \eqref{eqn: GR} can be straightforwardly met because, for any alternatives $i,j\in\mathcal{G}'$ with $\mu_i-\mu_j\geq \delta$, 
\[
\bar{X}_i(n_i(B))-\bar{X}_j(n_j(B)) =\left[\bar{X}_i(n_i(B))-\mu_i\right]-\left[\bar{X}_i(n_i(B))-\mu_i\right]+\left[\mu_i-\mu_j\right]> -\delta/2-\delta/2+\delta =0.
\]
To formally present the insight above, we define 
\begin{eqnarray}
L^j_{\mathcal{G}'}\left(\delta/2; n_0\right)= \sup \left\{n \in [n_0, \infty) : |\bar X_j(n)  - \mu_{j}| \geq \delta/2\right\}
\end{eqnarray}
as the \textit{last exit time}  of  the running average process $\{\bar{X}_j(n): n=n_0,n_0+1,\dots\}$ of alternative $j$ from the region $[\mu_j -  \delta/2, \mu_j + \delta/2]$ after $n=n_0$. By definition, once the sample size allocated to alternative $j$ exceeds $L^j_{\mathcal{G}'}(\delta/2;n_0)$, its sample mean will remain satisfying Equation \eqref{eqn: samplemean_delta} until the end of algorithm. In light of this, we tailor Observation \ref{obs: timelines_PGS} for good screening into the following Observation \ref{obs: timelines_PGSR} for the context of ensuring both good screening and ranking. Compared to Observation \ref{obs: timelines_PGS}, Observation \ref{obs: timelines_PGSR} includes a new timeline $L^j_{\mathcal{G}'}(\delta/2;n_0)-n_0$, which is marked as the second point, such that  condition \eqref{eqn: GR} can be satisfied.  
\begin{observation}\label{obs: timelines_PGSR}
A good screening and ranking is ensured if the evaluation budget in the greedy phase, i.e., $B-n_0k$, is sufficient to cover the following four kinds of key timelines:
\begin{enumerate}
    \item $T_{\mathcal{G}'}^j\,(j\in\mathcal{G}')$, the number of greedy selection rounds on alternatives $j\in\mathcal{G}'$ such that its minimum running average $\bar{X}_j^{*}$ is reached;
    \item $L^j_{\mathcal{G}'}(\delta/2;n_0)-n_0$, the number of greedy selection rounds on alternatives $j\in\mathcal{G}'$ such that its running average $\bar{X}_j^{*}$ remains within $\delta/2$ of its true mean $\mu_j$;
    \item $T_{\mathcal{K}\setminus\mathcal{G}}$, the number of greedy selection rounds on the inferior alternatives in $\mathcal{K}\setminus\mathcal{G}$ such that all their running averages first drop below $\bar{X}^{*}_{(m)}$;
    \item $T_{\mathcal{G}''}^l\,(l\in\mathcal{G}'')$, the number of greedy selection rounds on good alternative $l\in \mathcal{G}''$ such that its minimum running average $\bar{X}_l^{*}$ is reached.
\end{enumerate} 
\end{observation}

From the observation above, the maximum number of top-$m$ greedy selection rounds, $T_{GSR}$, to deliver a good screening and ranking is 
\begin{eqnarray}
    \begin{aligned}
        T_{GSR}&=\ \sum_{j\in\mathcal{G}'} \max\left\{T_{\mathcal{G}'}^j,L_{\mathcal{G}'}^j(\delta/2;n_0)-n_0\right\}+T_{\mathcal{K}\setminus\mathcal{G}}+\sum_{l\in\mathcal{G}''}T_{\mathcal{G}''}^l\\
        &=\ \sum_{j\in\mathcal{G}'} \left(L_{\mathcal{G}'}^j(\delta/2;n_0)-T_{\mathcal{G}'}^j\right)^++\sum_{j\in\mathcal{G}'} T_{\mathcal{G}'}^j+T_{\mathcal{K}\setminus\mathcal{G}}+\sum_{l\in\mathcal{G}''}T_{\mathcal{G}''}^l,
    \end{aligned}
\end{eqnarray}
where  $x^+=\max\{x,0\}$. Notice that, in contrast to its analogous term $T_{GS}$ in Equation \eqref{eqn: T_GS}, $T_{GSR}$ involves a new component $\sum_{j\in\mathcal{G}'} \left(L_{\mathcal{G}'}^j(\delta/2;n_0)-\argmin_{n \in [n_0, \infty)}\bar{X}_j(n)\right)^+$. This component essentially represents the extra greedy rounds required to ensure good ranking within the selected subset $\mathcal{G}'$.

Furthermore, Observation \ref{obs: timelines_PGSR} shows that, within a limited evaluation budget $B$, a good screening and ranking is guaranteed if $B\geq n_0k+mT_{GSR}$. In other words,  the PGSR$_m$ can be expressed as
\begin{eqnarray}\label{eqn: bound_PGSR}
    \notag \mbox{PGSR}_m &\geq&\Pr\left\{ B-n_0k \geq m\left[\sum\limits_{j\in\mathcal{G}'} \left(L_{\mathcal{G}'}^j(\delta/2;n_0)-T_{\mathcal{G}'}^j\right)^++\sum\limits_{j\in\mathcal{G}'} T_{\mathcal{G}'}^j+T_{\mathcal{K}\setminus\mathcal{G}}+\sum\limits_{l\in\mathcal{G}''}T_{\mathcal{G}''}^l\right]\right\}\\
    \notag &\geq& \Pr\left\{ B-n_0k \geq m\left[\sum\limits_{j\in\mathcal{G}'} \left(L_{\mathcal{G}'}^j(\delta/2;n_0)-\underset{n \in [n_0, \infty)}{\argmin}\bar{X}_j(n)\right)^+ \right.\right. \\
     && \qquad\qquad\qquad\qquad \left.\left. +\sum\limits_{l\in\mathcal{G}}\underset{n \in [n_0, \infty)}{\argmin}\bar{X}_l(n)+\sum\limits_{i\in \mathcal{K}\setminus\mathcal{G}}N_i(\bar{X}^*_{(m)};n_0)-n_0k\right]\right\},
\end{eqnarray}
where the last inequality holds from Equations   \eqref{eqn: bound_K-G} and \eqref{eqn: bound_argmin}. It is worth noticing that the PGSR$_m$ lower bound is no larger than the PGS$_m$ lower bound presented in Equation \eqref{eqn: bound_PGS2}.  This supports the intuition that ensuring both good screening and ranking is often more challenging than achieving merely good screening.

\subsection{Sample Optimality and Consistency for the PGSR$_{\mathbf{m}}$}
By repeating the arguments in Section \ref{sec: sample-optimality}, we may easily prove the sample optimality and consistency of the EFG-$m$ algorithm in terms of PGSR$_m$. The key distinction in this analysis involves evaluating the newly added component $\sum_{j\in\mathcal{G}'} \left(L_{\mathcal{G}'}^j(\delta/2;n_0)-\argmin_{n \in [n_0, \infty)}\bar{X}_j(n)\right)^+$. Interestingly, as we demonstrate in the following lemma, this component becomes negligible as $k\to\infty$, suggesting that it may not significantly influence the asymptotic behavior of PGSR$_m$. The detailed proof of Lemma \ref{lem: negligible} is included in \ref{ec: lastexit}.

\begin{lemma}
\label{lem: negligible}
    Suppose that Assumptions \ref{assu: guassian} and \ref{assum: bounded_g} hold. Then we have
    \begin{eqnarray*}
    \limsup_{k \rightarrow \infty} \frac{1}{k}\sum\limits_{j\in\mathcal{G}'} \left(L_{\mathcal{G}'}^j(\delta/2;n_0)-\argmin_{n \in [n_0, \infty)}\bar{X}_j(n)\right)^+ = 0.
\end{eqnarray*} 
\end{lemma}
With this lemma, we can obtain the sample optimality and consistency of the EFG-$m$ algorithm in terms of PGSR$_m$, which are summarized in the following theorem. The detailed proof of the theorem is included in \ref{ec: PGSR}.

\begin{theorem}
    \label{thm: nonnormalPGSR}
   Suppose that Assumptions \ref{assu: guassian} and \ref{assum: bounded_g} hold. If the total evaluation budget $B=(n_0 + n_g) k$ and $n_g > m\left[C\left(\frac{\delta}{\bar \sigma} ; n_0\right)-n_0\right]$, the PGSR$_m$ of the EFG-$m$ algorithm satisfies
    \begin{eqnarray*}
    \liminf\limits_{k \to \infty} {\rm PGSR}_m  \geq \Pr\left\{\bar X_{(m)}^* \geq  \mu_{m} - \delta_0\right\} > 0,
    \end{eqnarray*}
    where $\delta_0 \in (0, \delta)$ is a constant such that $m\left[C\left(\frac{\delta - \delta_0}{\bar \sigma}; n_0\right)-n_0\right] = n_g$. Therefore, the EFG-$m$ algorithm is sample optimal in terms of PGSR$_m$. Furthermore, we can show that the EFG-$m$ algorithm is also consistent in terms of PGSR$_m$.
    \end{theorem}

It is worth noticing from the comparison between Theorems  \ref{thm: nonnormalPGS} and \ref{thm: nonnormalPGSR} that the PGS$_m$ and PGSR$_m$ display the same limiting behavior as $k\to\infty$. This equivalence arises because the sampling cost of obtaining the IBR within the selected good subset $\mathcal{G}'$ is negligible as $k$ increases. Correspondingly, in our numerical experiments, the PGSR$_m$  of the EFG-$m$ algorithm may tend to become identical to the PGS$_m$ when $k$ is sufficiently large, demonstrating the \textit{free} ranking effect of the EFG-$m$ algorithm. 




\section{Algorithm Enhancements}
\label{sec: enhancements}
In this section, we focus on enhancing the practical efficiency of the EFG-$m$ algorithm in addressing large-scale virtual screening problems. Specifically, in Section \ref{subsec: top_M}, we improve the algorithm by refining the top-$m$ greedy phase. In Section \ref{subsec: adaptive_explore}, we improve the exploration phase. Then, in Section \ref{subsec: parallel}, we explore how to parallelize the proposed algorithms to utilize LLMs' capability of concurrent responses to improve computational efficiency. 


\subsection{Top-${M}$ Greedy Selection}
\label{subsec: top_M}


The top-$m$ greedy selection rule in the EFG-$m$ algorithm may bring inefficiency, consuming unnecessary evaluation budget in the greedy phase. To illustrate this inefficiency, let's first consider the simple case where $\mathcal{G}=S^*$ and the top-$m$ alternatives can be evaluated without any random noise, such that $\bar{X}_i(n)\equiv\mu_i$ for all $n \in [n_0, \infty)$ and all $i\in\mathcal{S}^*$. {For the general case where all alternatives are evaluated with random noise, the same intuition applies, though additional considerations arise. The more general discussion can be found in Section \ref{subsec: additional_topM}.} From the boundary-crossing perspective introduced in Section \ref{sec: top-m-boundary-crossing}, a correct screening is assured whenever the sample means of all inferior alternatives falls below the boundary $\mu_m$. Ideally, to achieve the highest efficiency, the greedy phase should  focus all the evaluation budget on the inferior alternatives, so that  this condition is met within limited evaluation budget. However, the top-$m$ greedy selection rule does not permit such targeted evaluation. Suppose that, at a certain round, the sample means of $q<m$ inferior alternatives are above $\mu_m$. The top-$m$ greedy selection rule will select and evaluate these $q$ alternatives along with $m-q$ alternatives from the top-$m$ alternatives, as they have the largest $m$ sample means. Clearly, these $m-q$ observations on the top-$m$ alternatives does not aid in  achieving a correct screening and thus leads to inefficiency. 
Moreover, this inefficiency accumulates as the selection rounds proceed.  Although this does not impact the sample optimality or consistency, it may compromise the practical efficiency of the EFG-$m$ algorithm.

To mitigate this inefficiency, we modify the greedy selection phase in the EFG-$m$ algorithm by extending the evaluation scope from the top $m$ to the top-$M$ alternatives at each round, where $M> m$. This strategic adjustment increases the chance of evaluation on the inferior alternatives whose sample means exceed the boundary at each round, thus accelerating the accomplishment of boundary crossings of these inferior alternatives. As a direct result, the necessary number of rounds required before achieving a correct screening is reduced, effectively alleviating the cumulative inefficiency over the greedy selection process. To better illustrate the impact of this adjustment, we also provide a visual comparison of sample paths under the original top $m$ and the revised top $M$ greedy selection rules in the following Figure \ref{fig: top_3_sampling}. From this figure, we can see that for selecting the top-2 alternatives in the example problem, using a top-3 selection rule reduces 4 rounds of greedy and thus saves 4 observations compared to using a top-2 selection rule. The new algorithm is referred to as the EFG-$M$ algorithm and we detail it in the following Algorithm \ref{algorithm: topM}. As we will see numerical experiments in  Section \ref{subsec: numerical_topm_enhancement} show that when $k$ is relatively large, the EFG-$M$ could improve the performance of the EFG-$m$ algorithm by a large margin.  

\begin{figure}[htbp]
    \begin{center}
        \includegraphics[width=0.99\textwidth]{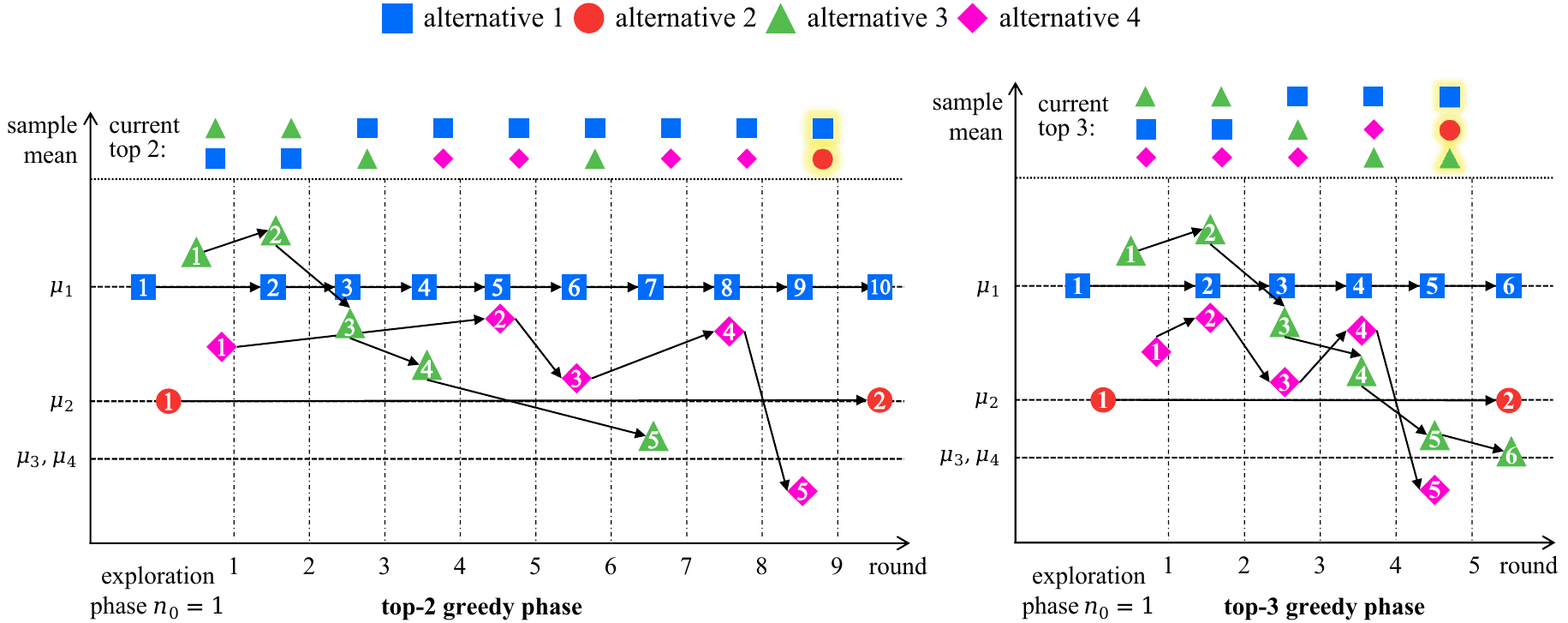}
    \end{center}
    \caption[]{A comparison between top-2 and top-3 greedy selection processes of an example problem with $4$ alternatives where $\bar X_1(n) = \mu_1$ and $\bar X_2(n) = \mu_2, \forall n \geq 1$. Numbers in the markers represent the total sample sizes.}
    \label{fig: top_3_sampling}
\end{figure}

\RestyleAlgo{ruled}
\LinesNumbered
\SetAlgorithmName{Algorithm}{Algorithm}{Algorithm}
\SetAlgoCaptionLayout{centerline}
\begin{algorithm}[hbtp]
\caption{\textbf{Explore-First Top-$M$ Greedy (EFG-$M$) Algorithm}}
    \label{algorithm: topM}
\KwIn{
$m$ and $M$, $k$ alternatives and their prompts, the total evaluation budget $B=ck$, the evaluation budget allocated to the exploration phase $n_0k (<B)$, and the LLM 
}
\For(\tcp*[f]{{Exploration Phase}}){$i=1$ \KwTo $k$ }{
prompt the LLM $n_0$ times for alternative $i$  to get $n_0$ observations $x_{i1},\ldots,x_{in_0}$\;
set $\bar X_i(n_0)=\frac{1}{n_0}\sum_{j=1}^{n_0} x_{ij}$ and let $n_i=n_0$\;
}
\While(\tcp*[f]{{Top-M Greedy Phase}}){$\sum_{i=1}^k n_i < B$ }{
let $\mathcal{S} = \,  \stackrel{1, \dots, M}{ \argmax}_{i \in \{1, \ldots, k\}} \bar X_i(n_i)$\;
  \For{$j \in\mathcal{S}$}{
  prompt the LLM for alternative $j$ to get one observation $x_{j}$;
  
  update $\bar X_{j}(n_j+1) = \frac{1}{n_j+1}\left[n_j\bar X_{j}(n_j) + x_j\right]$ and let $n_j = n_j+1$\;
  }
}
\KwOut{$\stackrel{1, \dots, m}{ \argmax}_{i \in \{1, \ldots, k\}} \bar X_i(n_i)$}
\end{algorithm}

Regarding the EFG-$M$ algorithm,  we would like to make several remarks. First, the choice of $M(> m)$ may significantly impact on the algorithm's performance. Based on extensive numerical experiments (see, e.g., Section \ref{subsec: numerical_topm_enhancement}), we recommend letting $M=2m$ or $M=3m$ to obtain a near-optimal performance.  Second, despite the structural similarity between top-\(M\) and top-\(m\) greedy selection, extending sample optimality to the EFG-\(M\) algorithm is not straightforward. 
Key results such as Observation~\ref{obs: boundary} in the boundary-crossing perspective of EFG-\(m\) may not hold for the top-\(M\) greedy phase. In the phase, even after all  inferior alternatives complete boundary crossing, some of them may still be selected in subsequent rounds as each round will always include at least \(M-m\) inferior alternatives,
making it difficult to ensure a correct screening event.  Resolving this issue is possible by checking the boundary-crossing event upon making the final screening decision; however, this complicates the algorithm, and we prefer to maintain the simplicity of EFG-\(M\). Third, due to the same issue—additional difficulty in securing a correct screening—the EFG-\(M\) algorithm is not always superior to EFG-\(m\). As shown in Section \ref{subsec: numerical_topm_enhancement}, when $k$ is relatively small (e.g., $k \leq 128$), EFG-\(M\) may underperform relative to EFG-\(m\), likely due to the limited number of rounds \(n_g k/M\) available for the top-$M$ greedy phase at this scale.

Additionally, top-\(M\) greedy selection is not the only way to enhance the top-\(m\) greedy phase. The top-\(m\) greedy phase is simple yet flexible, allowing for various modifications. One may explore alternative variants by incorporating ideas from existing budgeted learning algorithms, such as adjusting budget allocation based on sample variances, confidence bounds, or the distances to ``boundary'' alternatives \citep{chen2008efficient, gabillon2012best, bubeck2013multiple}. However, the impact of these variants on sample optimality remains unclear. We focus on the greedy selection mechanism due to its well-established theoretical understanding.





\subsection{Adaptive Exploration: Seeding and The Greedy-as-Remedy Approach}
\label{subsec: adaptive_explore}


We further improve the performance of top-$m$ greedy algorithms by refining the static and equal allocation rule used in the exploration phase. We consider two strategies: \emph{seeding} and the \emph{greedy-as-remedy} approach. The seeding approach \citep{hong2022solving} introduces an additional seeding phase before the exploration phase. In the seeding phase, a small number of observations are collected for each alternative to establish a rough ranking of all alternatives, which then may help allocate more exploration budget to alternatives having better performance. By adopting this strategy we develop the EFG-$M^+$ algorithm. We include a detailed description of this algorithm in \ref{subsec: EFG-$M$Plus} and demonstrate its superior performance in Section \ref{subsec: comparison}. Notice that beyond the improved performance, another advantage of the seeding approach is that it may facilitate the proof of sample optimality. Since the observations collected during seeding are not used in the exploration phase, the alternatives remain statistically independent after exploration. This independence simplifies the theoretical analysis and may allow for an extension of the results in Sections \ref{sec: sample-optimality} and \ref{sec: ranking}; see Section \ref{subsec: EFG-$M$Plus} for further discussion.

The greedy-as-remedy approach, first mentioned in \cite{li2023surprising} as a potential research direction, uses an existing small-scale algorithm in the exploration phase to allocate the budget. While these exploration algorithms alone may not be sample optimal, they can be more efficient than equal allocation. 
We numerically evaluate this approach using the well-known SAR algorithm \citep{bubeck2013multiple} in the exploration phase, as detailed in Section \ref{subsec: SAR-$M$Plus}. The results suggest that the greedy-as-remedy approach may also outperform EFG-$M$ while achieving sample optimality.
However, a major drawback of this approach is the complexity it may introduce in theoretical analysis. SAR’s adaptive budget allocation will create intricate statistical dependencies among the alternatives, 
making it particularly difficult to establish a rigorous analysis of the approach's performance. 
For this reason, we focus on the seeding approach.


\subsection{Parallel Implementation}
\label{subsec: parallel}

Modern LLM APIs and services support concurrent requests and batching, allowing users to query the model with multiple prompts simultaneously \citep{ollama_faq, openai_batch_api}. This built-in concurrency provides a natural opportunity to accelerate virtual screening tasks using parallel LLM calls. In this subsection, we describe how our algorithms can be adapted to exploit this capability and improve computational efficiency in practice. We focus on the EFG-\(M^+\) algorithm, though the same principles apply to the simpler EFG-\(m\) and EFG-\(M\) algorithms. EFG-\(M^+\) consists of three phases: seeding, exploration, and top-\(M\) greedy. The first two phases involve static allocation of evaluation budgets to alternatives, which can be efficiently parallelized by batching the prompt queries. The top-\(M\) greedy phase, however, requires additional consideration. While this phase is still inherently parallelizable (since the top-\(M\) alternatives can be queried simultaneously), it requires synchronization after each round to update the top-\(M\) set. To eliminate this bottleneck, we propose a simple asynchronous modification. At the start of the top-\(M\) greedy phase, several LLM query batches of the initial top-\(M\) alternatives are issued. As soon as responses are returned, the set of top-\(M\) alternatives is updated based on the latest estimates, and the next query for the updated top-\(M\) alternatives is issued immediately. This eliminates idle time and allows the algorithm to take full advantage of the LLM’s concurrent  capability. Note that at any given moment, different queries may target slightly different versions of the top-\(M\) set. We refer to this modified algorithm as EFG-\(M^{++}\). Numerical results in Section~\ref{subsec: additional_parallel} show that EFG-\(M^{++}\) may deliver similar PGS\(_m\) as EFG-\(M^+\) while successfully achieving asynchronization. Further implementation details and experimental results are provided in Section~\ref{subsec: additional_parallel}.

\section{Numerical Experiments}
\label{sec: numerical}

This section explores a case study on LLM-enabled product design screening to demonstrate the practical effectiveness of our algorithms. Before introducing the case study, we conduct a series of numerical experiments, following standard practices in the simulation optimization literature, to comprehensively investigate the performance of our algorithms. Key observations from these experiments are summarized in this section, while detailed numerical results are included in Sections \ref{sec: ec_num_justification} and \ref{sec: ec_num_justification2}.



\subsection{Summary of Numerical Justification for Theoretical Results and Proposed Algorithms}
\label{subsec:numerical_justification}

For ease of justification, we use synthetic problem configurations in which the random evaluation output \( X_i \) for each alternative \( i \) is generated according to a pre-specified distribution.
This synthetic setup allows us to flexibly and quickly assess our algorithms' performance under controlled conditions.
 We consider two types of configurations that differ in the true means of all \( X_i \).
 While we rely on the normality assumption in Assumption \ref{assu: normal} to establish the desired theoretical properties, we also evaluate the proposed algorithms' performance under non-normal observation distributions to assess robustness. 
 Furthermore, to assess and compare the algorithms’ performance across different problem sizes, we vary the number of alternatives and use a total sampling budget \( B = c k \), where \( c\) is a constant that may vary depending on the specific experiment.

\subsubsection{Demonstration of Sample Optimality and Consistency.}
We examine the sample optimality and consistency of the EFG-\(m\) algorithm under the \textcolor{black}{slippage configuration}, as it is often referred to as the worst-case mean configuration. Across the three distributional settings of observations \textcolor{black}{(i.e., normal, Parato and lognormal distributions)}, the PCS\(_m\) achieved by EFG-\(m\) stabilizes around a constant positive level (e.g., 60\%) as  \( k \) increases, demonstrating its sample optimality and robustness to non-normal observations. To assess the consistency, we fix a large $k$ and vary the sampling budget \( B = c k \). As \( c \) increases, the PCS\(_m\) improves monotonically and may converge to 1. Additionally, we investigate the sample complexity of EFG-$m$ concerning various problem parameters, and the results are consistent with the implications of Theorem \ref{thm: consistency}. Details can be found in Section~\ref{subsec: numerical_topm_sample_optimality}.

\subsubsection{Demonstration of Free Ranking Effect.}
We examine the free ranking effect under the \textcolor{black}{random-mean (RM)} configuration, as it offers a general and versatile testing scenario capable of representing a broad spectrum of possible mean configurations encountered in practice. Across all three distributional settings, the EFG-\(m\) algorithm achieves sample optimality for both PGS\(_m\) and PGSR\(_m\). Interestingly, when \( k \) is small, the PGSR\(_m\) of EFG-\(m\) may initially be lower than its  PGS\(_m\); however, as \( k \) increases, the two metrics converge. This result highlights the EFG-\(m\) algorithm’s ``free'' ranking ability in large-scale screening problems. Details can be found in Section~\ref{subsec: numerical_topm_ranking}.

\subsubsection{Performance Improvement by top-$M$ Greedy Selection.}
We use the RM configurations to demonstrate the performance enhancement of the EFG-\(M\) algorithm with setting \( M = 2m \). The EFG-\(M\) algorithm shows a marked improvement over EFG-\(m\) in both PGS$_m$ and PGSR$_m$ when \( k \) is relatively large, particularly for \( k \geq 2^{10} \), while still maintaining the free ranking effect. For instance, under the RM-Normal configuration with \( k = 2^{11} \), the PGS\(_m\) and PGSR\(_m\) of EFG-\(m\) increase from approximately 50\% to around 80\%. Additionally, we find that the performance of EFG-\(M\) is relatively insensitive to the specific choice of the ratio \( M/m \). Details can be found in Section~\ref{subsec: numerical_topm_enhancement} and Section~\ref{subsec: additional_select_M}.

\subsubsection{Comparison with Existing Algorithms.}
To provide a comparative evaluation, we include two well-known OSS algorithms -- OCBAm \citep{chen2008efficient} and SAR \citep{bubeck2013multiple} -- in our experiments and benchmark our EFG-\(M^+\) algorithm against them. Implementation details and a discussion of additional algorithms are provided in Section~\ref{subsec: additional_algorithms}. These algorithms are heuristics whose performance in large-scale settings has not been evaluated. In our experiments, OCBAm fails to achieve sample optimality in any of the tested cases. The SAR algorithm performs well under the normal configuration, but its performance deteriorates significantly under non-normal distributions, becoming non-sample-optimal. This highlights the need for caution when applying heuristics to large-scale screening problems. To further demonstrate the strength of our approach, we evaluate the algorithms on a practical simulation-based virtual screening problem.  EFG-\(M^+\) consistently outperforms both OCBAm and SAR across all problem instances for both PGS\(_m\) and PGSR\(_m\), which underscores its effectiveness. Details can be found in Sections~\ref{subsec: comparison} and \ref{subsec: numerical_practical}.

Given these justifications, we will focus on the EFG-\(M^+\) algorithm in the subsequent case study. Rather than further validating its properties or comparing it against heuristic algorithms, our objective is to investigate the feasibility of using LLMs as human evaluators for virtual screening and to assess whether the algorithm can enable cost-effective and scalable virtual screening.

\subsection{{Case Study: LLM-Based Product Design Screening}}
\label{subsec: numerical_LLM}
In this subsection, we conduct an illustrative case study on a product design screening problem, where the objective is to explore different product designs and identify the top configurations that maximize consumer willingness to pay (WTP). These top designs can then be further validated through market tests or real-world consumer studies. To illustrate this, we formulate a laptop design problem. A laptop consists of several key attributes that influence consumer preferences and willingness to pay, such as the CPU, RAM size, and storage drive. Each unique combination of these attributes constitutes a distinct design. The total number of alternative designs (shortly, alternatives) depends on the available attribute set and the choice set for each attribute. Given the large number of possible configurations, an efficient screening approach is required to identify the most promising designs for further evaluation.

Each product design alternative can, in principle, be evaluated through marketing research surveys that collect direct feedback from consumers. However, conducting large-scale surveys is costly and time-consuming. A promising alternative, enabled by recent advancements in LLM intelligence, is to use LLMs as human evaluators. This approach treats an LLM as a proxy for a customer population and conducts virtual surveys by querying the model with the same survey questions that would be posed to real consumers.  In this case study, we adopt this LLM-based evaluation approach to estimate WTP for different laptop designs. Notice that we select WTP as the evaluation metric of each alternative because it provides a natural and interpretable way to assess the reasonableness of LLM-generated responses. A reasonable LLM response should follow a fundamental principle: a laptop design with superior specifications (e.g., a better CPU or higher storage capacity) should correspond to a higher WTP than an inferior design.  In the following, we will introduce the experimental setup, illustrate the reasonableness of LLM-generated responses for the laptop design problem, and then solve the virtual screening problem using our proposed approach. Besides, we will also provide a cost analysis for LLM-based virtual screening.

\subsubsection{Experimental Setup and the Language Model.}
To conduct the experiments in this section, we use a computing server equipped with two Intel 64 Family 6 Model 85 CPUs with a total of 96 logical processors and 256 GB of RAM, along with two NVIDIA GeForce RTX 3090 GPUs. We run the LLM locally rather than using commercial APIs (e.g., OpenAI’s GPT-4), in order to reduce experimental costs and ensure flexibility. We use Ollama, a framework that provides an efficient environment for deploying and interacting with LLMs, and we select Llama 3.2, an open-source language model with 3 billion parameters. While this model is much smaller than commercial alternatives such as ChatGPT or DeepSeek, it is significantly less computationally demanding and can be efficiently executed on our server. The setup instructions for configuring the Llama 3.2 model as a customer simulator are provided in Section \ref{subsubsec: setup}. We interact with the LLM through survey-like prompts. In each prompt, a laptop design is presented to the LLM as part of a simulated survey, where the LLM is asked to provide a single WTP estimate for the given configuration. The prompts are included in Section \ref{subsubsec: prompt_WTP}.
From a simulation perspective, each LLM output can be interpreted as an observation of the mean WTP for that particular design. Notably, LLM outputs exhibit inherent randomness, which is controlled by the temperature parameter. In this study, we set the temperature to its default value of 1, allowing for a moderate level of variation in responses. This randomness is desirable in the context of virtual screening, as it naturally reflects the diversity of consumer preferences.

\subsubsection{The Reasonableness of Using LLM as Human Evaluators.}
To assess the feasibility of using LLMs as human evaluators for virtual screening, we conduct a small-scale experiment to evaluate whether LLM-generated responses exhibit reasonable consistency in estimating WTP. We generate a set of seven Lenovo laptop configurations by varying three key attributes: CPU, RAM size, and storage drive size. The configurations are listed in Table \ref{tab:laptop_designs}. For each configuration, we query the LLM 4,000 times to obtain 4,000 observations of the WTP. These observations are averaged to obtain the mean WTP estimate for each configuration. The rounded WTP estimates are presented in the last row of Table \ref{tab:laptop_designs}.

\begin{table}[htbp]
    \centering
    \caption{Laptop Design Configurations and LLM-Generated WTP Estimates}
    \label{tab:laptop_designs}
    \begin{tabular}{cccccccc}
        \hline
        
        \hline
        
        \hline
        {Attribute} & {Design 1} & {Design 2} & {Design 3} & {Design 4} & {Design 5} & {Design 6} & {Design 7} \\
        \hline

        \hline

        \hline
        CPU (Intel Core Series) & i5  & i7  & i9  & i9  & i9  & i9  & i9  \\
        RAM Size                    & 8 GB  & 8 GB  & 8 GB  & 16 GB & 32 GB & 64 GB & 64 GB \\
        Storage Drive Size          & 256 GB & 256 GB & 256 GB & 256 GB & 256 GB & 256 GB & 512 GB \\
        \hline

        \hline

        \hline
        Mean WTP Estimate (\$)           &  790 &  902 & 1070 & 1243 & 1456 &
       1699 & 1800 \\
            \hline

        \hline

        \hline
    \end{tabular}
\end{table}

Table \ref{tab:laptop_designs} demonstrates that the LLM-generated WTP estimates exhibit both price reasonableness and logical ranking behavior. The estimated WTP values align well with realistic market prices for Lenovo laptops, indicating that the LLM’s pricing intuition is consistent with consumer expectations. Furthermore, the WTP estimates follow a logical ordering, increasing as laptop specifications improve, confirming that the LLM correctly assigns higher value to superior configurations. Overall, these results support the feasibility of using LLMs as human evaluators for product design screening. Readers familiar with the laptop market may find the WTP increase for larger RAM sizes higher than expected. More accurate WTP estimates could potentially be obtained using a larger language model with more parameters. 

\subsubsection{LLM-Based Large-Scale Virtual Screening.}
\label{subsubsec: llm_screening}
Now, we evaluate the effectiveness of our EFG-$M^+$ algorithm in solving the product design screening problem. We formulate several problem instances, varying the total number of alternatives $k$ from 36 to 3,240 by adjusting six key attributes. 
The attribute sets and values for the problem instances are summarized in Section \ref{subsubsec: attributes}.   For each alternative in a given problem instance, we query the LLM 4,000 times to obtain observations of WTP and estimate the mean WTP. In the screening problem, we let \( m = 10 \) and \( \delta = 30 \) and use a total evaluation budget of \( B = 400k \). To estimate the PGS\(_m\) and PGSR\(_m\) of the EFG-$M^+$ algorithm, we conduct 500 macro independent replications for each $k$. Notice that instead of re-querying the LLM for each replication, we use the 4,000 observations as an empirical distribution and generate samples from it. This approach significantly reduces experimental costs (see the cost analysis below for a breakdown of a single macro replication).  Furthermore, in this experiment, we compare our EFG-$M^+$ algorithm with a direct querying approach (DIRECT), which queries LLM directly the top-\( m \) designs from listed alternatives. This comparison aims to verify whether the LLM-as-human-evaluator approach is necessary for LLM-based virtual screening. We explore several prompting strategies for DIRECT and select the most effective one (see Section \ref{subsubsec: direct prompt} for details). The PGS\(_m\) and PGSR\(_m\) of the EFG-$M^+$ algorithm and DIRECT across problem instances are summarized in Table \ref{table: LLM_PGS}.  

            \begin{table}[htbp]
                \centering
                \begin{tabular}{ccccccccc}
                    \\ \hline 
                    
                    \hline

                    \hline
                    Number of alternatives $k$      & \multicolumn{2}{c}{36}       & \multicolumn{2}{c}{360}       & \multicolumn{2}{c}{1,080}      & \multicolumn{2}{c}{3,240}      \\ \hline 
                    
                    \hline

                    \hline
                    Performance Measure  & PGS$_m$        & PGSR$_m$       & PGS$_m$        & PGSR$_m$       & PGS$_m$        & PGSR$_m$       & PGS$_m$        & PGSR$_m$          \\ \hline 
                    
                    \hline

                    \hline
                    EFG-$M^+$     & 0.976          & 0.916          & 1.000          & 0.950          & 1.000          &   1.000        &   1.000        & 1.000           \\
                    DIRECT     & 0.000          & 0.000          & 0.000          & 0.000          & 0.000          & 0.000          & 0.000          & 0.000            \\
                    \hline 

                    \hline

                    \hline
                    \end{tabular}
                    \caption{Estimated PGS$_m$ and PGSR$_m$ on LLM-based product design screening problem instances}
                    \label{table: LLM_PGS}
                    \end{table}

From Table \ref{table: LLM_PGS}, we observe that the EFG-\(M^+\) algorithm performs well in solving the LLM-based product design screening problem. Given a total evaluation budget of \( B = 400k \) (which, as discussed below, may not pose a computational or financial burden even for  large-scale problems), the PGS\(_m\) remains consistently strong as the number of alternatives increases. The free ranking effect also holds, as evidenced by the coincidence of PGS\(_m\) and PGSR\(_m\) for \( k \geq 1080 \).  In contrast, direct LLM queries alone may be ineffective for screening even in small-scale problems. The DIRECT approach fails entirely: both the PGS\(_m\) and PGSR\(_m\) remain near zero across all problem instances. These results suggest that while LLMs are powerful, they may not yet be reliable for directly selecting and ranking top-performing designs, let along achieving the sample optimality. Instead, treating LLMs as human evaluators and integrating them into a sample-optimal OSS algorithm may be essential for extracting meaningful screening results for large-scale problems.  

One might question why the DIRECT approach, despite the perceived ``intelligence'' of LLMs, may fail even in small-scale problems. We believe a key reason is the increased difficulty of selecting top-\( m \) alternatives (rather than top-1), a challenge that intensifies as \( m \) increases. Even humans may struggle with such multi-choice, multi-selection tasks. To further investigate this, we evaluate the performance of the DIRECT approach when \( m = 1 \) in Section \ref{subsubsec: direct_top1}. The results indicate that in this simpler scenario, when $k$ is small, the LLM is able to identify the top-1 design with nonzero probability. This suggests that LLMs may indeed possess a form of ``intelligence'' for small-scale, straightforward screening tasks. However, this intelligence may deteriorate as the screening problem grows larger and more complex, rendering them incapable of handling large-scale virtual screening problems directly. Beyond this degradation in accuracy, we also observe that the performance of the DIRECT approach is highly unstable, varying significantly across different problem instances and prompts. This further reinforces the necessity of a sample-optimal OSS framework in LLM-based virtual screening which may ensure both stability and superior performance. Interested readers may refer to Section \ref{subsubsec: direct_top1} for further details.


\subsubsection{Cost Analysis for LLM-Based Virtual Screening.}
A key advantage of LLM-based virtual screening is its cost-effectiveness, particularly compared to traditional human-driven marketing research. For a problem instance with 3,240 alternatives and a total evaluation budget of \( B = 400k \), we need to query the LLM 1,926,000 times. Running these queries on our computing server, equipped with two NVIDIA GeForce RTX 3090 GPUs (released in 2020 at \$1,500 each), we achieve a query rate of 40–50 queries per second. This allows us to complete the screening in only about 10 hours. While runtime may vary depending on prompt length, computing environment, and the deployed model, this execution time remains remarkably short for large-scale screening. If local computing resources are unavailable, commercial LLM APIs such as OpenAI’s GPT-4o mini or DeepSeek's DeepSeek V3 offer alternatives. These services charge based on tokens, where one token roughly corresponds to four characters or 0.75 words in English text. Each prompt in our experiment contains approximately 80 tokens (see Section \ref{subsubsec: prompt_WTP}), meaning the full screening across 3,240 alternatives requires processing 104 million tokens. According to OpenAI or DeepSeek's pricing, GPT-4o mini (DeepSeek V3) charges \$0.6 (\$0.27) per million tokens,  leading to a total experiment cost of \$62.4 (\$28.1). Apparently, compared to the high costs and time requirements of real consumer surveys, using LLMs as human evaluators, combined with our algorithms, may offer a scalable and cost-efficient alternative for early-stage product design screening. We are enthusiastic about the potential of this approach for other large-scale screening problems and leave its validation for future exploration.

\section{{Concluding Remarks}}
\label{sec: conclusion}
In this paper, we thoroughly investigate the budget allocation challenges in LLM-enabled virtual screening. We begin by defining the sample optimality in terms of the PGS$_m$, which denotes the probability of correctly selecting a good subset. Based on this definition, we propose the EFG-$m$ algorithm to achieve this sample optimality. From a boundary-crossing perspective, we derive a PGS$_m$ lower bound for the EFG-$m$ algorithm, based on which we then prove its sample optimality and consistency. Interestingly, the algorithm may also achieve the optimal sample complexity regarding various problem parameters. Furthermore, we consider rankings to provide decision makers with a deeper understanding into alternatives included in the selected subset. We for the first time introduce the concept of good ranking and define the PGSR$_m$. Surprisingly, we find that the EFG-$m$ algorithm has the additional benefit of achieving a good ranking for free in solving large-scale virtual screening problems. Again based on the boundary-crossing perspective, we prove the sample optimality and consistency of the EFG-$m$ algorithm for the PGSR$_m$. Moreover, we extend the EFG-$m$ algorithm to the EFG-$M$ and EFG-$M^+$ algorithms to enhance practical applicability. Finally, a comprehensive numerical study validates our theoretical results and demonstrates the efficiency of our algorithms in solving virtual screening problems. Additionally, a case study on LLM-based virtual screening demonstrates the reasonableness, efficacy, and cost-effectiveness of the LLM-as-human-evaluator approach when integrated with our sample-optimal algorithms.

{To conclude this paper, we point out several interesting directions for future research. First, numerical results in Section \ref{sec: numerical} demonstrate the robustness of our algorithms in solving virtual screening problems with non-normal observations. This sets a theoretical challenge to rigorously establish the sample optimality of the algorithms without relying on the normality assumption. Second, in our current setup, we adopt the default temperature setting of 1 when using LLMs as human evaluators. However, temperature plays a key role in controlling the variance of LLM outputs—lowering the temperature reduces randomness but may introduce bias. Exploring how to optimally tune the temperature to balance the bias–variance tradeoff under a fixed evaluation budget may offer meaningful improvements in screening efficiency. Third, we assume that LLM evaluations for a given alternative are i.i.d., which requires querying the LLM with the same prompt multiple times. However, many LLMs support conversational interactions, allowing the user to generate follow-up outputs without repeating the full prompt—potentially reducing cost. This approach, however, may introduce dependence among the generated responses. Understanding how such dependence affects screening accuracy and how to exploit its impact remains an intriguing and challenging problem.}

\bibliographystyle{informs2014}
\bibliography{ref.bib}

\ECSwitch
\EquationsNumberedBySection 
%
%
\ECHead{{E-Companion to \\ \vspace{0.2cm}
\textit{Efficient Budget Allocation for Large-Scale LLM-Enabled Virtual Screening}}}

\vspace{0.2cm}
\section{Technical Supplements for Section \ref{sec: problem}}
\label{subsec: proof_best_order}
In this section, we will prove that for a budget algorithm, as long as $\limsup_{k \rightarrow \infty} |\mathcal{G}| < \infty$, the total evaluation budget \(B\) required to keep a positive PGS$_m$ should  increase at least linearly in the total number of alternatives $k$. This result serves as the theoretical basis for Definition \ref{def: rate_optimality_PCS}. Inspired by \cite{itemZhong2022}, we introduce the following assumptions on the budget allocation behaviors of budget allocation algorithms. 
\begin{assumption}
\label{assu: random1}
Let \(\mathcal{A}_t\) denote the set of alternatives having no observation when the algorithm allocates the \(t\)-th observation. If the algorithm allocates the $t$-th observation to an alternative in \(\mathcal{A}_t\), then all alternatives in \(\mathcal{A}_t\) have an equal probability of being selected.
\end{assumption}
\begin{assumption}
\label{assu: random2}
 Let \(\mathcal{A}\) denote the set of alternatives having no observation when the total evaluation budget is exhausted.  If \(|\mathcal{A}|\geq m\)  and the algorithm chooses to select \(m\) alternatives from \(\mathcal{A}\), then all subsets of $\mathcal{A}$ containing $m$ alternatives have an equal probability of being selected.
\end{assumption}

Assumption \ref{assu: random1} requires that allocating observations to alternatives with observations do not provide any information on the alternatives without any observations. Assumption \ref{assu: random2} requires that after the total evaluation budget is exhausted, if the algorithm chooses to select $m$ alternatives from the set of alternatives without any observations, it can only randomly pick $m$ alternatives. Borrowing ideas from the proof of Theorem 1 in \cite{itemZhong2022}, we may prove the following theorem.
\begin{theorem}
Suppose that $\limsup_{k \rightarrow \infty} |\mathcal{G}| < \infty$ and Assumptions \ref{assu: random1} and \ref{assu: random2} hold. Then, for any budgete allocation algorithm, if $\lim_{k \rightarrow \infty}B/k = 0$,
\[
\limsup_{k \rightarrow \infty} {\rm PGS}_m = 0.
\]
\end{theorem}
\begin{proof}{Proof.}
We will prove the theorem by contradiction. Assume that 
\begin{eqnarray}
    \label{eq: assum}
    \limsup_{k \rightarrow \infty} \mbox{PGS}_m = \alpha
\end{eqnarray} for some constant $\alpha \in (0, 1]$. Let $g(k)=|\mathcal{G}|$ and $\bar g = \sup_{k \geq 1} g(k)$. Notice that since $\limsup_{k \rightarrow \infty} |\mathcal{G}| < \infty$, $\bar g < \infty$. Then, let \(c_{\alpha} = 1-\sqrt[\bar g]{1-{\alpha}/{4}}\). Furthermore, since $\lim_{k \rightarrow \infty}B/k = 0$, there must exist a finite $K_0$ such that as long as $k \geq K_0$, $B \leq c_{\alpha}k$. In the following, we consider \(k \geq \max\left\{{\bar g}/(1 - c_{\alpha}), K_0\right\}\).
For a budget allocation algorithm, the PGS$_m$, i.e., the probability that the selected subset $\mathcal{S}$ includes $m$ good alternatives, can be decomposed as
\begin{eqnarray}
    \label{eq: OSS_bound}
    \notag \mbox{PGS}_m  = && \Pr \left\{ \{\mathcal{S} \subseteq \mathcal{G} \} \cap \{ \mathcal{G}\text{ has at least one obs.}\}\right\}  + \Pr \left\{ \{\mathcal{S} \subseteq \mathcal{G}\} \cap \{ \mathcal{G}\text{ has no obs.}\}\right\} \\
     && \leq \Pr \left\{ \mathcal{G}\text{ has at least one obs.}\right\}  + \Pr \left\{ \mathcal{G}\text{ has no obs.}\right\}, 
\end{eqnarray}
where \(\{\mathcal{G}\text{ has at least one obs.}\}\) means that there is at least one good alternative having observations, while \(\{\mathcal{G}\text{ has no  obs.}\}\) means that none of the good alternatives has any observations. We then analyze the two probabilities on the right-hand side (RHS) of Equation \eqref{eq: OSS_bound} separately.

Let's denote the event that the algorithm allocates the \(t\)-th observation to an alternative not belonging to  \(\mathcal{G}\) as  \(\Omega_t\).  Then, for the first term on the RHS of Equation \eqref{eq: OSS_bound}, we have
\begin{eqnarray}
    \label{eq: first}
      \Pr\left\{\mathcal{G}\text{ has at least one obs.}\right\} & \leq  & 1-\Pr\left\{\bigcap_{t=1}^{B} \Omega_t\right\} = 1-\Pr\left\{\Omega_1\right\} \times \prod_{t=2}^{B} \Pr\left\{\Omega_t \mid \Omega_1, \ldots, \Omega_{t-1}\right\}. \quad \quad \quad  \quad
\end{eqnarray} 
Under Assumption \ref{assu: random1}, we have
\begin{eqnarray}
    \label{eq: first_p}
    \Pr\left\{\Omega_1\right\}  \geq \frac{k  - g(k)}{k} \text {  and  } 
    \Pr\left\{\Omega_t \mid \Omega_1, \ldots, \Omega_{t-1}\right\} \geq \frac{k - t + 1 - g(k)}{k - t + 1} \quad \forall t \geq 2,
\end{eqnarray}
because  there are at least \(k - (t - 1)\) alternatives having no any observation when the algorithm allocates the \(t\)-th observation.  Combining Equations  \eqref{eq: first} and \eqref{eq: first_p}, we can further derive
\begin{eqnarray}
\label{eq: PGS_bound_1}
      \Pr\left\{\mathcal{G}\text{ has at least one obs.}\right\} 
    \notag & \leq & 1-\frac{k-g(k)}{k} \cdot \frac{k-1-g(k)}{k-1} \cdot \cdots \cdot \frac{k-B+1-g(k)}{k-B+1} \\
   \notag & \leq & 1-\frac{k-\bar g}{k} \cdot \frac{k-1-\bar g}{k-1} \cdot \cdots \cdot \frac{k-c_{\alpha}k+1-\bar g}{k-c_{\alpha}k+1}  \\
   \notag & = & 1-  \prod_{i=1}^{\bar g} \frac{k-c_{\alpha}k-(i-1)}{k-(i-1)} \\
   & \leq & 1-   \left(\frac{k-c_{\alpha}k -(\bar g-1) }{k-(\bar g-1)}\right)^{\bar g} 
     =:  f_1(k; c_{\alpha}, \bar g), \qquad \qquad \qquad
\end{eqnarray} 
where the second inequality holds because $g(k) \leq \bar g$ and $B\leq c_{\alpha} k$, and the last inequality holds because $ k \geq {\bar g}/(1 - c_{\alpha})$. 

For the second term on the RHS of  Equation \eqref{eq: OSS_bound}, notice that since $k \geq \bar g/(1-c_{\alpha})$, there are at least \(k - c_{\alpha}k \geq \bar g \geq m\) alternatives having no observation. Then, according to Assumption \ref{assu: random2},
\begin{eqnarray}
\label{eq: PGS_bound_2}
    \notag \Pr  \left\{ \{\mathcal{G}\text{ has no obs.}\} \right\} & = & \Pr  \left\{ \{\text{at least } k-c_{\alpha}k \text{ alternatives has no obs.}\}  \cap \{\mathcal{G}\text{ has no obs.}\} \right\}  \\
     & \leq & \frac{1}{\tbinom{k-c_{\alpha}k}{g(k)}} \leq  \frac{g(k)}{k-c_{\alpha}k}  \leq   \frac{\bar g}{k-c_{\alpha}k} =: f_2(k; c_{\alpha}, \bar g).
\end{eqnarray}

Combining Equations \eqref{eq: PGS_bound_1} and \eqref{eq: PGS_bound_2} and the fact that \(\lim_{k \to \infty} f_1(k; c_{\alpha}, \bar g) = 1 - (1 - c_{\alpha})^{\bar g} = \alpha/4\) and \(\lim_{k \to \infty} f_2(k; c_{\alpha}, \bar g) = 0\), there must exist a finite \(K_1\) such that when \(k \geq K_1\),
$$ \mbox{PGS}_m  \leq  f_1(k; c_{\alpha}, \bar g) + f_2(k; c_{\alpha}, \bar g) \leq  \frac{\alpha}{2},$$
which contradicts with Equation \eqref{eq: assum}. The conclusion of interest is now proved.   \hfill \Halmos
\end{proof}

\section{Technical Supplements for Section \ref{sec: sample-optimality}}

\subsection{Proof of Theorem \ref{thm: nonnormalPGS}}
\label{subsec: proof_nonnormal}
The proof presented here follows the proof of Theorem 3 in \cite{li2023surprising}. To maintain self-containment and clarity, we provide below a detailed argument.
\begin{proof}{\textbf{Proof of Theorem \ref{thm: nonnormalPGS}}:}
    Given a total evaluation budget $B=(n_0+n_g) k$, the PGS$_m$ lower bound of the EFG-$m$ algorithm stated in Equation \eqref{eqn: bound_PGS3} can be rewritten as
    \begin{eqnarray}
        \label{lem: proof_PACS_formula}
        \notag \mbox{PGS}_m &\geq & \Pr \left\{ \frac{n_g}{m} + n_0 \geq  \frac{1}{k} \sum_{j\in\mathcal{G}}\argmin_{n \in [n_0, \infty)}  \bar{X}_j(n)
        + \frac{1}{k} \sum_{i\in\mathcal{K}\setminus\mathcal{G}}  N_i \left(\bar X^*_{(m)}; n_0\right)\right\}.
        \end{eqnarray}
  Since $n_g > m\left[C(\delta/\bar{\gamma}; n_0)-n_0\right]$, we can arbitrarily fix a constant $\lambda$ such that $0<\lambda<n_g/m+n_0-C(\delta/\bar\gamma;n_0)$. Then, we may select a positive $\delta_\lambda\in (0,\delta)$ such that 
   \begin{eqnarray}\label{eqn: proof1_delta_lambda}
       C\left(\frac{\delta-\delta_\lambda}{\bar\sigma};n_0\right) = n_g/m+n_0-\lambda >0.
   \end{eqnarray}
   The existence of such $\delta_{\gamma}$ can be ensured due to that by the virtue of definition, $C(x; n_0)$ is a continuous and monotonically decreasing function on $x \in (0, \infty)$, and $C(x;n_0)\to 0$ as $x\to\infty$ for any fixed $n_0$. For the given $\gamma$, we may further derive the PGS lower bound in Equation \eqref{lem: proof_PACS_formula} as
       \begin{eqnarray}
        \label{lem: proof_PACS_formula2}
        \notag  \mbox{PGS}_m  \geq & & \Pr \Bigg\{\Bigg\{\frac{n_g}{m} + n_0 \geq  \frac{1}{k} \sum_{j\in\mathcal{G}}\argmin_{n \in [n_0, \infty)}  \bar{X}_j(n)
        + \frac{1}{k} \sum_{i\in\mathcal{K}\setminus\mathcal{G}}  N_i \left(X^*_{(m)}; n_0\right)\Bigg\} \qquad \qquad \qquad \qquad \\ 
        & & \qquad \qquad \qquad \qquad \qquad \qquad \qquad \qquad  \qquad \qquad \qquad \cap \left\{\bar X^*_{(m)}\geq \mu_m-\delta_{\gamma}\right\}\Biggr\}.
        \end{eqnarray}

Now we proceed to study the limiting behavior of the PGS$_m$ bound presented in Equation 
\eqref{lem: proof_PACS_formula2}. Firstly,
according to Argument \ref{argu4}, we have  $\sum_{j\in\mathcal{G}}\argmin_{n \in [n_0, \infty)}  \bar{X}_j(n)<\infty$ almost surely. This indicates
\begin{eqnarray}\label{eqn: proof1_argmin}
\Pr\left\{\frac{1}{k}\sum_{j\in\mathcal{G}}\argmin_{n \in [n_0, \infty)}\bar{X}_j(n)\leq \frac{\lambda}{2}\right\} \nearrow\Pr\left\{\sum_{j\in\mathcal{G}}\argmin_{n \in [n_0, \infty)}\bar{X}_j(n)<\infty\right\}=1, \mbox{ as } k\to\infty.
\end{eqnarray}
Additionally, Argument \ref{argu5} states that, under the condition $\bar X^*_{(m)}\geq \mu_m-\delta_{\gamma}$,
\[
\liminf_{k\to\infty}\ \frac{1}{k}\sum_{i\in\mathcal{K}\setminus\mathcal{G}}  N_i \left(X^*_{(m)};n_0\right) \leq 
\limsup_{k\to\infty}\ \frac{1}{k}\sum_{i\in\mathcal{K}\setminus\mathcal{G}}  N_i \left(X^*_{(m)};n_0\right)\leq C\left(\frac{\delta-\delta_{\gamma}}{\bar\sigma};n_0\right), \mbox{ almost surely}.
\]
Therefore, it is clear that, for the given $\lambda$, 
\begin{eqnarray}\label{eqn: proof1_N}
   \liminf_{k\to\infty}\ \Pr\left\{ \frac{1}{k}\sum_{i\in\mathcal{K}\setminus\mathcal{G}}  N_i \left(X^*_{(m)};n_0\right)\leq C\left(\frac{\delta-\delta_{\gamma}}{\bar\sigma};n_0\right)+\frac{\lambda}{2}\right\}= 1.
\end{eqnarray}
For the ease of presentation, we define two events
\begin{eqnarray*}
    \Omega_{\argmin} = \left\{\frac{1}{k}\sum_{j\in\mathcal{G}}\argmin_{n \in [n_0, \infty)}\bar{X}_j(n)\leq \frac{\lambda}{2}\right\}\mbox{ and }\Omega_{N}=\left\{ \frac{1}{k}\sum_{i\in\mathcal{K}\setminus\mathcal{G}}  N_i \left(X^*_{(m)};n_0\right)\leq C\left(\frac{\delta-\delta_{\gamma}}{\bar\sigma};n_0\right)+\frac{\lambda}{2}\right\}.
\end{eqnarray*}
By Equations \eqref{eqn: proof1_argmin} and \eqref{eqn: proof1_N}, they satisfy
\begin{eqnarray}\label{eqn: proof1_limit}
    \lim_{k\to\infty}\Pr\{\Omega_{\argmin}\}= 1 \mbox{ and } \liminf_{k\to\infty}\ \Pr\{\Omega_{N}|\bar X^*_{(m)}\geq \mu_m-\delta_{\gamma}\}= 1.
\end{eqnarray}

Now we are ready to study the limiting behavior of PGS$_m$. Notice that Equation \eqref{lem: proof_PACS_formula2} can be further reformulated as
\begin{eqnarray}\label{eqn: proof_PGS_1}
        \notag \mbox{PGS}_m  &\geq& \Pr \Bigg\{\Bigg\{\frac{n_g}{m} + n_0 \geq  \frac{1}{k} \sum_{j\in\mathcal{G}}\argmin_{n \in [n_0, \infty)}  \bar{X}_j(n)
        + \frac{1}{k} \sum_{i\in\mathcal{K}\setminus\mathcal{G}}  N_i \left(\bar{X}^*_{(m)}; n_0\right)\Bigg\}\cap\{\bar X^*_{(m)}\geq \mu_m-\delta_{\gamma}\}\Bigg\}  \qquad\\
      \notag  &\geq & \Pr \Bigg\{\Bigg\{\frac{n_g}{m} + n_0 \geq  \frac{1}{k} \sum_{j\in\mathcal{G}}\argmin_{n \in [n_0, \infty)}  \bar{X}_j(n)
        + \frac{1}{k} \sum_{i\in\mathcal{K}\setminus\mathcal{G}}  N_i \left(\bar{X}^*_{(m)}; n_0\right)\Bigg\}  \\
        \notag && \qquad \qquad \qquad \qquad  \qquad \qquad \qquad \quad \cap \,\Omega_{\argmin}\cap\Omega_{N}  \cap \left\{\bar X^*_{(m)}\geq \mu_m-\delta_{\gamma}\right\}\Bigg\}\\
        \notag &\geq& \Pr \left\{\left\{\frac{n_g}{m} + n_0\geq  \frac{\lambda}{2}
        + C\left(\frac{\delta-\delta_{\gamma}}{\bar\sigma};n_0\right)+\frac{\lambda}{2}\right\}\cap \Omega_{\argmin}\cap\Omega_{N}\cap\{\bar X^*_{(m)}\geq \mu_m-\delta_{\gamma}\}\right\}\\
        &=& \Pr \left\{\Omega_{\argmin}\cap\Omega_{N}\cap\{\bar X^*_{(m)}\geq \mu_m-\delta_{\gamma}\}\right\},
        \end{eqnarray}
where the second inequality holds because the additional condition $\Omega_{\argmin}\cap\Omega_{N}$ is imposed, the third inequality holds by the definition of $\Omega_{\argmin}$ and $\Omega_{N}$, and the last equality holds from the choice of $\delta_\lambda$ in Equation \eqref{eqn: proof1_delta_lambda}. 
By using the Bonferroni's inequality in Equation \eqref{eqn: proof_PGS_1}, we may further have
\begin{eqnarray*}
    \mbox{PGS}_m 
    &\geq& \Pr \left\{\Omega_{N}\cap\{\bar X^*_{(m)}
    \geq \mu_m-\delta_{\gamma}\}\right\}-\Pr\left\{\Omega_{\argmin}^c\right\} \\
    &=& \Pr \left\{\Omega_{N}|\bar X^*_{(m)}\geq \mu_m-\delta_{\gamma}\right\}\Pr\{\bar X^*_{(m)}\geq \mu_m-\delta_{\gamma}\}-\Pr\left\{\Omega_{\argmin}^c\right\}.
\end{eqnarray*}
Then, according to Equation \eqref{eqn: proof1_limit},
\begin{eqnarray}\label{eqn: proof_PGS_3}
    \notag \liminf_{k\to\infty} \mbox{PGS}_m 
    &\geq&\ \Pr\{\bar X^*_{(m)}\geq \mu_m-\delta_{\gamma}\} \liminf_{k\to\infty}\ \Pr \left\{\Omega_{N}|\bar X^*_{(m)}\geq \mu_m-\delta_{\gamma}\right\}-\limsup_{k\to\infty}\Pr\left\{\Omega_{\argmin}^c\right\} \\
    &=&\ \Pr\{\bar X^*_{(m)}\geq \mu_m-\delta_{\gamma}\}.
\end{eqnarray}
Lastly, by the definition of $\delta_{\gamma}$, we can see that $\delta_\lambda\to \delta_{0}$ as $\lambda\to 0$. Meanwhile, it is important to notice that Equation \eqref{eqn: proof_PGS_3} is true for any $\lambda\in (0,n_g/m+n_0-C(\delta/\bar\sigma;n_0)) $ and its corresponding choice of $\delta_\lambda$. Therefore, we can let $\lambda\to 0$ in Equation \eqref{eqn: proof_PGS_3} to further get 
\begin{eqnarray}\label{eqn: proof_PGS_4}
    \notag \liminf_{k\to\infty} \mbox{PGS}_m 
    &\geq&  \Pr\{\bar X^*_{(m)}\geq \mu_m-\delta_{0}\}.
\end{eqnarray}
    This concludes the proof of the theorem. \hfill\Halmos
    \end{proof}

\subsection{Proof of Lemma \ref{lem: hitting_time_prob}}
To prove the lemma, we first prepare the following lemma.
\label{subsec: proof_hitting_time_prob}
\begin{lemma}
    \label{lem: normal_cdf}
        Let $\Phi(x)$ be the cumulative distribution function of the standard normal distribution. Then, for any $ x > 0$,
        $$1 - \Phi(x) \leq \exp\left(-\frac{x^2}{2}\right).$$
    \end{lemma}

\begin{proof}{\textbf{Proof of  Lemma \ref{lem: hitting_time_prob}}:}Let $B(t)$ denote a standard Brownian motion with zero drift and volatility one. Notice that the running average process $\{n(\bar X_j(n)-\mu_j)/\sigma_j: n=n_0,n_0+1,\dots\}$ has the same joint distribution with the discrete-time process $\{B(t): t=n_0,n_0+1, \dots\}$. Then, 
 \begin{eqnarray}\label{eqn: proof_lemma3_1}
        \notag \Pr\left\{\min_{n\ge n_0} \bar X_j(n) > \mu_j-\delta_0\right\}  
        \notag &=&  \Pr\left\{\forall\,n_0\leq n<\infty, \bar X_j(n) > \mu_j-\delta_0 \right\}\\
         \notag &=&  \Pr\left\{\forall\,n_0\leq n<\infty, n(\bar X_j(n)-\mu_j)/\sigma_j > -n\delta_0/\sigma_j\right\}\\
        \notag &=& \Pr\left\{\forall\,n_0\leq n<\infty, B(n) > -n\delta_0/\sigma_j\right\}\\
        &\geq& \Pr\left\{\forall\,n_0\leq t<\infty, B(t) > -t\delta_0/\sigma_j\right\}.
    \end{eqnarray}
   By the time-inversion property of standard Brownian motion $\{B(t):t\geq 0\}$, the new stochastic process $\{\tilde{B}(t): t\geq 0\}$  defined as $\tilde{B}(t):=tB(1/t)$ for $t>0$ and $\tilde{B}(0)=0$ for $t=0$ is also a standard Brownian motion. Therefore, we may further derive Equation \eqref{eqn: proof_lemma3_1} as
   \begin{eqnarray}\label{eqn: proof_lemma3_2}
        \notag \Pr\left\{\,\min_{n\ge n_0} \bar X_j(n) > \mu_j-\delta_0\right\}  
        \notag &\geq& \Pr\left\{\,\forall\,n_0\leq t<\infty, B(t) > -t\delta_0/\sigma_j\right\}\\
        \notag &=& \Pr\left\{\,\forall\,n_0\leq t<\infty, s=1/t, sB(1/s) > -\delta_0/\sigma_j\right\}\\
        \notag &= & \Pr\left\{\,\forall\,0< s\leq 1/n_0, \tilde{B}(s) > -\delta_0/\sigma_j\right\}\\
        \notag &=& \Pr\left\{\inf_{0\leq s\leq 1/n_0}  \tilde{B}(s) >-\delta_0\sigma_j\right\}\\
        &=& \Pr\left\{\sup_{0\leq s\leq 1/n_0}  \tilde{B}(s) \leq \delta_0/\sigma_j\right\},
    \end{eqnarray}
    where the last equality holds from the space symmetry of standard Brownian motion. Furthermore, considering the reflection principle of standard Brownian motion, we see that
    \begin{eqnarray}\label{eqn: proof_lemma3_3}
        \Pr\left\{\sup_{0\leq s\leq 1/n_0}  \tilde{B}(s) \geq \delta_0/\sigma_j\right\} = 2\Pr\{\tilde{B}(1/n_0) \geq \delta_0/\sigma_j\} = 2\left(1-\Phi\left(\frac{\sqrt{n_0}\delta_0}{\sigma_j}\right)\right). \quad \quad
    \end{eqnarray}
    Plugging Equation \eqref{eqn: proof_lemma3_3} into Equation \eqref{eqn: proof_lemma3_2} results in 
    \begin{eqnarray}\label{eqn: proof_min}
        \Pr\left\{\,\min_{n\ge n_0} \bar X_j(n) > \mu_j-\delta_0\right\}\geq 1-2\left(1-\Phi\left(\frac{\sqrt{n_0}\delta_0}{\sigma_j}\right)\right)\geq 1-2\left(1-\Phi\left(\frac{\sqrt{n_0}\delta_0}{\bar\sigma}\right)\right), \quad \quad
    \end{eqnarray}
where the the last inequality arises from Assumption \ref{assu: guassian}.
Then, by Lemma  \ref{lem: normal_cdf}, it is straightforward to draw the desired conclusion.
    \hfill \Halmos
    \end{proof}

    \subsection{Proof of Theorem \ref{thm: consistency}}
    \label{subsec: proof_consistency}

To prove the theorem, we first prepare the following lemma.
\begin{lemma}[Lemma 6 of \citealt{itemLi2024}]
    \label{lem: bct_mean_convergence}
    For any positive constant $ b > 0$, there exist positive constants $\beta$ and $\kappa$ such that $C\left(b; n_0\right) - n_0 \le \beta e^{-\kappa n_0}$, where $\kappa= \frac{b^2}{2}$ and $\beta=\left(1-e^{-\kappa}\right)^{-1}$.
    \end{lemma}
    
    \begin{proof}{\textbf{Proof of Theorem \ref{thm: consistency}}:}
  It follows from Lemma \ref{lem: bct_mean_convergence} and the choice of $n_g$ that
        \begin{eqnarray}
        \label{eq: pac_1}
           n_g=  m\left[C\left(\frac{\delta}{2\bar \sigma}; n_0\right)-n_0\right]  
            \leq  m\beta \exp\left(-\frac{\delta^2}{8\bar\sigma^2} n_0\right),
        \end{eqnarray}
where $\beta=\left(1-e^{-\frac{\delta^2}{8\bar\sigma^2}}\right)^{-1}$. By the choice of $n_0$, namely, $n_0=\frac{8 \bar \sigma^2}{\delta^2}\log\frac{2m}{\delta}$, we further derive Equation \eqref{eq: pac_1} as
\begin{eqnarray}
        \label{eq: pac_2}
           n_g \leq  m\beta \exp\left(-\frac{\delta^2}{8\bar\sigma^2} n_0\right)=\frac{\beta\alpha}{2}.
        \end{eqnarray}
Next, notice that $e^{-x}\leq 1/(1+x)$ for any $x>0$, so we have,
 \begin{eqnarray}\label{eq: pac_3}
     \beta=\left(1-e^{-\frac{\delta^2}{8\bar\sigma^2}}\right)^{-1}\leq \left(1- \frac{1}{\frac{\delta^2}{8\bar\sigma^2}+1}\right)^{-1} = \frac{\frac{\delta^2}{8\bar\sigma^2}+1}{\frac{\delta^2}{8\bar\sigma^2}} = 1+\frac{8\bar\sigma^2}{\delta^2}.
 \end{eqnarray}
 Combining Equation \eqref{eq: pac_2} and Equation \eqref{eq: pac_3} leads to the conclusion.
         \hfill \Halmos\end{proof}

\subsection{Proof of Lemma \ref{lem: negligible}}\label{ec: lastexit}

\begin{proof}{Proof.}
    In this proof, we first demonstrate that $L_{\mathcal{G}}^j(\delta/2;n_0)=\sup \left\{n \in [n_0, \infty) : |\bar X_j(n)  - \mu_{j}| \geq \delta/2\right\} $ is finite almost surely for any alternative $j\in\mathcal{G}$. For any sufficiently large  integer $M>n_0$, we have
    \begin{eqnarray}\label{eqn: proof_lastexit_1}
       \notag \Pr\left\{L_{\mathcal{G}}^j(\delta/2;n_0)< M\right\}
     \notag   &=& \Pr\left\{\forall\, n\geq M, \, |\bar{X}_j(n)-\mu_j|\leq \delta/2\right\}\\
   \notag     &=& \Pr\left\{\min_{n\geq M} \bar{X}_j(n)\geq \mu_j- \delta/2,\max_{n\geq M} \bar{X}_j(n)\leq \mu_j+\delta/2\right\}\\
       &\geq& \Pr\left\{\min_{n\geq M} \bar{X}_j(n)\geq \mu_j- \delta/2\right\}-\Pr\left\{\max_{n\geq M} \bar{X}_j(n)\geq \mu_j+\delta/2\right\}, \quad \quad \quad \quad \quad
    \end{eqnarray}
where the inequality holds because $\Pr\{A\cap B\}=\Pr\{A\}-\Pr\{A\cap B^c\}\geq \Pr\{A\}-\Pr\{ B^c\}$ for any two events $A$ and $B$.

Notice that the second term in Equation \eqref{eqn: proof_lastexit_1} satisfies
\begin{eqnarray}\label{eqn: proof_lastexit_3}
   \notag \Pr\left\{\max_{n\geq M} \bar{X}_j(n)\geq \mu_j+\delta/2\right\}
    \notag &=& 1-\Pr\left\{\max_{n\geq M} \bar{X}_j(n)\leq \mu_j+\delta/2\right\}\\
    \notag &=& 1-\Pr\left\{\max_{n\geq M}\, (\bar{X}_j(n)-\mu_j)\leq \delta/2\right\}\\
    \notag &=& 1-\Pr\left\{\min_{n\geq M}\, (\bar{X}_j(n)-\mu_j)\geq -\delta/2\right\}\\
  &=& 1-\Pr\left\{\min_{n\geq M}\, \bar{X}_j(n)\geq \mu_j-\delta/2\right\},
\end{eqnarray}
where the third equality holds because $\{\bar{X}_j(n)-\mu_j: n\geq 1\}$ and $\{-(\bar{X}_j(n)-\mu_j): n\geq 1\}$ have the same joint probability density function. Plugging Equation \eqref{eqn: proof_lastexit_3} into \eqref{eqn: proof_lastexit_1} leads to 
\begin{eqnarray*}
    \Pr\left\{L_{\mathcal{G}}^j(\delta/2;n_0)< M\right\}\geq 2\Pr\left\{\min_{n\geq M}\, \bar{X}_j(n)\geq \mu_j-\delta/2\right\}-1.
\end{eqnarray*}
    Then, similar to Equation \eqref{eqn: proof_min} in the proof of Lemma \ref{lem: hitting_time_prob}, we have
\begin{eqnarray*}
        \Pr\left\{\,\min_{n\geq M} \bar X_j(n) > \mu_j-\delta/2\right\}\geq 1-2\left(1-\Phi\left(\frac{\sqrt{M}\delta}{2\sigma_i}\right)\right),
    \end{eqnarray*}
where $\Phi(\cdot)$ denotes the cumulative distribution function of a standard normal distribution. Therefore, 
\begin{eqnarray*}
    \Pr\left\{L_{\mathcal{G}}^j(\delta/2;n_0)< M\right\}\geq 1-4\left(1-\Phi\left(\frac{\sqrt{M}\delta}{2\sigma_i}\right)\right)
\end{eqnarray*}
Letting $M\to\infty$ in the equation above, we may conclude that
\begin{eqnarray}\label{eqn: proof_lastexit_4}
    \Pr\left\{L_{\mathcal{G}}^j(\delta/2;n_0)<\infty\right\}=1.
\end{eqnarray}

Now we move to demonstrate the desired conclusion. Observe that
 \begin{eqnarray*}
    &&\limsup_{k \rightarrow \infty} \frac{1}{k} \sum\limits_{j\in\mathcal{G}'} \left(L_{\mathcal{G}'}^j(\delta/2;n_0)-\argmin_{n \in [n_0, \infty)}\bar{X}_j(n)\right)^+\\
    & \leq & \limsup_{k \rightarrow \infty} \frac{1}{k} \sum\limits_{j\in\mathcal{G}} \left(L_{\mathcal{G}}^j(\delta/2;n_0)-\argmin_{n \in [n_0, \infty)}\bar{X}_j(n)\right)^+\quad (\mbox{because } \mathcal{G}'\subseteq\mathcal{G})\\
    & \leq &  \limsup_{k \rightarrow \infty} \frac{1}{k} \sum\limits_{j\in\mathcal{G}} L_{\mathcal{G}}^j(\delta/2;n_0)+\limsup_{k \rightarrow \infty} \frac{1}{k} \sum\limits_{j\in\mathcal{G}} \argmin_{n \in [n_0, \infty)}\bar{X}_j(n)\\
    &=&0, \mbox{ almost surely.}
\end{eqnarray*} 
The last equality holds from Equation \eqref{eqn: proof_lastexit_4} and Lemma \ref{lem: argmin}. The proof is completed. \hfill\Halmos
\end{proof}

\subsection{Proof of Theorem \ref{thm: nonnormalPGSR}}
\label{ec: PGSR}
\begin{proof}{Proof.}
Given the similarity between the lower bounds for PGS$_m$ and PGSR$_m$ as listed in Equations \eqref{eqn: bound_PGS2} and \eqref{eqn: bound_PGSR},  we may follow the proofs of Theorem \ref{thm: nonnormalPGS} and Theorem \ref{thm: consistency} to complete this proof, by specifically dealing with the additional term $\sum_{j\in\mathcal{G}'} \left(L_{\mathcal{G}'}^j(\delta/2;n_0)-\argmin_{n \in [n_0, \infty)}\bar{X}_j(n)\right)^+$ in Equation \eqref{eqn: bound_PGSR}.

When the total evaluation budget $B=(n_0 + n_g) k$, the PGSR$_m$ of the EFG-$m$ algorithm  in Equation \eqref{eqn: bound_PGSR} can be rewritten as
    \begin{eqnarray}\label{eqn: proof3_PGSR}
      \notag \mbox{PGSR}_m 
         &\geq& \Pr\left\{  \frac{n_g}{m} + n_0\geq \frac{1}{k}\left[\sum\limits_{j\in\mathcal{G}'} \left(L_{\mathcal{G}'}^j(\delta/2;n_0)-\underset{n \in [n_0, \infty)}{\argmin}\bar{X}_j(n)\right)^+ \right.\right. \\
        && \qquad \qquad\qquad\qquad\qquad\left.\left. +\sum\limits_{l\in\mathcal{G}}\underset{n \in [n_0, \infty)}{\argmin}\bar{X}_l(n)+\sum\limits_{i\in \mathcal{K}\setminus\mathcal{G}}N_i(\bar{X}^*_{(m)};n_0)\right]\right\}. \qquad
        \end{eqnarray}
Inspired by Equation \eqref{eqn: proof1_delta_lambda}, we can arbitrarily fix a constant $\lambda\in(0,n_g/m+n_0-C(\delta/\bar{\sigma};n_0))$ and accordingly set $\delta_\lambda\in(0,\delta)$ such that
\begin{eqnarray}\label{eqn: proof3_delta_lambda}
       C\left(\frac{\delta-\delta_\lambda}{\bar\sigma};n_0\right) = n_g/m+n_0-\lambda >0.
   \end{eqnarray}
   Meanwhile, for the ease of presentation, we define the events
   \begin{eqnarray*}
           \Omega_{L} =\ \left\{\frac{1}{k}\sum\limits_{j\in\mathcal{G}'} \left(L_{\mathcal{G}'}^j(\delta/2;n_0)-\underset{n \in [n_0, \infty)}{\argmin}\,\bar{X}_j(n)\right)^+\leq \frac{\lambda}{3}\right\}, \,  \Omega_{\argmin} =\ \left\{\frac{1}{k}\sum\limits_{l\in\mathcal{G}}\underset{n \in [n_0, \infty)}{\argmin}\,\bar{X}_l(n)\leq \frac{\lambda}{3}\right\},
   \end{eqnarray*}
   and for the chosen $\delta_{\gamma}$, 
   \begin{eqnarray*}
    \Omega_{N}= \left\{ \frac{1}{k}\sum_{i\in\mathcal{K}\setminus\mathcal{G}}  N_i \left(X^*_{(m)};n_0\right)\leq C\left(\frac{\delta-\delta_{\gamma}}{\bar\sigma};n_0\right)+\frac{\lambda}{3}\right\}.
   \end{eqnarray*}
   Then, the PGS$_m$ lower bound in Equation \eqref{eqn: proof3_PGSR} can be  further derived as
   \begin{eqnarray}\label{eqn: proof3_PGSR_1}
         \notag  \mbox{PGSR}_m  &\geq& \Pr\Bigg\{ \Bigg\{ \frac{n_g}{m} + n_0\geq \frac{1}{k}\Bigg[\sum\limits_{j\in\mathcal{G}'} \left(L_{\mathcal{G}'}^j(\delta/2;n_0)-\underset{n \in [n_0, \infty)}{\argmin}\bar{X}_j(n)\right)^+ \\
         \notag  && \qquad \qquad\qquad\qquad +\sum\limits_{l\in\mathcal{G}}\underset{n \in [n_0, \infty)}{\argmin}\bar{X}_l(n)+\sum\limits_{i\in \mathcal{K}\setminus\mathcal{G}}N_i(\bar{X}^*_{(m)};n_0)\Bigg]\Bigg\}  \cap\, \Omega_L\cap \,\Omega_{\argmin}\cap\, \Omega_N\Bigg\}\\
       \notag &\geq& \Pr\left\{ \left\{ \frac{n_g}{m} + n_0\geq C\left(\frac{\delta-\delta_{\gamma}}{\bar\sigma};n_0\right)+\lambda\right\}\cap\, \Omega_L\cap \,\Omega_{\argmin}\cap\, \Omega_N\right\}\\
      \notag &=& \Pr\left\{ \Omega_L\cap \,\Omega_{\argmin}\cap\, \Omega_N\right\} \\
      \notag &\geq & \Pr\left\{\Omega_N\right\}
      -\Pr\left\{\Omega_{\argmin}^c\right\}
      -\Pr\left\{\Omega_L^c\right\}\\
      &\geq & \Pr\left\{ \Omega_N \mid \{\bar X^*_{(m)}\geq \mu_m-\delta_{\gamma}\}\right\}\Pr\left\{\bar X^*_{(m)}\geq \mu_m-\delta_{\gamma}\right\} -\Pr\left\{\Omega_{\argmin}^c\right\}
      -\Pr\left\{\Omega_L^c\right\},
     \end{eqnarray}
   where the second inequality holds because of the definitions of the events $ \Omega_{L}$, $ \Omega_{N}$ and $ \Omega_{\argmin}$, the first equality holds due to Equation \eqref{eqn: proof3_delta_lambda}, and the third inequality holds because of the Bonferroni's inequality. 
   Notice that Equation \eqref{eqn: proof_PGS_3} gives the limiting behavior of the first two probability terms in \eqref{eqn: proof3_PGSR_1}. More specifically, we have
     \begin{eqnarray*}
         \liminf_{k\to\infty}\mbox{PGSR}_m\geq \Pr\left\{\bar X^*_{(m)}\geq \mu_m-\delta_{\gamma}\right\}-\limsup_{k\to\infty} \Pr\left\{\Omega_L^c\right\}.
     \end{eqnarray*}  
Then, it readily follows from Lemma \ref{lem: negligible} that $\limsup_{k\to\infty} \Pr\left\{\Omega_L^c\right\}=0$.
As the inequality above holds for any $\lambda\in(0,n_g/m+n_0-C(\delta/\bar{\sigma};n_0))$, by letting $\lambda\to 0$, we further have $\liminf_{k\to\infty}\mbox{PGSR}_m\geq \Pr\left\{\bar X^*_{(m)}\geq \mu_m-\delta_{0}\right\}$. The sample optimality is now proved.
     
     For the consistency, notice that PGSR$_m$ shares the same limiting lower bound as the PGS$_m$. Then, following exactly the same logic in the proof of Theorem \ref{thm: consistency}, it is straightforward to show that the EFG-$m$ algorithm is also consistent in terms of PGSR$_m$. The proof is completed.
   \hfill\Halmos
   \end{proof}
         \label{subsec: properties_pcs}
             \label{cor: nonnormalPCS}
         
    \newpage

\section{Supplements to Section \ref{sec: enhancements}}


\subsection{Further Discussion on Top-$M$ Greedy Selection}
\label{subsec: additional_topM}

We extend the discussion from Section \ref{subsec: top_M} to analyze the inefficiency of the top-\(m\) greedy phase in the general case where the observations of all alternatives are subject to random noise and illustrate the benefits of top-\(M\) greedy selection. In the top-\(m\) greedy phase, recall from Section \ref{sec: top-m-boundary-crossing} that, from a sample-path viewpoint, \( \bar{X}^{**} = \min_{j\in\mathcal{S}^*} \bar{X}_j^* \), where \( \bar{X}_j^* = \min_{n \in [n_0, \infty)} \bar{X}_j(n) \) denotes the minimum running average for each alternative \( j \in \mathcal{S}^* \), serves as a common boundary for all inferior alternatives. A correct screening event is ensured \emph{once} the sample means of all inferior alternatives fall below this boundary. Therefore, \emph{any additional sampling of the top-\(m\) alternatives after they reach the boundary $\bar{X}^{**}$ is unnecessary} and leads to inefficiency.  
To illustrate this inefficiency in the general case, we visualize the top-2 greedy selection process for an example problem with 4 alternatives in the left panel of Figure \ref{fig: top_3_sampling_general}. In this process, \( \bar{X}_2^* \) acts as the boundary for the inferior alternatives (3 and 4). After the 3rd round of the top-2 greedy phase, alternative 2 reaches the boundary \( \bar{X}_2^* \). Ideally, after this round, no additional observations from alternatives 1 and 2 are required. However, because alternative 1 remains in the current top-2 set in subsequent rounds, it continues to be selected and evaluated unnecessarily until alternatives 3 and 4 fall below the boundary at round 9. These redundant evaluations of alternative 1 do not contribute to the boundary-crossing event with respect to \( \bar{X}_2^* \).

\begin{figure}[htbp]
    \begin{center}
        \includegraphics[width=0.99\textwidth]{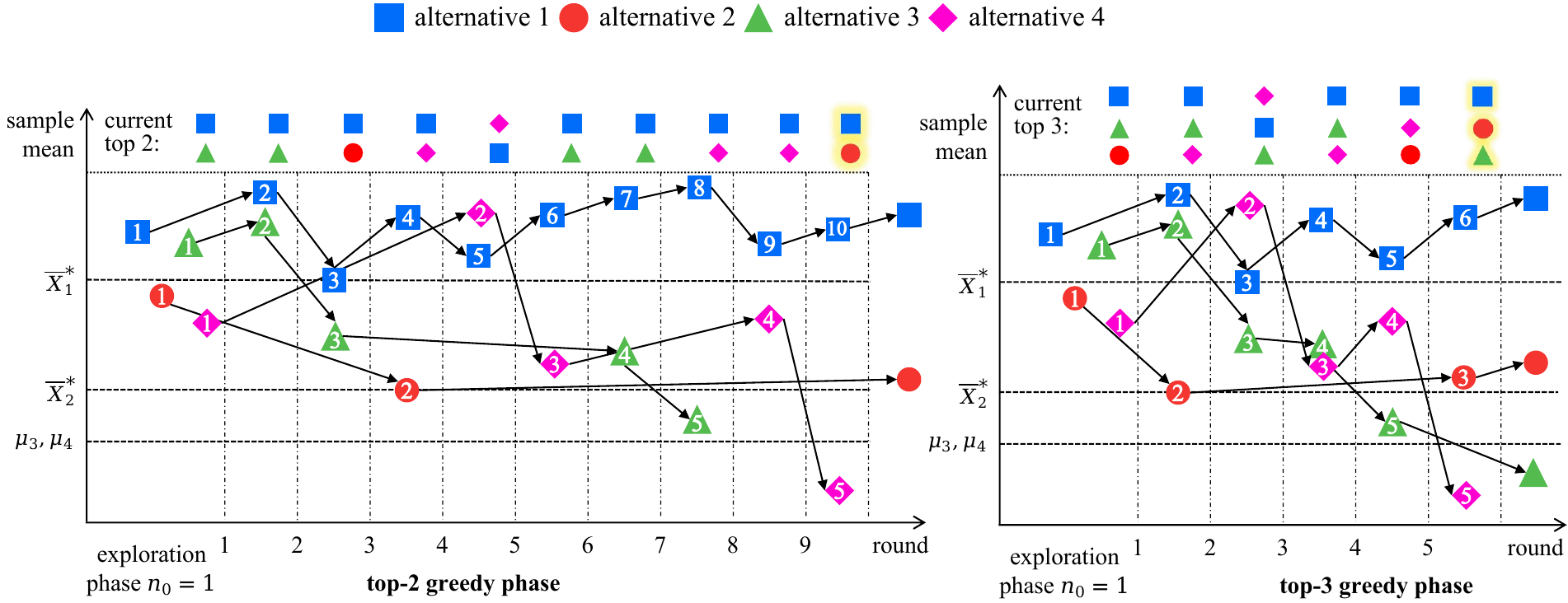}
    \end{center}
    \caption[]{A comparison between top-2 and top-3 greedy selection processes of an example problem with $4$ alternatives. Numbers in the markers represent the total sample sizes.}
    \label{fig: top_3_sampling_general}
\end{figure}

The top-\(M\) greedy selection mitigates this inefficiency by accelerating the entire boundary-crossing process, effectively reducing the number of rounds required for all inferior alternatives to fall below the boundary. For example, in the top-3 greedy selection process shown in the right panel of Figure \ref{fig: top_3_sampling_general}, it takes only 2 rounds for the top-3 greedy phase to reach \( \bar{X}_2^* \), and alternatives 3 and 4 fall below the boundary after round 5. Compared to the top-2 greedy selection process, this saves 4 rounds and eliminates 4 unnecessary observations from alternative 1.  Furthermore, it is intuitive to see that this benefit will become more and more evident as the total number of inferior alternatives increases.

\subsection{The EFG-$M$$^+$ Algorithm}
\label{subsec: EFG-$M$Plus}
As introduced in Section \ref{subsec: adaptive_explore}, to further enhance the EFG-$M$ algorithm, we add a seeding phase before the exploration phase and obtain the EFG-$M^+$ algorithm.  Specifically, the EFG-$M^+$ algorithm consists of three phases: the seeding phase, the exploration phase, and the top-$M$ greedy phase. In the seeding phase, the algorithm evaluates each alternative \( n_{sd} \) times to obtain \( n_{sd} \) observations and then computes its sample mean; based on the sample information, the algorithm then ranks the alternatives in descending order and divides the ranked alternatives into $G$ disjoint groups so that group 1 includes the fewest alternatives, group 2 doubles group 1, group 3 doubles group 2, and so on so force. Subsequently, in the exploration phase, the algorithm equally allocates the exploration budget $n_0 k$ to the $G$ groups and for each group, it equally allocates the group budget to the alternatives inside the group. By doing so, the algorithm will allocate more budget to the alternatives in top groups and that are likely to have top mean performances, which introduces adaptation to exploration. Then, in the top-$M$ greedy phase, the algorithm selects the current top-$M$ alternatives at each round. Finally, it selects the top-$m$ alternatives after the total evaluation budget is exhausted. See Algorithm \ref{algorithm: topMPlus} for a detailed description of the algorithm.

As discussed in Section \ref{subsec: adaptive_explore}, a key advantage of the seeding approach is that it allows us to establish the sample optimality and consistency. To illustrate this, we consider the case of \(M = m\) and analyze the sample optimality for the PCS\(_m\). Let \( e_i \) denote the allocated number of observations for alternative \( i \) in the exploration phase, and define \( \mathbf{e} = [e_1, \ldots, e_k] \). Using the boundary-crossing perspective introduced in Section \ref{subsec: boundary_crossing_PCS}, we may obtain that the PCS\(_m\) of the EFG-\(M^+\) algorithm satisfies
\begin{eqnarray}
    \label{eq: general_PCS2lbound_Plus}
\mbox{PCS}_m &\geq & \E_{\mathbf{e}} \left[ \Pr \left\{ B-n_0k \geq m \left[\sum_{i\in\mathcal{K}\setminus\mathcal{S}^*} N_i \left(\min_{j\in\mathcal{S}^*} \min_{n \geq e_j} \bar{X}_j(n), e_i\right) + \sum_{j\in\mathcal{S}^*} \argmin_{n\geq e_j} \bar{X}_j(n)-n_0k\right] \Bigm| \mathbf{e} \right\} \right]. \quad \quad \quad 
\end{eqnarray}
To simplify the analysis, we remove the dependence on \( \mathbf{e} \) inside the expectation. In the EFG-\(M^+\) algorithm, \( n^g \) denotes the allocated sample size for alternatives in group \( 1 \leq g \leq G \), and we have $
n^G \leq e_i \leq n^1$ for every $i=1, \cdots, k$.
From a sample-path viewpoint, this implies $
\argmin_{n\geq e_j}\bar X_j(n) \leq \argmin_{n \geq n^1}\bar X_j(n)
$ and  
\[
N_i \left(\min_{j\in\mathcal{S}^*} \min_{n \geq e_j} \bar{X}_j(n), e_i\right) \leq N_i \left(\min_{j\in\mathcal{S}^*} \min_{n \geq e_j} \bar{X}_j(n), n^1\right) \leq N_i \left(\min_{j\in\mathcal{S}^*} \min_{n \geq n^G} \bar{X}_j(n), n^1\right).
\]
Consequently,  the PCS\(_m\) lower bound in Equation \eqref{eq: general_PCS2lbound_Plus} can be further refined as
\begin{eqnarray}
    \label{eq: general_PCS2lbound_Plus2}
\notag \mbox{PCS}_m &\geq &  \Pr \left\{ B-n_0k \geq m \left[\sum_{i\in\mathcal{K}\setminus\mathcal{S}^*} N_i \left(\min_{j\in\mathcal{S}^*} \min_{n \geq n^G} \bar{X}_j(n), n^1\right) + \sum_{j\in\mathcal{S}^*} \argmin_{n\geq n^1} \bar{X}_j(n)-n_0k\right]  \right\}. \quad 
\end{eqnarray}
This lower bound has the same structure as the PCS\(_m\) lower bound for the EFG-\(m\) algorithm in Equation \eqref{eq: general_PCS2lbound}. Therefore, we can follow the proof strategy in Section \ref{sec: sample-optimality} to establish the sample optimality.

Notice that in the EFG-$M^+$ algorithm, the number of groups $G$ is a user-specified parameter. In our experiments, we let $G=\left\lfloor \log_2 (k/m)\right\rfloor$ to ensure that each group includes at least one alternative. Numerical results in Section \ref{subsec: numerical_practical} show that this algorithm performs very well in solving large-scale virtual screening problems.

\RestyleAlgo{ruled}
\LinesNumbered
\SetAlgorithmName{Algorithm}{Algorithm}{Algorithm}
\SetAlgoCaptionLayout{centerline}
\begin{algorithm}[hbtp]
\caption{\textbf{Enhanced Explore-First Top-${M}$ Greedy (EFG-$M$$^{\mathbf+}$) Algorithm}}
    \label{algorithm: topMPlus}
\KwIn{$m$ and $M$, $k$ alternatives and their prompts, the total evaluation budget $B=(n_{sd} + n_0 + n_g) k $, $n_{sd}$, $n_0$, $n_g$, the total number of groups $G$, and the LLM}

\For(\tcp*[f]{{Seeding Phase}}){$i=1$ \KwTo $k$ }{
prompt the LLM $n_{sd}$ times for alternative $i$  to get $n_{sd}$ observations $x_{i1},\ldots,x_{in_{sd}}$\;
set $\bar X^{sd}_i =\frac{1}{n_{sd}}\sum_{j=1}^{n_{sd}} x_{ij}$\;
}
according to $\bar X^{sd}_i$, sort the alternatives in descending order as $\{ (1), (2), \ldots, ( k) \}$\;
    
let $I^r$ denote the group of alternatives for $r=1,\ldots, G$. Set $\Delta = 2^G-1$. Then, let $I^{1} = \{ (1), (2), \ldots, ( \lfloor  k / \Delta \rfloor ) \}$,  $I^{r} = \{ (\lfloor k 2^{r-2} / \Delta \rfloor+1), \ldots, ( \lfloor k 2^{r-1} / \Delta \rfloor) \}$ for $r=2, \ldots, G-1$, and $I^{G} = \{ (\lfloor k 2^{G-2} / \Delta \rfloor+1), \ldots, (k) \}$\;

    \For(\tcp*[f]{{Exploration Phase}}){$r=1$ \KwTo $G$}{
        set the sample size $n^r = \left\lfloor \frac{n_0 (2^G-1)}{G 2^{r-1} } \right\rfloor$\;
        \For{ $i \in I^r$}{
            prompt the LLM $n^r$ times for alternative $i$  to get $n^r$ observations $x_{i1},\ldots,x_{in^r}$\;
            
             set $\bar X_i(n^r)=\frac{1}{n^r}\sum_{j=1}^{n^r} x_{ij}$ and  let $n_i=n^r$\;
        }
    }
\While(\tcp*[f]{{Top-M Greedy Phase}}){$\sum_{i=1}^k n_i < B$ }{
let $\mathcal{S} = \,  \stackrel{1, \dots, M}{ \argmax}_{i \in \{1, \ldots, k\}} \bar X_i(n_i)$\;
  \For{$j \in\mathcal{S}$}{
  prompt the LLM for alternative $j$ to get one observation $x_{j}$; 
  
  update $\bar X_{j}(n_j+1) = \frac{1}{n_j+1}\left[n_j\bar X_{j}(n_j) + x_j\right]$ and let $n_j = n_j+1$\;
  }
}
\KwOut{$\stackrel{1, \dots, m}{ \argmax}_{i \in \{1, \ldots, k\}} \bar X_i(n_i)$}
\end{algorithm}

\subsection{The Greedy-as-Remedy Approach for the SAR Algorithm}
\label{subsec: SAR-$M$Plus}
In this subsection, we evaluate the greedy-as-remedy approach  by using the SAR algorithm in the exploration and comparing it to  EFG-\(M\) and SAR. While SAR may be not sample optimal, its adaptive budget allocation strategy could be more efficient than the equal allocation used in the exploration phase of EFG-\(M\). By incorporating a top-\(M\) greedy phase  after SAR, we introduce a new algorithm, SAR-$M$, and examine whether this combination can improve performance over EFG-\(M\) while achieving sample optimality for SAR.   To assess this, we conduct numerical experiments comparing  SAR-$M$ with EFG-\(M\) using the RM-Normal, RM-LogNormal, and RM-Pareto configurations listed in Table \ref{table: topm_confs} and the same experimental settings as in Section \ref{subsec: numerical_topm_enhancement}. For the EFG-$M$ and SAR-$M$ algorithms we let $M=2m$. The following Figure \ref{figure: greedy_remedy_SAR} plots the PCS\(_m\) curves of the SAR, SAR-\(M\) and EFG-\(M\) algorithms under the three configurations against \( k \).
    \begin{figure}[ht]
        \centering
        \includegraphics[width=0.8 \textwidth]{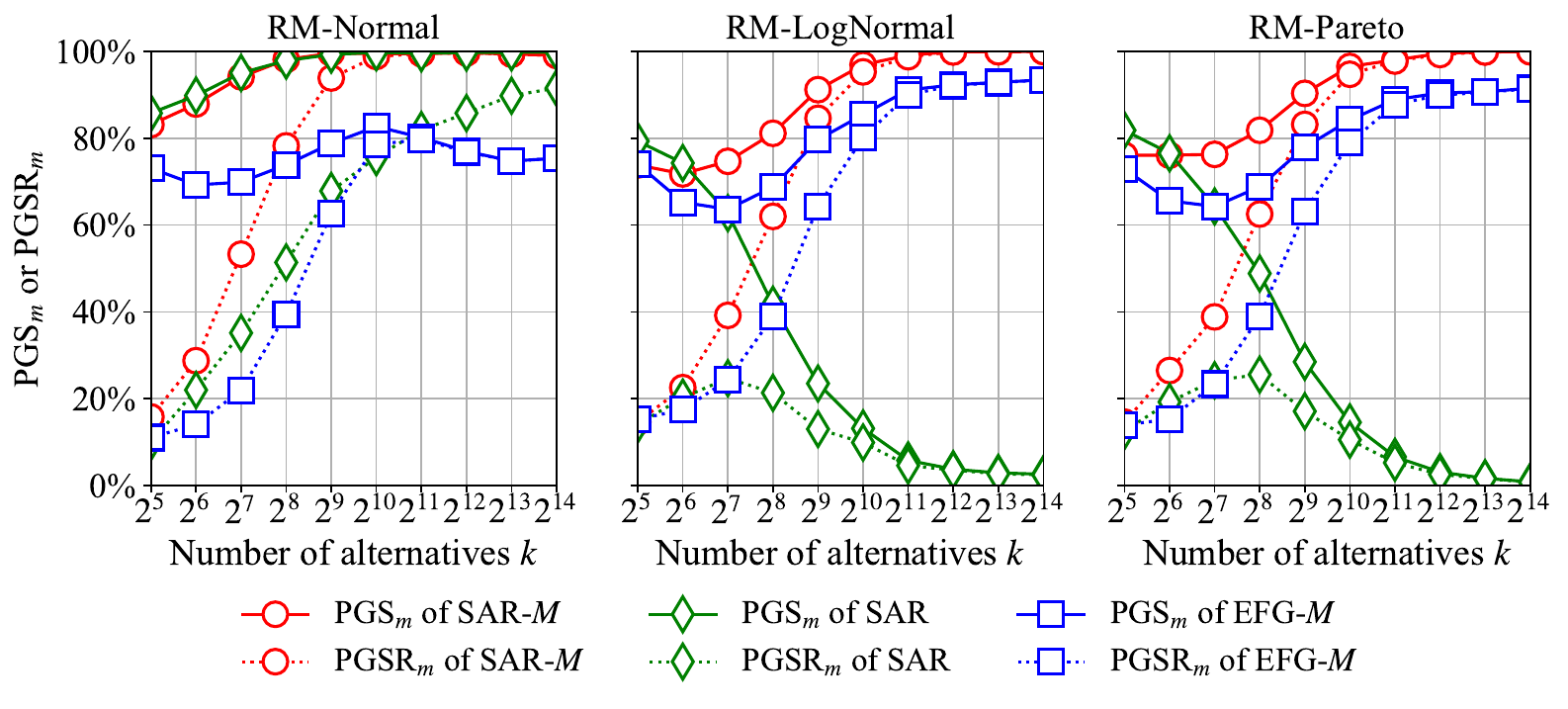}
        \caption{A comparison between the PGS$_m$ of the EFG-$M$, SAR, and SAR-$M$ algorithms}
        \label{figure: greedy_remedy_SAR}
        \end{figure}

From Figure \ref{figure: greedy_remedy_SAR}, we observe several key findings. First, the SAR-\(M\) algorithm consistently outperforms EFG-\(M\) across all configurations, demonstrating that using SAR in the exploration phase leads to a more effective budget allocation than the equal allocation used in EFG-\(M\). Second, while SAR alone is not sample optimal in non-normal configurations, adding a top-\(M\) greedy phase corrects this limitation. The SAR-\(M\) algorithm appears sample optimal and achieves significantly better performance than SAR. Third, under the RM-Normal configuration, where SAR alone may already perform well, SAR-\(M\) maintains at least comparable performance, indicating that the greedy-as-remedy approach does not introduce inefficiencies. Overall, these findings validate the potential of the greedy-as-remedy approach in enhancing the performance of EFG-\(M\) while achieving sample optimality.  

\subsection{Parallelization of the EFG-$M$$^+$ Algorithm}
\label{subsec: additional_parallel}

In this subsection, we present additional numerical results to evaluate the screening accuracy and parallel efficiency of the asynchronous EFG-\(M^{++}\) algorithm, using the non-parallel EFG-\(M^+\) as the benchmark comparator. 
We adopt a client-server framework for parallel prompting, in which a central server assigns evaluation tasks to multiple LLM instances (workers), tracks updates, and maintains sample information of all alternatives.
Our objective is to analyze how the screening accuracy and parallel efficiency of EFG-\(M^{++}\) vary with different numbers of workers \( q \).  
For the experiments, we use the RM-Normal configuration listed in Table \ref{table: topm_confs}, set \( k = 8,192\) and $m=10$, and let the total evaluation budget \( B = 100k \). For both EFG-\(M^+\) and EFG-\(M^{++}\), 20\% of the budget is allocated to the top-\(M\) greedy phase and $M=2m$. To model practical scenarios with varied response times of LLMs, we introduce a time stop distributed as Uniform(0, 1ms) for the generation of each observation. Then, we perform 200 independent macro replications for EFG-\(M^{++}\) to estimate its PGS\(_m\), PGSR\(_m\) and average wall-clock time per replication, while for non-parallel EFG-\(M^+\), we conduct 10 macro replications to estimate its average wall-clock time per replication. Additionally, we run 5,000 replications of EFG-\(M^+\) to estimate its PGS\(_m\) and PGSR$_m$, which are both 0.913. To assess the parallel efficiency of EFG-\(M^{++}\), we use two key metrics: speedup, defined as the ratio of the wall-clock time per replication of EFG-\(M^+\) to that of EFG-\(M^{++}\), and utilization ratio, which is speedup divided by $q$, measuring how efficiently computational resources are utilized. The wall-clock times, PGS\(_m\) and PGSR\(_m\) values for the EFG-$M^{++}$ algorithm, speedup, and utilization ratios are summarized in Table \ref{tab: parallel_efficiency}.

\begin{table}[htbp]
\centering
\begin{tabular}{ccccccc}
\hline

\hline

\hline
 $\#$ of workers $q$ & $\begin{array}{c}\text {PGS}_m \text{ of}\\
    \text {EFG$-M^{++}$}\end{array}$   & $\begin{array}{c}\text {PGSR}_m \text{ of} \\
\text {EFG$-M^{++}$}\end{array}$ & $\begin{array}{c}\text {Wall-clock time} \\
\text {of EFG$-M^{++}$ (s)}\end{array}$ &
$\begin{array}{c}\text {Wall-clock time} \\
\text {of EFG$-M^{+}$ (s)}\end{array}$ & Speed-up          & $\begin{array}{c}\text {Utilization} \\
\text {Ratio}\end{array}$        \\
\hline

\hline

\hline
10                                                                              & 0.945  & 0.945  & 48.559  & 483.618               & 9.959                                                                                                                                              &  99.6\%  \\
20                                                                               & 0.935   & 0.935 & 25.184                & 483.618                                                                                                                                             & 19.203   &  96.0\%\\
30                                                                               & 0.895 & 0.895  & 17.401               & 483.618                                                                                                                                             & 27.793   &92.6\%  \\
40                                                                             & 0.910 & 0.910   & 13.131                & 483.618                                                                                                                                            & 36.831   &  92.1\% \\ 
\hline

\hline

\hline
\end{tabular}
\caption{Performance and parallel efficiency of the EFG$-M^{++}$ algorithm}
\label{tab: parallel_efficiency}
\end{table}

From Table \ref{tab: parallel_efficiency}, we can observe that the asynchronous EFG-\(M^{++}\) algorithm achieves comparable PGS\(_m\) and PGSR$_m$ to the EFG-\(M^+\) algorithm while maintaining high parallel efficiency. Across different values of \( q \), the PGS\(_m\) of EFG-\(M^{++}\) may become slightly lower than that of EFG-\(M^+\); however, the differences remain negligible and do not increase as \( q \) grows. Furthermore, as \( q \) increases, the speedup scales proportionally with it. Notably, the utilization ratio remains consistently above 85\%, regardless of \( q \). These results indicate an efficient parallelization.

\section{Numerical Justification of Theoretical Results and Proposed Algorithms}
\label{sec: ec_num_justification}
In this section, we conduct numerical experiments to validate our theoretical results and evaluate the performance of our proposed algorithms.  We begin in Section \ref{subsec: problem_conf} by introducing the synthetic problem configurations, experimental settings, and several algorithm comparators. In Section \ref{subsec: numerical_topm_sample_optimality}, we verify the sample optimality and consistency of the EFG-\(m\) algorithm and demonstrate the free ranking effect.  We then test the impact of top-\(M\) greedy selection in Section \ref{subsec: numerical_topm_enhancement}. In Section \ref{subsec: comparison}, we compare the performance of our algorithms against the comparators. 

\subsection{Synthetic Problem Configurations, Experimental Settings, and Comparators}
\label{subsec: problem_conf}
\subsubsection{Problem Configurations.} In this section, we use synthetic problem configurations to verify our theoretical results and algorithm enhancements. 
By ``synthetic", we mean that the random evaluation output \( X_i \) from the LLM for alternative \( i \) is generated according to the assumed distribution. In Section \ref{subsec: numerical_LLM}, we present numerical results based on LLM-generated outputs.
We consider two types of problem configurations that are common in the literature: the slippage configuration (SC) and the configuration with randomly generated means (RM). The SC configurations will be used to evaluate the PCS$_m$ while the the RM configurations will be used to evaluate the PGS$_m$ and PGSR$_m$. To formulate the problem configurations, we first set the distribution of alternative 1 (e.g., \(X_1 \sim \text{Normal}(\mu, \sigma^2)\)), and then use mean shifting to set the distributions of the other alternatives. The mean shifting methods for each SC configuration and RM configuration are described by Equations \eqref{eq: SC} and \eqref{eq: topm_random_conf1}, respectively, 
\begin{equation}
    \label{eq: SC}
    X_i \stackrel{d}{=} 
    \begin{cases}
        X_1 &  i=2, \ldots, m, \\
        X_1-\gamma &   i=m+1, \ldots, k;
    \end{cases}
  \end{equation}
  \begin{equation}
    \label{eq: topm_random_conf1}
    X_i \stackrel{d}{=}  X_1 + \delta_i \text{ with }
\delta_i = 
\begin{cases}
    \text{Uniform}\left(\delta, 3\delta\right)  & i=2, \ldots, m,\\ 
    \text{Uniform}\left(0,\delta\right) & i=m+1, \ldots, g,\\
    \text{Uniform}(-1,0) & i=g+1, \ldots, k. \quad \quad \quad \quad 
\end{cases}
\end{equation}
Here, the symbol \(\stackrel{d}{=}\) denotes the identically distributed relationship between two random variables, and $g \geq m$ represents the maximal number of good alternatives having a mean no smaller than $\mu_m - \delta$.  By default, we set \(\gamma=\delta=0.1\), $m=10$, and for simplicity, fix $g=15$.

Notice that despite the convention of assuming normality in designing and analyzing a algorithm, it is also of particular interest to test the algorithm's performance for non-normal observations to improve its practical applicability. Motivated by this, for both SC and RM configurations, we consider three types of distributions to generate the observations in our experiments: normal, log-normal, and Pareto. The normal distribution represents a light-tailed setting, while the log-normal and Pareto distributions are heavy-tailed, posing more challenging scenarios. As a result, we obtain six problem configurations including SC-Normal, SC-LogNormal, SC-Pareto, RM-Normal, RM-LogNormal and RM-Pareto.  Note that the normal, log-normal, and Pareto distributions are all determined by two parameters. We choose the parameter values such that the alternatives of different configurations have similar variances within each class. Details of each problem configuration and the parameter values are summarized in Table \ref{table: topm_confs}.
\begin{table}[htbp]
    \centering
    \begin{tabular}{ccc}
        \hline 
                    
        \hline

        \hline
   Problem Configuration & Distributional Assumption & Parameters \\ 
   \hline 
                    
   \hline

   \hline
   SC-Normal    & $X_1\sim \text{Normal}(\mu, \sigma^2)$   & $\mu=0.1$,  $\sigma=0.6$  \\ \hline
   
   SC-LogNormal & $X_1\sim \text{LogNormal}(\mu, \sigma)$& $\mu=-3.7$,  $\sigma=1.8$ \\ \hline
   SC-Pareto    & $X_1\sim \text{Pareto}(x, a)$          & $x=3.1$,  $a=0.8$         \\ \hline
   RM-Normal    & $X_1\sim \text{Normal}(\mu, \sigma^2)$  & $\mu=0.0$,  $\sigma=1.0$  \\ \hline
   RM-LogNormal & $X_1\sim \text{LogNormal}(\mu, \sigma)$  & $\mu=-2.2$,  $\sigma=1.5$ \\ \hline
   RM-Pareto    & $X_1\sim \text{Pareto}(x, a)$            & $x=2.6$,  $a=0.8$         \\                 \hline 
                    
   \hline

   \hline
    \end{tabular}
    \caption{Synthetic Problem Configurations and Parameters}
    \label{table: topm_confs}
    \end{table}

\subsubsection{Experimental Settings and Comparators.} Unless otherwise specified, in our experiments we consider multiple values of the number of alternatives \( k \), denoted by \( k = 2^l \), with \( l \) increasing from 5 to 14. This allows us to assess each algorithm's performance across different problem sizes.  When solving a specific problem under a problem configuration,  we estimate the PCS\(_m\) (PGS\(_m\) or PGSR\(_m\))  of each algorithm based on 2000 independent macro replications. For our top-$m$ and top-$M$ greedy algorithms, given the sample budget \( B=ck \), we allocate 20\% of the total sample budget to the top-$m$ greedy phase. For the EFG-$M^+$ algorithm, we allocate another 20\% of the sample budget to the seeding phase.

There exist several OSS algorithms in the literature, and all of them could potentially be applied to virtual screening problems. However, these algorithms were primarily designed for small-scale problems, and their performance in large-scale settings has not yet been evaluated or analyzed. 
Despite this, to provide a comparison, we include two well-known OSS algorithms in our experiments—one from the simulation literature and one from the bandit literature:  
\begin{itemize}  
    \item Heuristic 1: The Optimal Computing Budget Allocation (OCBAm) algorithm of \cite{chen2008efficient}.  
    \item Heuristic 2: The Successive Accept-and-Reject (SAR) algorithm of \cite{bubeck2013multiple}.  
\end{itemize}  
Implementation details of these algorithms are included in Section \ref{subsec: additional_algorithms}, where we also discuss some other algorithms not being tested.

\subsection{Sample Optimality and Consistency}
\label{subsec: numerical_topm_sample_optimality}
\subsubsection{Sample Optimality.}
\label{subsubsec: sample_optimality}


We start by testing the sample optimality of the EFG-$m$ algorithm for the PCS$_m$ using the SC configurations. Under each SC configuration, we set the total sample budget \( B = 500k \) for every \( k \). Then, for each problem configuration, we plot the PCS\(_m\) of the EFG-$m$ algorithm against different \( k \) in Figure \ref{fig: topm_optimality}. For comparison, we also include the OCBAm and SAR algorithms in this experiment. From Figure \ref{fig: topm_optimality}, we may obtain the following findings. First, the EFG-\(m\) algorithm is sample optimal across all problem configurations. Taking the SC-Normal configuration as an example, when the number of alternatives \( k \) exceeds \( 2^6 = 64 \), the PCS\(_m\) of EFG-\(m\) stabilizes around 60\%, confirming its sample optimality.   Second, the OCBAm algorithm is clearly not sample optimal in any of the cases considered, as its PCS\(_m\) quickly decreases to zero under all three configurations. Third, the SAR algorithm’s performance is heavily influenced by the distribution type. Under the SC-Normal configuration, its PCS\(_m\) decreases gradually as \( k \) increases. However, under the non-normal configurations, the decline is much steeper, with PCS\(_m\) dropping to zero rapidly, indicating its non-sample-optimality under non-normal distributions.   In contrast, the results demonstrate the robustness of the EFG-\(m\) algorithm  across different distribution types.

    \begin{figure}[ht]
        \centering
        \includegraphics[width=0.8 \textwidth]{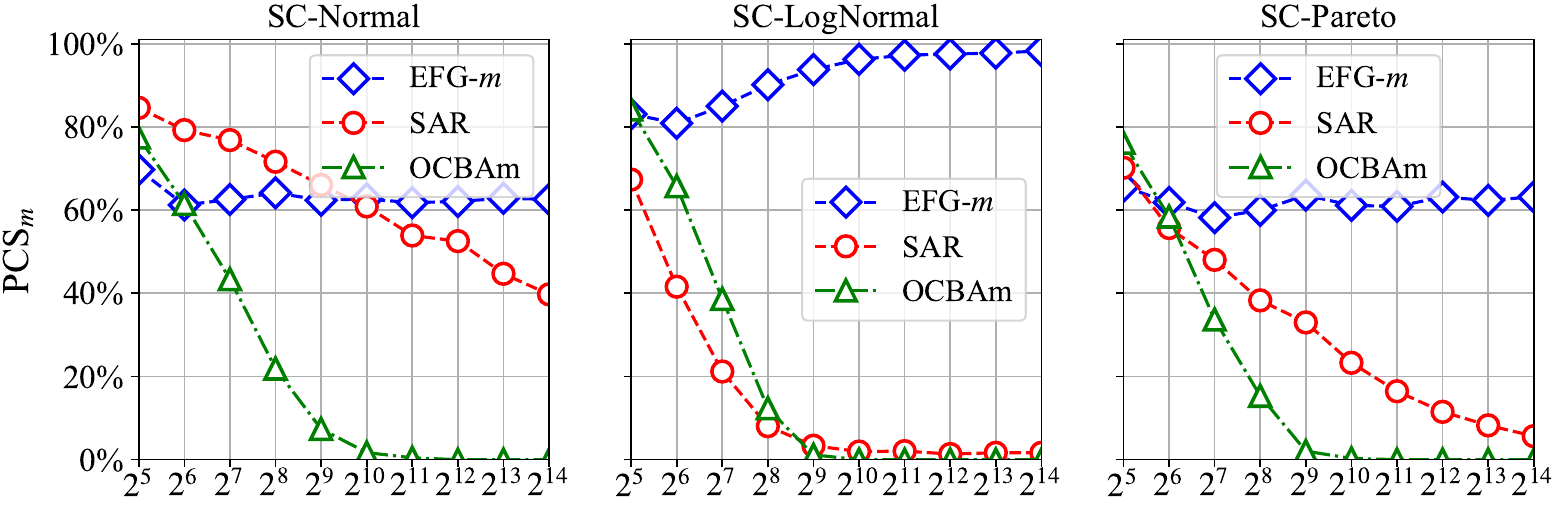}
        \caption{Sample Optimality of the EFG-$m$ algorithm for the PCS$_m$}
        \label{fig: topm_optimality}
        \end{figure}
        
\subsubsection{Consistency.}
\label{subsec: numerical_topm_consistency}
We now proceed to verify the consistency of the EFG-$m$ algorithm. For each SC configuration, we fix the number of alternatives \( k = 2^{14} \); to demonstrate the trend of PCS\(_m\) of the EFG-$m$ algorithm as \( B = ck \) changes, we set \( c \) to increase from 100 to 1000. Then, we plot the PCS\(_m\) of the EFG-$m$ algorithm again different $c$  for each configuration in Figure \ref{fig: topm_consistency}. Figure \ref{fig: topm_consistency} demonstrates the consistency of the EFG-$m$ algorithm. Under the three problem configurations, as \( c \) in the sample budget \( B = ck \) increases, the PCS\(_m\) of the EFG-$m$ algorithm monotonically rises and ultimately converges to 1. From Figure \ref{fig: topm_consistency}, we can also observe a diminishing marginal effect in the PCS\(_m\) growth curve as the evaluation budget increases. Once the PCS\(_m\) surpasses 60\%, the increments in PCS\(_m\) become progressively smaller.

\begin{figure}[ht]
    \centering
    \includegraphics[width= 0.75 \textwidth]{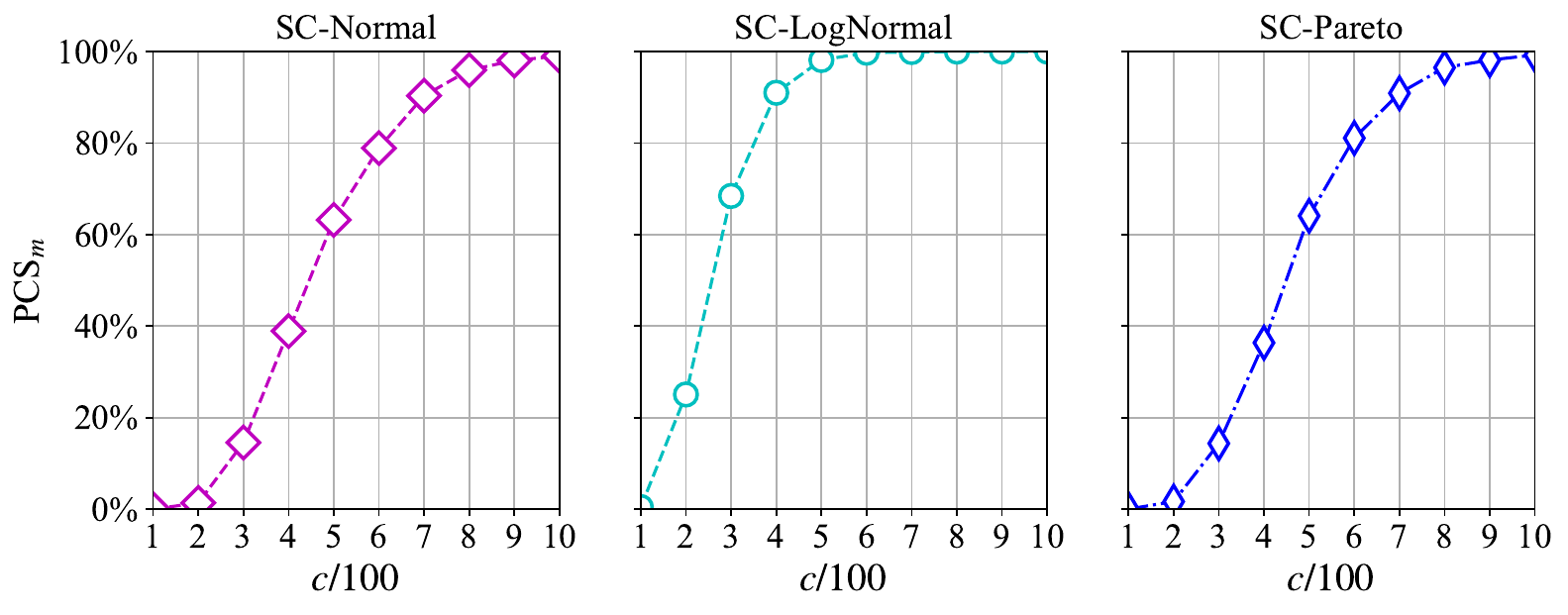}
    \caption{Consistency of the EFG-$m$ algorithm for the PCS$_m$ ($B=ck$)}
            \label{fig: topm_consistency}
    \end{figure}

Next, we use the SC-Normal configuration to show the sample complexity of the EFG-$m$ algorithm regarding various problem parameters. Theorem \ref{thm: consistency} suggests that to maintain the PCS$_m$ of the EFG-$m$ algorithm at least \(1-\alpha\), where $\alpha \in (0, 1)$ is the probability of incorrect screening (PICS$_m$), the order of \(c\) in the total evaluation budget \(B=ck\) should be \(O\left( \frac{\sigma^2}{\gamma^2}\log\frac{m}{\alpha}\right)\). Initially, we focus on the order of the total evaluation budget concerning \(\alpha\). Letting the number of alternatives be \(k=2^{10}\), we fix the other problem parameters and then gradually reduce \(\alpha\) from \(0.2=\frac{2^1}{10}\) to \(\frac{2^{-5}}{10}\). For each value of \(\alpha\), following Theorem \ref{thm: consistency}, we set \(n_0=\frac{c_1 \sigma^2}{\gamma^2}\log\frac{2m}{\alpha}\) and \(n_g =\frac{\alpha}{2} + \frac{4 \alpha \sigma^2}{\gamma^2}\) where \(c_1=5.4\) is a constant independent of \(\alpha\). For each \(\alpha\), we use  30,000 independent macro replications to estimate the EFG-$m$ algorithm's PICS\(_m\). Then, we plot the PICS\(_m\) against different \(\alpha\) in Figure \ref{fig: topm_optimal_alpha}.
    From Figure \ref{fig: topm_optimal_alpha}, it can be seen that as \(\alpha\) gradually decreases, the PICS\(_m\) of the EFG-$m$ algorithm also diminishes; when \(\alpha\) is sufficiently small, the slope of the EFG-$m$ algorithm's PICS\(_m\) is close to -1, demonstrating that it decreases at the same rate as \(\alpha\).  This phenomenon indicates that setting \(c\) as \(O\left(\log \frac{1}{\alpha}\right)\) is sufficient to ensure the EFG-$m$ algorithm's PICS$_m$ remains at most \(\alpha\). We have also investigated the order of the total evaluation budget concerning \(\gamma\), \(\sigma^2\), and \(m\). The numerical results further validate Theorem \ref{thm: consistency} and are included in \ref{subsec: additional_numerical}.

\begin{figure}[ht]
    \centering
    \includegraphics[width=0.35 \textwidth]{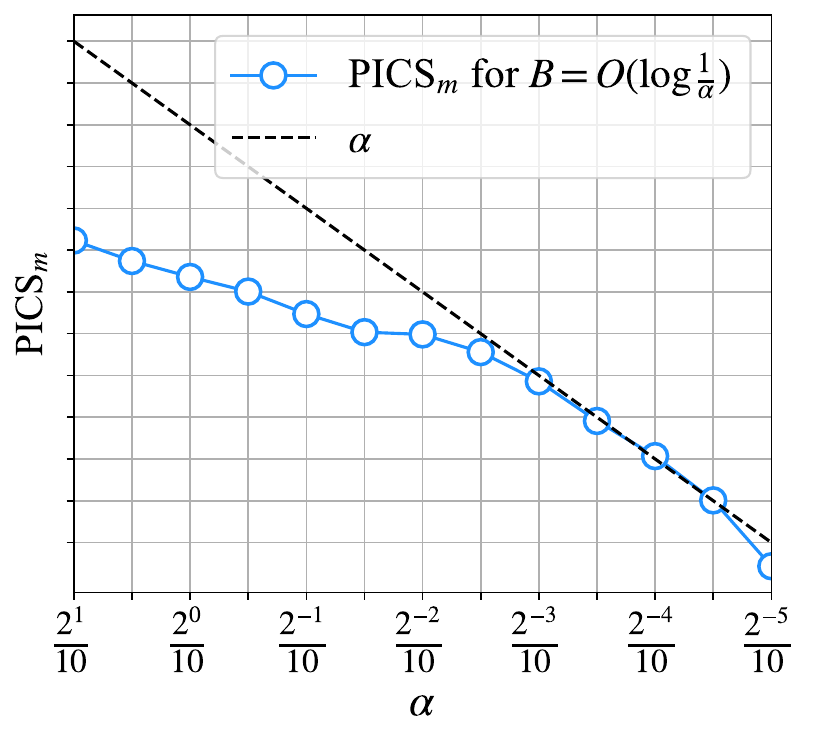}
    \caption{Sample complexity of the EFG-$m$ algorithm regarding $\mathbf{\alpha}$}
    \label{fig: topm_optimal_alpha}
    \end{figure}

\subsection{Free Ranking Effect}
\label{subsec: numerical_topm_ranking}
    
In this subsection, we study the PGS$_m$ and PGSR$_m$ of the EFG-$m$ algorithm under the RM configurations. For each value of \(k\), we let the sample budget \(B = 150k\).  Then, we plot the EFG-$m$ algorithm's PGS\(_m\) and PGSR\(_m\) against different \(k\) for each problem configuration in Figure \ref{fig: topm_ranking}. 
From Figure \ref{fig: topm_ranking}, several findings can be drawn. First, for all three configurations, the EFG-$m$ algorithm achieves the sample optimality for both the  PGS$_m$ and PGSR\(_m\). For instance, under the RM-Pareto configuration, when \(k \geq 2^9 = 512\), the PGS\(_m\) and PGSR$_m$ of the algorithm both stabilize around 80\%. Second, when \(k\) is small, the PGSR\(_m\) of the EFG-$m$ algorithm is smaller than the PGSR$_m$. This is because when \(k\) is small, the evaluation budget \(n_g k\) for the top-$m$ greedy phase is also small, preventing the EFG-$m$ algorithm from having sufficient budget to correct the ranking within the selected subset. However, this issue diminishes as \(k\) increases: the PGSR$_m$  will gradually increase. Notably, for all three configurations, when \(k\) is large, the PGSR\(_m\) of the EFG-$m$ algorithm becomes identical to its PGS\(_m\). This result highlights the EFG-$m$ algorithm’s inherent ``free" ranking ability when solving large-scale virtual screening problems. 

\begin{figure}[ht]
    \centering
    \includegraphics[width=0.8\textwidth]{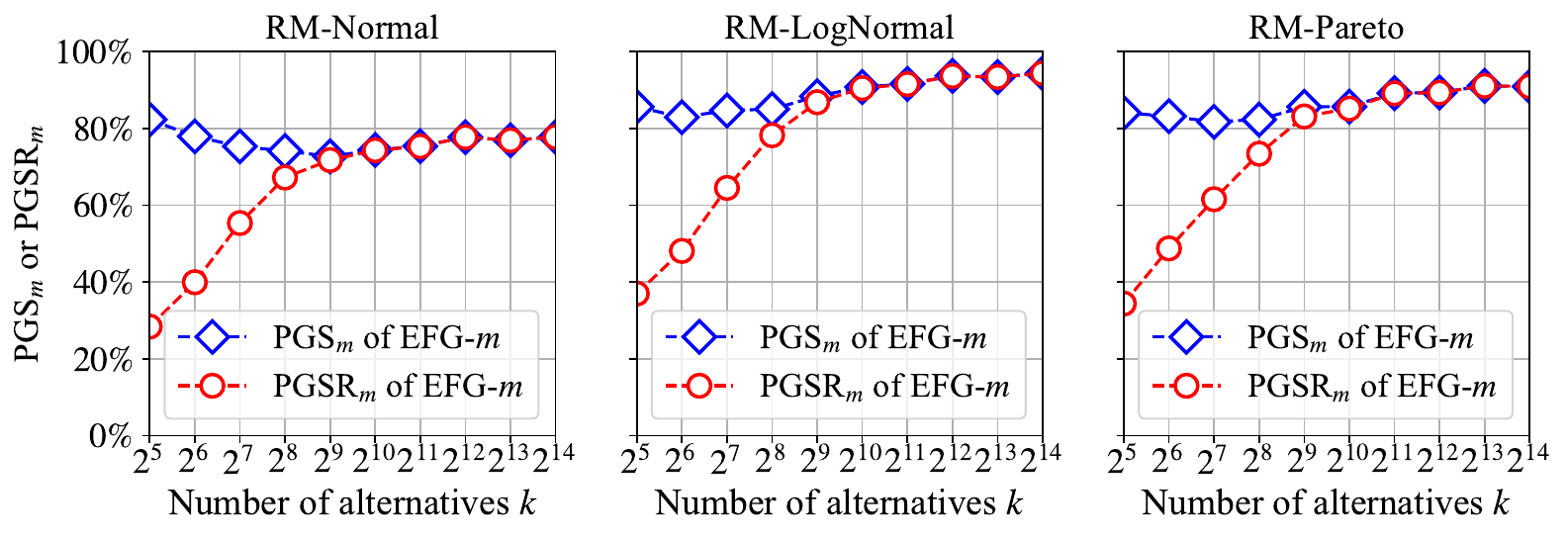}
    \caption{Free ranking effect of the EFG-$m$ algorithm}
    \label{fig: topm_ranking}
    \end{figure}


\subsection{Performance Improvement via Top-$M$ Greedy Selection}
\label{subsec: numerical_topm_enhancement}
This subsection will showcase the performance enhancement of the EFG-$M$ algorithm with a top-$M$ greedy phase, which is introduced in Section \ref{subsec: top_M}, compared to the original EFG-$m$ algorithm. We use the three RM problem configurations and set the total evaluation budget $B=100k$ for each $k$. For the EFG-$M$ algorithm, we set \(M=2m\), which means that in each round of the greedy phase, the EFG-$M$ algorithm will select and sample  the current top-20 alternatives. We show in Section \ref{subsec: additional_select_M} that the performance of  EFG-$M$  is not sensitive to \(M\).  Then,  for each problem configuration, we estimate and plot the PGS$_m$ and PGSR$_m$ of the EFG-$M$ algorithm against different $k$ in Figure \ref{fig: topm_improve_PCS$_m$}. To contrast the performance of the EFG-$M$ algorithm with the EFG-$m$ algorithm, we also include the PGS\(_m\) curves of the EFG-$m$ algorithm in Figure \ref{fig: topm_improve_PCS$_m$}.

From Figure \ref{fig: topm_improve_PCS$_m$},  we observe the following findings. First, similar to the EFG-$m$ algorithm, EFG-$M$ achieves the sample optimality for both the PGS$_m$ and PGSR$_m$. Second, it preserves the free ranking effect. When \(k\) is large, the PGSR\(_m\) of the EFG-$M$ algorithm coincides with its PGS\(_m\). Lastly, and more importantly, the EFG-$M$ algorithm demonstrates a marked improvement in performance over the EFG-$m$ algorithm when \(k\) is relatively large, particularly when \(k \geq 2^{10}\). Notice that in scenarios with smaller \(k\) values, the performance of the EFG-$M$ algorithm can fall short compared to the EFG-$m$ algorithm, likely due to the small evaluation budget \(n_gk\) for the greedy phase at this scale.
\begin{figure}[ht]
    \centering
    \includegraphics[width=0.8 \textwidth]{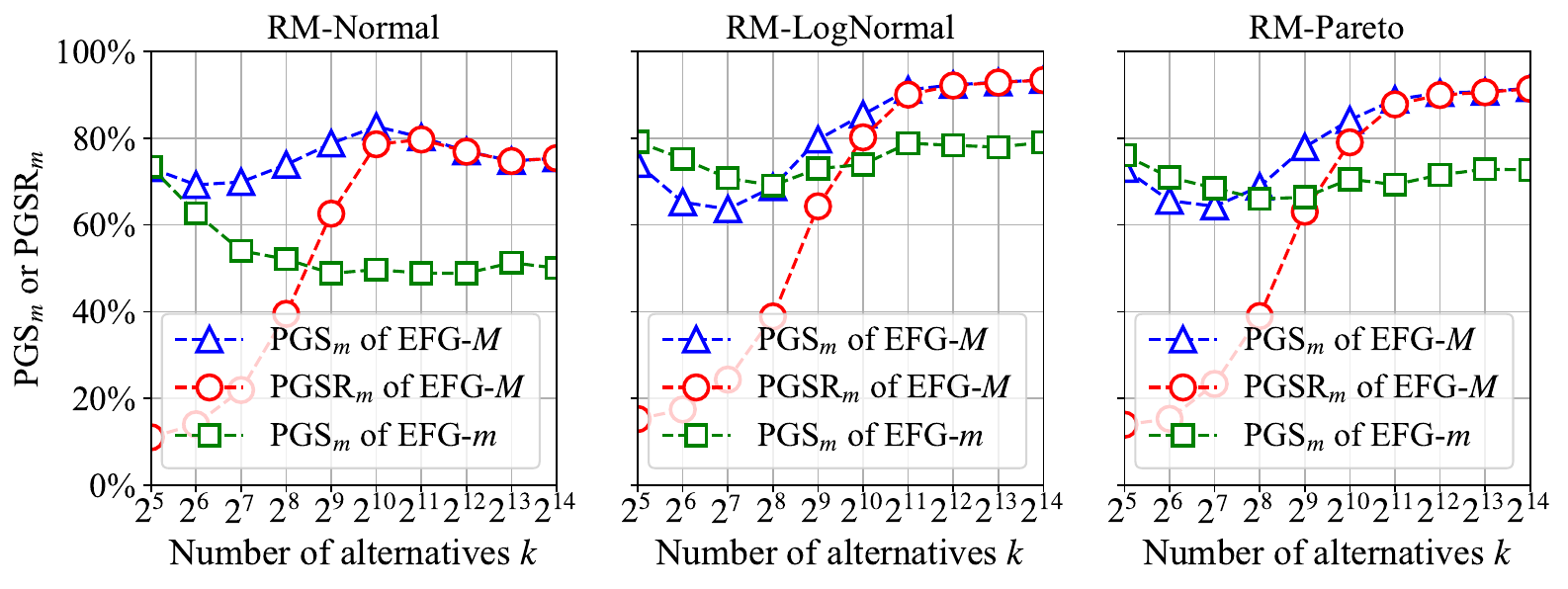}
    \caption{Performance enhancement through top-$M$ greedy selection}
    \label{fig: topm_improve_PCS$_m$}
    \end{figure}

\subsection{Comparison with Existing Algorithms}
\label{subsec: comparison}

In this subsection, we compare the PGS$_m$  and PGSR$_m$ of our seeding-enhanced EFG-$M^+$ algorithm with those of the two heuristic comparators, OCBAm and SAR. We start with a comparison under the synthetic RM configurations with a total evaluation budget $B=100k$ for each $k$.
The PGS$_m$, PGSR$_m$ curves of the three algorithms under each RM configuration in Figure \ref{fig: comparison_PGS_PGSR}  reveal several key observations. First, the EFG-\(M^+\) algorithm, like the EFG-\(m\) and EFG-\(M\) algorithms, is sample optimal for both PGS\(_m\) and PGSR\(_m\) under all configurations and exhibits the free ranking effect.   Second, as in Section \ref{subsubsec: sample_optimality}, the OCBAm algorithm is not sample optimal in any of the cases considered. Interestingly, at a certain point, as PGS\(_m\) and PGSR\(_m\) decrease, they appear to converge, giving the impression of a free ranking effect. However, due to its non-sample-optimality, the effect may have limited practical value.   Third, the SAR algorithm presents a more interesting case. Under the RM-Normal configuration, it performs very well, appearing sample optimal while also obtaining high PGSR$_m$ for large $k$. However, its performance deteriorates significantly under non-normal configurations, where both sample optimality and the ranking effect may vanish.

These results highlight the need for caution when applying heuristic algorithms to large-scale virtual screening problems. As heuristics, they may be either non-sample-optimal or highly sensitive to distribution types. In contrast, the EFG-\(M^+\) algorithm demonstrates consistent performance across different configurations. While the existing heuristics may outperform our algorithms in certain scenarios, our approach is more reliable due to its theoretical foundation.  This reliability may be particularly important in practice.

\begin{figure}[ht]
    \centering
    \includegraphics[width=0.8 \textwidth]{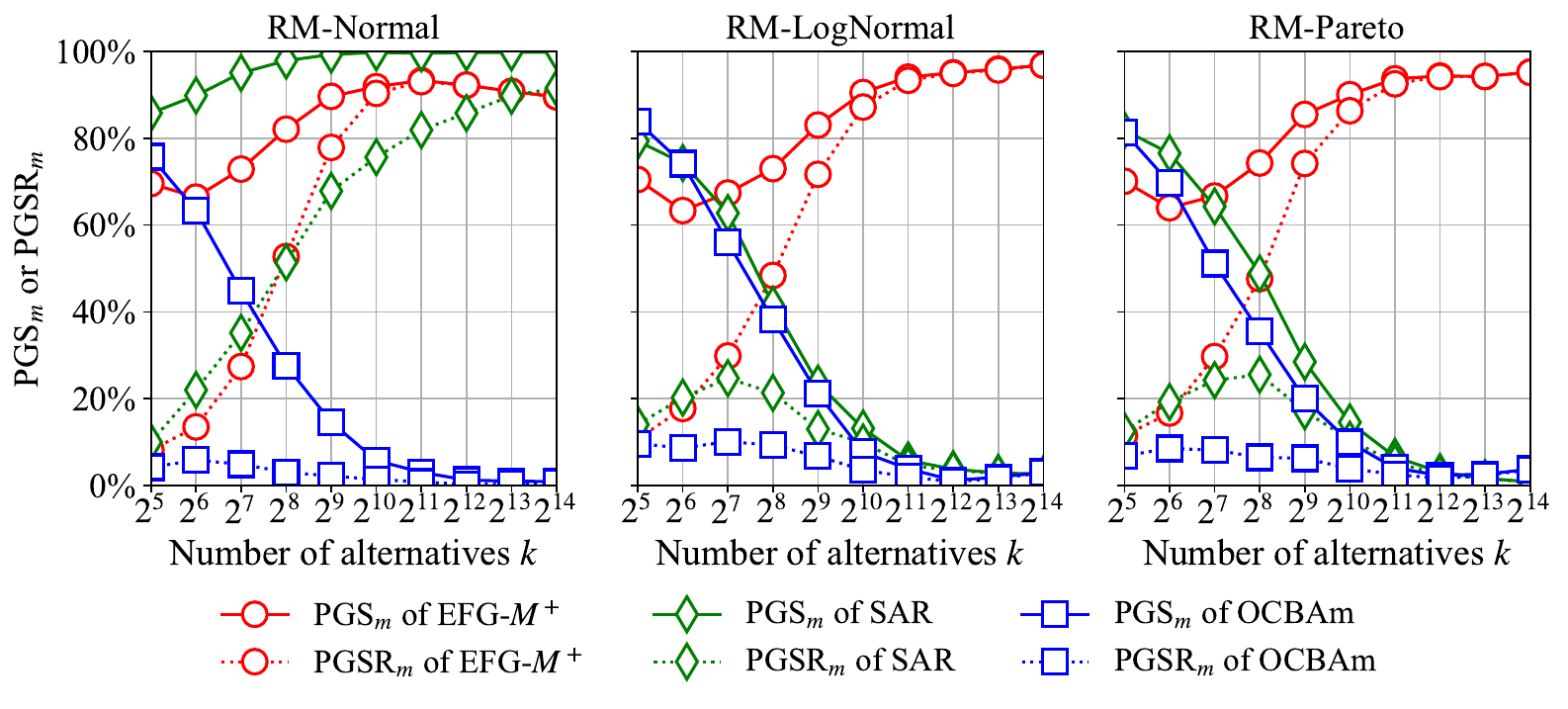}
        \caption{Comparison of PGS$_m$ and PGSR$_m$ for the EFG-$M^+$, OCBAm and SAR algorithms}
    \label{fig: comparison_PGS_PGSR}
    \end{figure}

To further demonstrate the effectiveness of our algorithms, we compare their performance with OCBAm and SAR on a practical simulation-based virtual screening problem drawn from the context of reliability engineering. The detailed problem setup and experimental results are presented in Section~\ref{subsec: numerical_practical}. From the experimental results, we observe patterns consistent with those discussed above, with the EFG-\(M^+\) algorithm consistently outperforming the OCBAm and SAR algorithms across all tested problem configurations in terms of both PGS\(_m\) and PGSR\(_m\). These findings further support that EFG-\(M^+\) is a powerful tool for achieving high-quality screening and ranking in large-scale virtual screening problems.

\section{Supplements to Section \ref{sec: ec_num_justification}}
\label{sec: ec_num_justification2}
\subsection{Algorithms in Comparison}
\label{subsec: additional_algorithms}
In Section \ref{subsec: comparison}, we compare our algorithms with two existing algorithms in the OSS and BAI literature, OCBAm \citep{itemChen2008} and SAR \citep{itemBubeck2013}, through several numerical experiments. Here we provide a brief introduction of these algorithms and detail the implementation settings:

\begin{itemize}
    \item \textbf{The SAR algorithm:} the algorithm is originally proposed in the BAI literature and improves the SR algorithm of \cite{itemAudibert2010} to solve OSS problems. It eliminates one alternative in each iteration, stopping after \( k-1 \) iterations. Alongside the evaluation budget, it does not require additional parameters to be specified. Thus, one may simply follow the pseudocodes in \cite{itemBubeck2013} to implement the algorithm. 
    \item \textbf{The OCBAm algorithm:} the algorithm extends the traditional OCBA (optimal computing budget allocation) approach to solve OSS problems. It operates in two phases. In the initial phase, it allocates to each alternative \( n_1 \geq 1 \) observations to obtain initial estimates of its mean and variance. In the subsequent sequential phase, at each round, the OCBAm algorithm recalculates the budget allocation by solving an optimization problem based on the current sample information; then, it allocates the next observation to the most starving alternative and updates its sample mean and sample variance. Typically in solving small-scale problems, the OCBAm algorithm's performance insensitive to the size of $n_1$. However, our preliminary experiments show that when the problem scale is large, the performance of the OCBAm algorithm may become very sensitive to \( n_1 \).   To set an appropriate $n_1$, we refer to the analysis in \cite{itemWu2018} for the original OCBA algorithm and set \( n_1 = 0.4n \)  given a total evaluation budget  $B=nk$ in our experiments. Additionally, to speed up the experiments, we let the algorithm allocates a small batch of \( \Delta =10\) observations to the most-starving alternative at each round, which would not affect the performance of the OCBAm algorithm.     Notice that alongside the OCBA algorithm, there are also several other OCBA  algorithms for solving small-scale OSS problems (see Section EC.A of \citealt{itemZhang2023} for a comprehensive review). Among these algorithms, we regard the OCBAm algorithm as a representative and include the algorithm in the comparison only. 
    Numerical results in Section \ref{sec: numerical} show that the OCBAm algorithm may be not sample optimal and we believe that this phenomenon may hold for other OCBA algorithms. 
\end{itemize}

\subsection{Other Existing Algorithms and the Computational Complexity}
\label{subsec: other_algorithms}
Alongside the SAR and OCBA algorithms in comparison, there are also several other algorithms in the OSS and BAI literature, including the AOAm algorithm of  \cite{itemZhang2023} and the UGapE algorithm of \cite{itemGabillon2012}.  The AOAm algorithm is developed based on a Bayesian dynamic programming formulation of the optimal budget allocation. It allocates the evaluation budget sequentially and at each round, it approximates the value function of sampling every alternative using the current sample information and then selects the alternative that maximizes the value function approximations (VFA). \cite{itemZhang2023}  demonstrate that this delicately designed algorithm may perform better than the OCBA algorithms in solving small-scale OSS problems. However, this performance gain comes at a more significant computational cost. It can be shown that the computational complexity of  estimating the VFA for all $k$ alternatives in each round is $\mathcal{O}(k^2)$. Therefore, as $k$ grows,  the time spent on selecting the alternative to evaluate may grow significantly and even exceed the time for evaluating it, which makes it computationally prohibitive to apply the AOAm algorithm to solve large-scale virtual screening problems. Similar to the AOAm algorithm, the UGapE algorithm's computational complexity for each selection is also $\mathcal{O}(k^2)$. Thus, we do not include these two algorithms in our numerical comparisons. 

Compared to these algorithms, the EFG-$m$ algorithm is computationally cheap.  The computational complexity of selecting the current top-$m$ alternatives in each round is $\mathcal{O}(k)$ if implemented with the \emph{argpartition} operator. Furthermore, since the EFG-$m$ algorithm makes budge allocation decisions based on the sample means, it is straightforward to use the data structure \emph{heap queue} to reduce the complexity to $\mathcal{O}(\log k)$. We may construct a heap queue of the sample means of all alternatives after the exploration phase. Then, in each round the greedy phase, querying the top-$m$ alternatives from the queue and updating their sample means take only $\mathcal{O}(\log k)$ operations. This feature makes our EFG-$m$ algorithms computationally appealing in solving large-scale problems.

    \subsection{Sample Complexity of the EFG-${m}$ Algorithm}
    \label{subsec: additional_numerical}
    Section \ref{subsec: numerical_topm_consistency} shows the optimal sample complexity of the EFG-$m$ algorithm regarding the probability of incorrect screening (PICS$_m$) $\alpha$. We include in this e-companion additional numerical results to investigate the sample complexity of the EFG-$m$ algorithm regarding the mean difference parameter $\gamma$, the common variance $\sigma^2$ and the subset size $m$.   When investigating each of these three parameters apart, we fix the other parameters, let $k=2^{14}$ and consider multiple values of the parameter; for each value, we set \(n_0=\frac{c_1 \sigma^2}{\gamma^2}\log\frac{2m}{\alpha}\) and \(n_g =\frac{\alpha}{2} + \frac{4 \alpha \sigma^2}{\gamma^2}\) where $c_1$ is a constant independent of the parameter in question. We then plot the PCS\(_m\) curves of the EFG-$m$ algorithm against different parameters in Figure \ref{fig: topm_optimal_others}. From the figure, we can find that given a total evaluation budget $B=\mathcal{O}\left(\frac{k}{\gamma^2}\log\left(\frac{m}{\alpha}\right)\right)$, when $k$ is sufficiently large, the EFG-$m$ algorithm may maintain a constant level of PCS$_m$ regardless of the change in $\gamma$ and $\sigma^2$, indicating the optimal sample complexity of the EFG-$m$ algorithm regarding $\gamma$ and $\sigma^2$ for large-scale virtual screening problems. 



    \begin{figure}[ht]
        \centering
        \includegraphics[width=0.8 \textwidth]{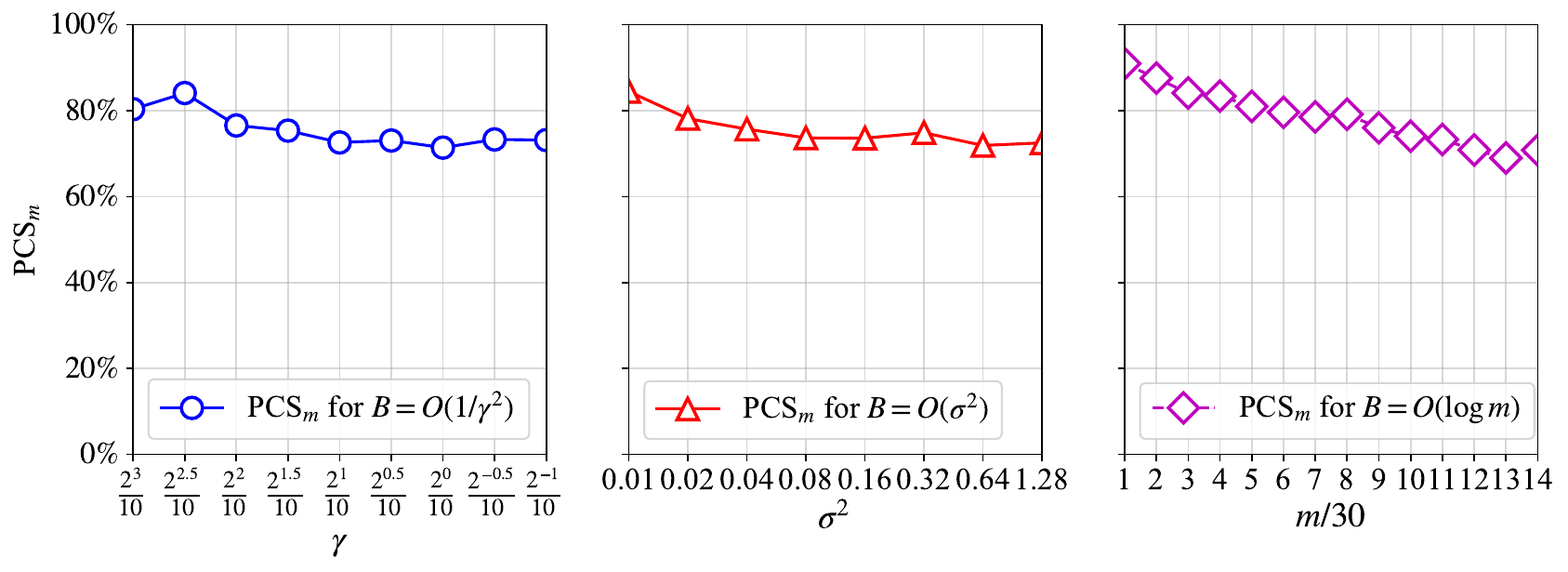}
        \caption{Sample complexity of the EFG-$m$ algorithm regarding \(\gamma\), \(\sigma^2\), and \(m\)}
        \label{fig: topm_optimal_others}
        \end{figure}

Interestingly, as \( m \) increases, the PCS\(_m\) of the EFG-$m$ algorithm exhibits a downward trend. This is because, from a boundary-crossing perspective, as \( m \) increases, the required budget for the top-\( m \) alternatives to reach the minimum of their sample means increases. While \( k \) remains fixed, this budget may become relatively significant and non-negligible compared to the total evaluation budget $B=(n_0+n_g)k$, causing Argument \ref{argu5} to no longer hold. In such cases, the PCS\(_m\) lower bounds established in Theorems \ref{thm: nonnormalPGS} and \ref{thm: consistency} may no longer be valid.  To ensure these results hold, \( k \) should be sufficiently large relative to \( m \).

\subsection{Setting $M$ in EFG-$M$ Algorithms}
\label{subsec: additional_select_M}

It is crucial to recognize the impact of the parameter \(M\) on the EFG-$M$ algorithm's performance. To explore this, we fix \(k = 2^{14}\) and vary \(M/m\) from 1 to 10 to analyze how changes in \(M\) influence the PGS$_m$. Figure \ref{fig: topm_improve_PGS$_m$2} shows the PGS$_m$ curves of the EFG-$M$ algorithm (notice that when $M/m=1$, EFG-$M$ becomes EFG-$m$). These curves show that the EFG-$M$ algorithm's performance may be insensitive to the choice of $M/m$. For the RM-Normal configuration, the curve is almost flat when $M/m \geq 2$; for the other configurations, as $M/m$ rises, the PGS$_m$ of EFG-$M$ declines slowly and keeps its superiority over that of EFG-$m$ in a wide range.  Besides, notice that for all tested configurations, the PGS$_m$ peaks at \(M/m = 2\) or \(M/m = 3\). Based on this and our extensive tuning experience, we recommend setting \(M\) to $2m$ or $3m$ to achieve near-optimal performance in practice.
\begin{figure}[ht]
    \centering
    \includegraphics[width=0.85 \textwidth]{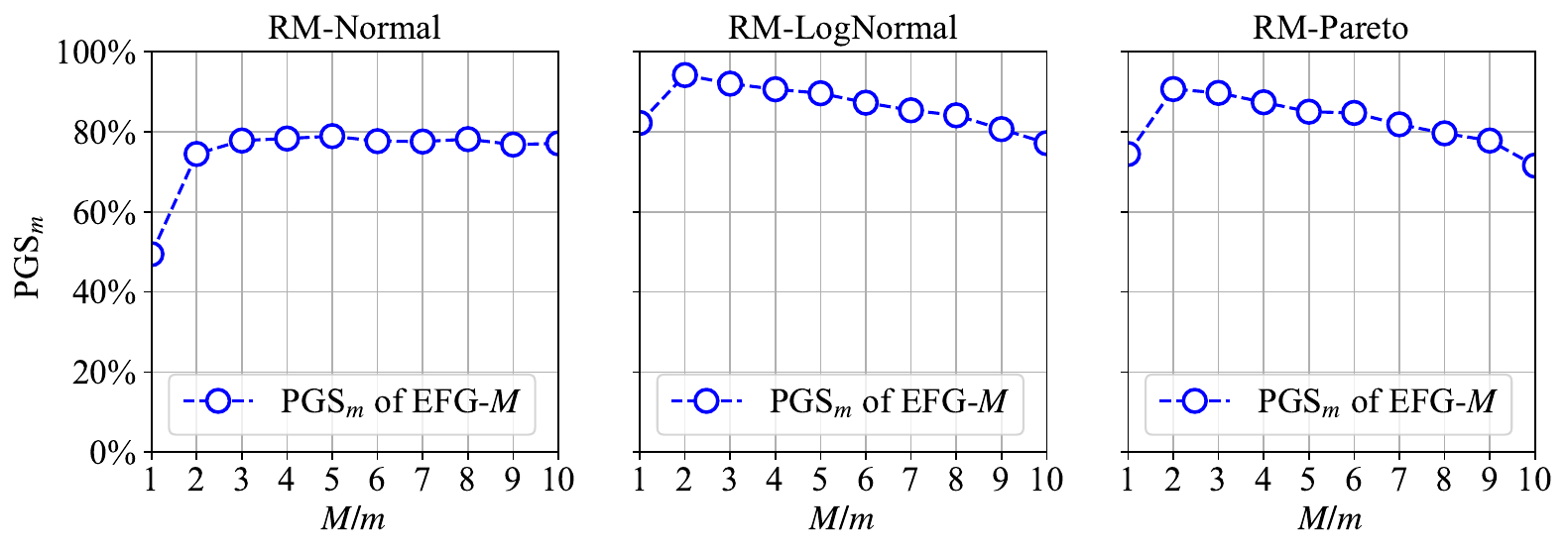}
    \caption{A comparison between the PGS$_m$ of the EFG-$M$ algorithm under different ${M/m}$}
    \label{fig: topm_improve_PGS$_m$2}
    \end{figure}

\subsection{Simulation-Based Redundancy Allocation}
\label{subsec: numerical_practical}

In this subsection, we demonstrate the effectiveness of our sample-optimal EFG-\(m\) algorithms in solving a practical virtual screening problem within the context of reliability engineering. Specifically, we consider a simulation-based redundancy allocation (RA) problem where the goal is to select the most reliable system configurations from a vast set of candidate designs.
A typical form of the RA problem involves a network of subsystems; the objective of the problem is to allocate a limited number of standby components to the subsystems to maximize the overall system reliability, commonly measured by the system lifetime.  As a fundamental problem of reliability engineering, it attracts much research attention and has also been considered in the simulation literature; see, e.g.,  \cite{itemChang2018}.

\begin{figure}[htbp]
    \centering
    \includegraphics[width=0.5\textwidth]{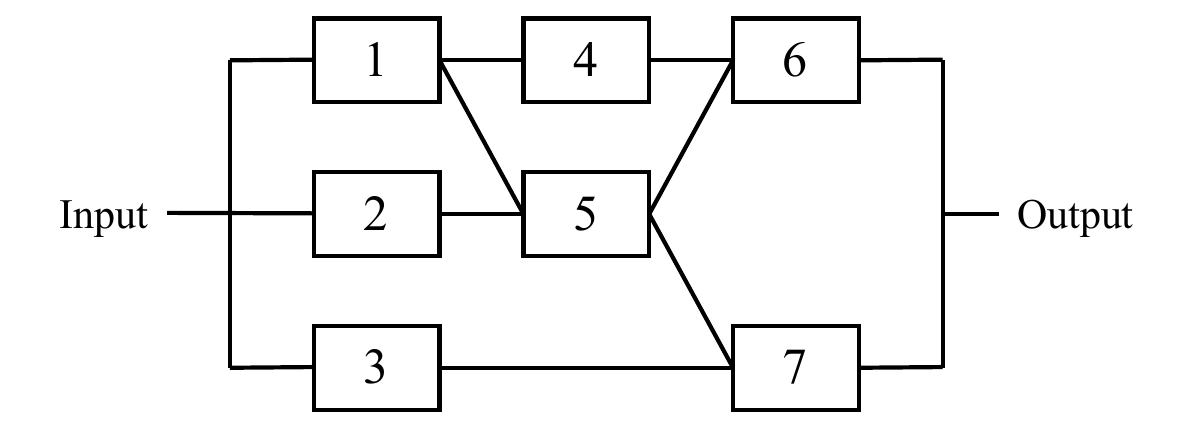}
    \caption{A communication network system}
    \label{fig: allocation_problem}
\end{figure}

In our experiments, we adapt the communication network from \cite{itemZhao2003}, which is shown in Figure \ref{fig: allocation_problem}. The system has 7 subsystems (nodes) and 6 paths.  The lifetime of the system is determined by the most reliable path, that is, the path with the longest lifetime. However, for each path, P$_j$, $j=1, \dots, 6$, the lifetime is determined by the most fragile subsystem in the path. Let $\phi_{i, k}$  represent the (random) lifetime of $k$-th component of subsystem $i$. Assuming independence across components, the total lifetime of subsystem $i$, when allocated to $x_i$ standby components, is $\sum_{k=1}^{x_i} \phi_{i, k}$.  Then,  for each alternative (i.e., feasible allocation scheme) $\mathbf{x}=[x_{1}, x_{2} \ldots, x_{7}] \in \mathbb{N}^7_{+}$, its mean performance may be  expressed as
    \begin{eqnarray}
        \label{eq: reliability_mean}
        \mu(\mathbf{x}) = && \E \left[ \max_{j \in {1, \ldots, 6}} \min_{i \in {P}_j} \sum_{k=1}^{x_i} \phi_{i,k}    \right].
        \end{eqnarray}
For simplicity, we assume that the standby components of different subsystems have the same price, and the total budget allows for purchasing $L$ components.  For a given $L$, our objective is to find a good subset of alternatives with $m=10$ and $\delta=0.1$ among all allocations satisfying $\sum_{i=1}^{7}x_i = L$.  Following the reliability engineering literature, we let $\phi_{i, k}$  follow a  lognormal distribution with shape parameter $\mu_i$ and scale parameter $\sigma_i$ for subsystem $i$.
    The total number of feasible alternatives depends on \( L \), and we consider multiple values of \( L \) to explore different problem scales. \textcolor{black}{To maintain a concise presentation, the experimental setup and parameter settings are provided in Section \ref{subsec: additional_reliability}. Details of the problem instances can be found in Table \ref{table: reliability_problems}, and the PGS$_m$ and PGSR$_m$ of the compared algorithms are summarized in Table \ref{table: reliability_PGS}.}

            \begin{table}[htbp]
                \centering
                \begin{tabular}{ccccccccc}
                    \\ \hline 
                    
                    \hline

                    \hline
                                        Parameter $L$      & \multicolumn{2}{c}{13}       & \multicolumn{2}{c}{14}       & \multicolumn{2}{c}{16}      & \multicolumn{2}{c}{19}      \\ \hline 
                    
                    \hline

                    \hline
                    Number of alternatives $k$      & \multicolumn{2}{c}{1,716}       & \multicolumn{2}{c}{3,432}       & \multicolumn{2}{c}{1,1440}      & \multicolumn{2}{c}{50,388}      \\ \hline 
                    
                    \hline

                    \hline
                    Performance Measure  & PGS$_m$        & PGSR$_m$       & PGS$_m$        & PGSR$_m$       & PGS$_m$        & PGSR$_m$       & PGS$_m$        & PGSR$_m$          \\ \hline 
                    
                    \hline

                    \hline
                    EFG-$m$     & 0.560          & 0.455          & 0.415          & 0.400          & 0.445          & 0.445          & 0.235          & 0.235            \\
                    EFG-$M$     & 0.785          & 0.400          & 0.910          & 0.755          & 0.960          & 0.960          & 0.930          & 0.930           \\
                    EFG-$M$$^+$ & \textbf{0.905} & \textbf{0.490} & \textbf{0.965} & \textbf{0.775} & \textbf{1.000} & \textbf{1.000} & \textbf{1.000} & \textbf{1.000} \\ \hline 
                    SAR      & 0.770          & 0.235          & 0.825          & 0.510          & 0.855          & 0.835          & 0.875          & 0.875           \\
                    OCBAm    & 0.000          & 0.000          & 0.000          & 0.000          & 0.000          & 0.000          & 0.010          & 0.000           \\ \hline 

                    \hline

                    \hline
                    \end{tabular}
                \caption{Estimated PGS$_m$ and PGSR$_m$ on the redundancy allocation problem instances}
                \label{table: reliability_PGS}
                \end{table}

                Table \ref{table: reliability_PGS} offers insights into the performance of our EFG-$m$ algorithms in solving large-scale virtual screening problems. First, the EFG-$m$ algorithm, while being the simplest among the tested algorithms, maintains a nonzero PGS$_m$ across all tested problem instances, reaffirming its sample optimality. In contrast, the OCBAm algorithm fails to achieve this, highlighting the necessity of tailored algorithms in large-scale screening tasks. \
Second, Table \ref{table: reliability_PGS} illustrates the progression of improvements from EFG-\(m\) through EFG-\(M\) to EFG-\(M^+\), with each step introducing significant benefits.  
                Third, the EFG-$m$, EFG-$M$, and EFG-$M$$^+$ algorithms all demonstrate the free ranking effect; when $k$ is sufficiently large, their PGS$_m$ coincide with their respective PGSR$_m$. 
                Fourth, while the SAR algorithm performs better than the original EFG-\(m\) algorithm and also demonstrates the free ranking effect, the EFG-\(M^+\) algorithm consistently outperforms it across all five problem instances. \textcolor{black}{These findings suggest that
                the EFG-\(M^+\) algorithm can be a powerful tool for achieving high-quality screening and ranking in large-scale virtual screening problems.}

\subsection{Supplementary Problem Setup and Implementation Details to Section \ref{subsec: numerical_practical}}
\label{subsec: additional_reliability}

In our experiments, we simply let $\mu_1=\mu_5 = 0.1$, $\mu_2 = \mu_4 = \mu_6 = 0.2$, $\mu_3 = \mu_7 = 0.3$, and $\sigma_1=\cdots = \sigma_7 = 1.5$. Notice that the total number of alternatives $k$ is determined by the value of $L$. We consider several choices of $L$ to represent different problem scales. For each $L$, we estimate the mean performance of each alternative $i=1, \ldots, k$, which is expressed in Equation \eqref{eq: reliability_mean}, based on 300,000 macro independent replications.  We then summarize the information about the problem instances with different $L$ in Table \ref{table: reliability_problems}.
    \begin{table}[htbp]
        \centering
            \begin{tabular}{cccccc}
                \hline 
                    
                \hline
                
                \hline
            Parameter $L$ &  Number of alt.  $k$ &  $\mu_1$ & $\mu_{m}$ &  $\mu_{k}$ & Number of good alt.  $|\mathcal{G}|$ \\ \hline 
                    
            \hline

            \hline

            13 & 1,716 &  4.315 & 4.049 & 1.795 & 23 \\ 
            14 & 3,432 & 4.779 & 4.567 & 1.800 & 18 \\ 
            16 & 11,440& 5.944 & 5.703  & 1.800 & 18 \\ 
            19 & 50,388 & 7.844 &  7.545 &1.800 & 16 \\
            \hline 
                    
            \hline

            \hline
            \end{tabular}
            \caption{Information about the redundancy allocation problem instances}
            \label{table: reliability_problems}
            \end{table}

            In the experiments, we evaluate the EFG-$m$, EFG-$M$, and EFG-$M$$^+$ algorithms using the problem instances detailed in Table \ref{table: reliability_problems}, comparing them against the OCBAm and SAR algorithms. For every problem instance, we set the total evaluation budget $B=200k$.  Then,  we allocate 20\% of this budget to the greedy phase for all our three algorithms; for the EFG-$M$$^+$ algorithm, we allocate 10\% of the budget to the seeding phase. Moreover, we let $M=2m$ for the EFG-$M$ algorithm. Details on the implementation settings for the OCBAm and SAR algorithms are provided in Section \ref{subsec: additional_algorithms}. We then estimate the PGS$_m$ and PGSR$_m$ of each algorithm for every problem instance based on 200 independent macro replications.

\section{Supplements to the LLM-based Case Study in Section \ref{subsec: numerical_LLM}}            
\label{subsec: additional_LLM}

\subsection{Setup Instructions for Configuring the Llama 3.2 Model}
\label{subsubsec: setup}
In our experiments, we use Ollama to set up the Llama 3.2 model, which is available at \url{https://ollama.com/library/llama3.2}. Models from the Ollama library can be customized using a Modelfile for specific tasks. To set up the Llama 3.2 model as a \emph{virtual customer}, we create a Modelfile with the following content:
\begin{promptenv}
FROM llama3.2:3b\\
PARAMETER temperature 1\\
SYSTEM """
You are a customer. You are selected at random while shopping for laptops to participate in a survey.
"""
\end{promptenv}

In this Modelfile, the system message instructs the model to simulate a laptop customer who has been randomly selected to participate in a survey (we borrow the message from Appendix B of \citealt{itemBrand2023}). This customization ensures that the model responds in a manner aligned with a real customer's perspective. Subsequent sections will describe the survey prompts used to query the willingness to pay (WTP) for each laptop design. Readers interested in how to use Modelfile can refer to the official Ollama documentation\footnote{\href{https://github.com/ollama/ollama?tab=readme-ov-file\#customize-a-prompt}{https://github.com/ollama/ollama?tab=readme-ov-file\#customize-a-prompt}}.

\subsection{Prompt for WTP Queries and WTP Observation Generation}
\label{subsubsec: prompt_WTP}

After setting up the system message, we may query the LLM to generate WTP observations for each laptop design. For example, consider a laptop configuration with an Intel Core i5 CPU, 16 GB RAM, and a 256 GB Storage Drive. The corresponding prompt is as follows:
\begin{promptenv}
The customer is asked: What is the maximum price you would be willing to pay for a Lenovo laptop with Intel Core i5 CPU, 16 GB RAM and 256 GB Storage Drive? Please give a single price in numbers (no descriptions).
\end{promptenv}

The system message and such prompt together account for approximately 70–80 tokens.
To generate WTP observations, we send this prompt for every queried alternative design and extract the numerical response from the LLM output. We notice that in some cases, the LLM may not respond in the expected format (i.e., a single numerical price). When a response does not contain an extractable price, we discard the response and re-query the LLM to ensure valid data collection.   Additionally,  we discard rare responses where the WTP exceeds 6,000 to mitigate the impact of erroneous extreme values. 

    
\subsection{Attribute Sets and Values for Different Problem Instances}
\label{subsubsec: attributes}
Section \ref{subsubsec: llm_screening} considers four problem instances in the LLM-based product design screening problem, each defined by a different set of laptop attributes and values. Table \ref{tab:attribute_sets} summarizes the attribute sets and values used in these problem instances.  
Each problem instance expands the number of alternatives by introducing additional attributes. The smallest instance ($k = 36$) considers only CPU, RAM, and storage, while the largest instance ($k = 3,240$) incorporates GPU, screen size, and display resolution.

\begin{table}[htbp]
    \centering
    \caption{Attribute sets and values for different problem instances}
    \label{tab:attribute_sets}
    \begin{tabular}{cl}
        \hline
        
        \hline
        
        \hline
        \text{Problem Instance} & \text{Attributes and Values} \\
        \hline

        \hline

        \hline
        $k = 36$  & \small \begin{tabular}{@{}l@{}}
                      \textit{CPU}: Intel Core i5, Intel Core i7, Intel Core i9, AMD-R5, AMD-R7, AMD-R9 \\
                      \textit{RAM}: 16 GB, 32 GB, 64 GB \\
                      \textit{Storage Drive}: 256 GB, 512 GB
                   \end{tabular} \\  
        \hline
        $k = 360$ & \small\begin{tabular}{@{}l@{}}
                      \textit{CPU}: Intel Core i5, Intel Core i7, Intel Core i9, AMD-R5, AMD-R7, AMD-R9 \\
                      \textit{RAM}: 8 GB, 16 GB, 32 GB, 64 GB \\
                      \textit{Storage Drive}: 256 GB, 512 GB, 1024 GB \\
                      \textit{GPU}: NVIDIA GeForce 20M, 30M, 40M, \\
                      \quad \quad \quad AMD Radeon 5000M, 6000M
                   \end{tabular} \\  
        \hline
        $k = 1,080$ & \small\begin{tabular}{@{}l@{}}
                      \textit{CPU}, \textit{RAM}, \textit{Storage Drive}, and \textit{GPU} as in $k=360$ \\
                      \textit{Screen Display Resolution}: 1080p Full HD, 1440p Quad HD, 4K Ultra HD
                     \end{tabular} \\
        \hline
        $k = 3,240$ & \small \begin{tabular}{@{}l@{}}
                      \textit{CPU}, \textit{RAM}, \textit{Storage Drive}, \textit{GPU}, and \textit{Resolution} as in $k=1,080$ \\
                      \textit{Screen Size}: 13.3 inch, 14 inch, 15.6 inch
                     \end{tabular} \\
        \hline

        \hline

        \hline
    \end{tabular}
\end{table}

\subsection{Prompts for the DIRECT Approach}
\label{subsubsec: direct prompt}
To evaluate the performance of the DIRECT approach, which queries the LLM directly to obtain the top-\( m \) designs from a given set of alternatives, we experiment with three different prompting strategies. Each prompt presents the LLM with a list of all alternative designs and requests a response specifying the top-\( m \) choices.  The examples below illustrate the three prompt formulations for a given problem instance.

\begin{promptenv}
 \makebox[1\linewidth]{\bfseries Prompt 1 for the DIRECT Approach}
The customer is asked: Below is a list of attribute configurations for a Lenovo laptop. Suppose that your budget is unlimited. Please select the top 10 attribute configurations for which you would be willing to pay the highest price, ranking them in descending order of the willingness to pay. Please give only the number IDs in descending order (no descriptions or symbols). Response should be like (1, 2, 3, 4, 5, 6, 7, 8, 9, 10). 
\\

1. A laptop with Intel Core i5 CPU, 16 GB RAM, 256 GB Storage Drive 

... (omitted)

 \makebox[1\linewidth]{\bfseries Prompt 2 for the DIRECT Approach}
    The customer is asked: Below is a list of attribute configurations for a Lenovo laptop: \\

1. A laptop with Intel Core i5 CPU, 16 GB RAM, 256 GB Storage Drive

... (omitted) \\
\\
Suppose that your budget is unlimited. Please select the top 10 attribute configurations for which you would be willing to pay the highest price, ranking them in descending order of your willingness to pay. Please give only the number IDs in descending order (no descriptions or symbols). Response should be like (1, 2, 3, 4, 5, 6, 7, 8, 9, 10).

 \makebox[1\linewidth]{\bfseries Prompt 3 for the DIRECT Approach}
The customer is asked: Below is a list of attribute choices for a Lenovo laptop: \\

CPU: A (Intel Core i5), B (Intel Core i7), C (Intel Core i9), D (AMD-R5), E (AMD-R7), F (AMD-R9)

RAM: A (8GB), B (16GB), C (32GB)

Storage Drive: A (256GB), B (512GB) \\
\\Suppose that you have no budget constraint. Please select the top 10 attribute configurations for which you would be willing to pay the highest price, ranking them in descending order of your willingness to pay. Each configuration should be described by the combination of attribute IDs in the format (CPU, RAM, Storage Drive). The CPU attribute ID should be in [A, B, C, D, E, F]. The RAM attribute ID should be in [A, B, C]. The Storage Drive attribute ID should be in [A, B]. Please provide only the combinations of IDs in descending order (no descriptions or symbols) and ensure no duplicates. Response should be like (A, A, A),(A, A, A),(A, A, A),(A, A, A),(A, A, A).
\end{promptenv}

                    

                    

                    




\subsection{Performance of The DIRECT Approach When $\mathbf{m=1}$} 
\label{subsubsec: direct_top1}
As discussed in Section \ref{subsubsec: llm_screening}, we further evaluate the LLM’s performance under the DIRECT approach when \( m=1 \). In this experiment, we use PGS\(_1\), the probability of selecting a good design whose mean WTP is within \( \delta \) of the design with the maximal mean WTP, as the performance measure. We assess PGS\(_1\) across the three prompting strategies described in Section \ref{subsubsec: direct prompt} and five relatively small-scale problem instances defined in Table \ref{tab:attribute_sets}. Notice that the last two instances with \( k=36 \) and \( k=360 \) are the same as the corresponding problem instances used in Section \ref{subsubsec: llm_screening} for \( m=10 \). The results are summarized in Table \ref{table:LLM_DIRECT_PGS}.

\begin{table}[htbp]
    \centering
    \caption{Attribute sets and values for small-scale problem instances}
    \label{tab:attribute_sets_small}
    \begin{tabular}{cl}
        \hline
        
        \hline
        
        \hline
        \text{Problem Instance} & \text{Attributes and Values} \\
                \hline

        \hline

        \hline
        $k = 6$  & \small \begin{tabular}{@{}l@{}}
                      \textit{CPU}: Intel Core i5, Intel Core i7, Intel Core i9 \\
                        \textit{RAM}: 16 GB, 32 GB
                   \end{tabular} \\  

        \hline
        $k = 12$  & \small \begin{tabular}{@{}l@{}}
                      \textit{CPU} and \textit{RAM} as in $k=6$  \\
                      \textit{Storage Drive}: 256 GB, 512 GB
                   \end{tabular} \\  

        \hline
        $k = 18$  & \small \begin{tabular}{@{}l@{}}
                      \textit{CPU}: Intel Core i5, Intel Core i7, Intel Core i9 \\
                      \textit{RAM}: 16 GB, 32 GB, 64 GB \\
                        \textit{Storage Drive}: 256 GB, 512 GB
                   \end{tabular} \\  
        \hline
        $k = 36$  & \small \begin{tabular}{@{}l@{}}
                      \textit{CPU}: Intel Core i5, Intel Core i7, Intel Core i9, AMD-R5, AMD-R7, AMD-R9 \\
                      \textit{RAM}: 16 GB, 32 GB, 64 GB \\
                      \textit{Storage Drive}: 256 GB, 512 GB
                   \end{tabular}  \\
        \hline
        $k = 360$ & \small\begin{tabular}{@{}l@{}}
                      \textit{CPU}, \textit{RAM}, and \textit{Storage Drive} as in $k=36$ \\
                      \textit{GPU}: NVIDIA GeForce 20M, 30M, 40M, \\
                      \quad \quad \quad AMD Radeon 5000M, 6000M
                   \end{tabular} \\  
        \hline

        \hline

        \hline
    \end{tabular}
\end{table}

\begin{table}[htbp]
    \centering
    \caption{Estimated PGS$_1$ of the DIRECT approach under different prompts}
    \label{table:LLM_DIRECT_PGS}
    \begin{tabular}{ccccccc}
        \hline

                \hline

                        \hline
        Number of alternatives $k$  & 6  & 12  & 18  & 36  & 360\\ 
                \hline

                \hline

                        \hline
        Prompt 1     & 0.200  & 0.340  & 0.100  & 0.040 & 0.000  \\
        Prompt 2        & 0.200  & 0.370  & 0.670  & 0.070 & 0.000\\
        Prompt 3     & 0.095  & 0.220  & 0.080  & 0.120 & 0.002 \\
        \hline

                \hline

                        \hline
    \end{tabular}
\end{table}

Table \ref{table:LLM_DIRECT_PGS} reveals several observations. First, for small-scale problems (\( k = 6, 12, 18, 36 \)), the DIRECT approach achieves nonzero PGS\(_1\). This result suggests that LLMs may have the ability to identify good designs when the number of alternatives is limited. 
Second, as the number of alternatives increases, the effectiveness of the DIRECT approach declines rapidly. For \( k = 36 \) and \( k = 360 \), the PGS\(_1\) values drop significantly, with nearly all prompts failing entirely when \( k = 360 \). This result suggests that LLMs may struggle with direct screening tasks as problem size grows and may not achieve sample optimality.  
Third, while Prompt 2 outperforms Prompt 1 and Prompt 3 for \( k = 6, 12, 18 \), it is outperformed by Prompt 3 for \( k = 36 \) and \( k = 360 \). This highlights the instability of direct screening performance for LLMs across different prompts.  
Overall, these results indicate that the performance of the DIRECT approach is highly unstable, varying significantly across problem instances and prompting strategies. Notice that the PGS\(_1\) results depend on the choice of language model. Using a larger model with more parameters may improve the DIRECT approach's performance. However, we believe that the general observations hold similarly across models.

\end{document}